\documentclass{article}

\usepackage[final,nonatbib]{nips_2016}

\usepackage[utf8]{inputenc} 
\usepackage[T1]{fontenc}    
\usepackage{hyperref}       
\usepackage{url}            
\usepackage{booktabs}       
\usepackage{amsfonts}       
\usepackage{nicefrac}       
\usepackage{microtype}      

\usepackage{macros} 
\usepackage{amsmath,amsthm,amssymb,thmtools,thm-restate} 
\usepackage{xcolor}
\usepackage{algorithm}
\usepackage[noend]{algorithmic}

\def\shownotes{0}  
\ifnum\shownotes=1
\newcommand{\authnote}[2]{{$\ll$\textsf{\footnotesize #1 notes: #2}$\gg$}}
\else
\newcommand{\authnote}[2]{}
\fi

\theoremstyle{plain}
\newtheorem{thm}{Theorem}
\newtheorem{property}[thm]{Property}

\newtheorem{lem}[thm]{Lemma}
\newtheorem{claim}[thm]{Claim}

\theoremstyle{definition}
\newtheorem*{defn}{Definition}

\title{
Recovery Guarantee of Non-negative Matrix Factorization  via Alternating Updates
}

\author{
	Yuanzhi Li, Yingyu Liang, Andrej Risteski \\
	Computer Science Department at Princeton University\\
	35 Olden St, Princeton, NJ 08540 \\
	\texttt{\{yuanzhil, yingyul, risteski\}@cs.princeton.edu}
}

\begin{document}

\maketitle

\begin{abstract}
Non-negative matrix factorization is a popular tool for  decomposing data into feature and weight matrices under non-negativity constraints. It enjoys practical success but is poorly understood theoretically. This paper proposes an algorithm that alternates between decoding the weights and updating the features, and shows that assuming a generative model of the data, it provably recovers the ground-truth under fairly mild conditions. In particular, its only essential requirement on features is linear independence. Furthermore, the algorithm uses ReLU to exploit the non-negativity for decoding the weights, and thus can tolerate adversarial noise that can potentially be as large as the signal, and can tolerate unbiased noise much larger than the signal. The analysis relies on a carefully designed coupling between two potential functions, which we believe is of independent interest.
\end{abstract}

\section{Introduction} \label{sec:intro}

In this paper, we study the problem of non-negative matrix factorization (NMF), where given a matrix $\bY \in \mathbb{R}^{m \times N}$, the goal to find a matrix $\bA \in \mathbb{R}^{m \times n}$ and a \emph{non-negative} matrix  $\bX \in \mathbb{R}^{n \times N}$ such that $\bY \approx \bA \bX$.\footnote{In the usual formulation of the problem, $\bA$ is also assumed to be non-negative, which we will not require in this paper.}  $\bA$ is often referred to as \emph{feature matrix} and $\bX$ referred as \emph{weights}. 
NMF has been extensively used in extracting  a parts representation of the data (e.g.,~\cite{LeeSeu97,LeeSeu99,LeeSeu01}). Empirically it is observed that the non-negativity constraint on the coefficients forcing features to combine, but not cancel out, can lead to much more interpretable features and improved downstream performance of the learned features. 

Despite all the practical success, however, this problem is poorly understood theoretically, with only few provable guarantees known. 
Moreover, many of the theoretical algorithms are based on heavy tools from algebraic geometry (e.g., \cite{AroGeKanMoi12}) or tensors (e.g. \cite{anandkumar2012two}), which are still not as widely used in practice primarily because of computational feasibility issues or sensitivity to assumptions on $\bA$ and $\bX$. Some others depend on specific structure of the feature matrix, such as separability~\cite{AroGeKanMoi12} or similar properties~\cite{bhattacharyya2016nonnegative}.

A natural family of algorithms for NMF alternate between decoding the weights and updating the features. More precisely, in the decoding step, the algorithm represents the data as a non-negative combination of the current set of features; in the updating step, it updates the features using the decoded representations. This meta-algorithm is popular in practice due to ease of implementation, computational efficiency, and empirical quality of the recovered features. However, even less theoretical analysis exists for such algorithms. 

This paper proposes an algorithm in the above framework with provable recovery guarantees. To be specific, the data is assumed to come from a generative model $y = \bAg x^* + \nu$. Here, $\bAg$ is the ground-truth feature matrix, $x^*$ are the non-negative ground-truth weights generated from an unknown distribution, and $\nu$ is the noise. Our algorithm can provably recover $\bAg$ under mild conditions, even in the presence of large adversarial noise.  

\textbf{Overview of main results.}
The existing theoretical results on NMF can be roughly split into two categories. In the first category, they make heavy structural assumptions on the feature matrix $\bAg$ such as separability (\cite{arora1}) or allowing running time exponential in $n$ ( \cite{AroGeKanMoi12}). In the second one, they impose strict distributional assumptions on $x^*$ (\cite{anandkumar2012two}), where the methods are usually based on the method of moments and tensor decompositions and have poor tolerance to noise, which is very important in practice. 

In this paper, we present a very simple and natural alternating update algorithm that achieves the best of both worlds. First, we have minimal assumptions on the feature matrix $\bAg$: the only essential condition is linear independence of the features. Second, it is robust to adversarial noise $\nu$ which in some parameter regimes can potentially be on the same order as the signal $\bAg x^*$, and is robust to unbiased noise potentially even higher than the signal by a factor of $O(\sqrt{n})$. The algorithm does not require knowing the distribution of $x^*$, and allows a fairly wide family of interesting distributions. We get this at a rather small cost of a mild ``warm start''. Namely, we initialize each of the features to be ``correlated'' with the ground-truth features. 
This type of initialization is often used in practice as well, for example in LDA-c, the most popular software for topic modeling (\cite{ldac}).   
  


A major feature of our algorithm is the significant robustness to noise. 
In the presence of adversarial noise on each entry of $y$ up to level $\cnoise$, the noise level $\|\nu\|_1$ can be in the same order as the signal $\bAg x^*$. Still, our algorithm 
is able to output a matrix $\bA$ such that the \emph{final} $\|\bAg - \bA \|_1 \le O(\| \nu \|_1)$ in the order of the noise in \emph{one} data point. 
If the noise is unbiased (i.e., $\E[\nu|x^*] =0$), the noise level $\|\nu\|_1$ can be $\Omega(\sqrt{n})$ times larger than the signal $\bAg x^*$, while we can still guarantee $\| \bAg - \bA \|_1 \le O\rbr{ \| \nu \|_1 \sqrt{\ntopic} }$ -- so our algorithm is not only tolerant to noise, but also has very strong denoising effect. Note that even for the unbiased case the noise can potentially be correlated with the ground-truth in very complicated manner, and also, all our results are obtained only requiring the columns of $\bAg$ are independent. 

\textbf{Technical contribution.}
The success of our algorithm crucially relies on exploiting the non-negativity of $x^*$ by a ReLU thresholding step during the decoding procedure. Similar techniques have been considered in prior works on matrix factorization, however to the best of our knowledge, the analysis (e.g., \cite{aroradictionary1}) requires that the decodings are correct in all the intermediate iterations, in the sense that the supports of $x^*$ are recovered with no error. Indeed, we cannot hope for a similar guarantee in our setting, since we consider adversarial noise that could potentially be the same order as the signal. Our major technical contribution is a way to deal with the erroneous decoding throughout all the intermediate iterations. We achieve this by a coupling between two potential functions that capture different aspects of the working matrix $\bA$. 
While analyzing iterative algorithms like alternating minimization or gradient descent in non-convex settings is a popular topic in recent years, the proof usually proceeds by showing that the updates are approximately performing gradient descent on an objective with some local or hidden convex structure. Our technique diverges from the common proof strategy, and we believe is interesting in its own right. 

\textbf{Organization.} 
After reviewing related work, we define the problem in Section~\ref{sec:problem} and describe our main algorithm in Section~\ref{sec:algo_purification}. To emphasize the key ideas, we first present the results and the proof sketch for a simplified yet still interesting case in Section~\ref{sec:result_simplified}, and then present the results under much more general assumptions in Section~\ref{sec:result_general}. The complete proof is provided in the appendix. 

\section{Related work} \label{sec:relatedwork}

Non-negative matrix factorization relates to several different topics in machine learning. We provide a high level review, and discuss in more details in the appendix. 
  
\textbf{Non-negative matrix factorization.}  The area of non-negative matrix factorization (NMF) has a rich empirical history, starting with the practical algorithm of \cite{LeeSeu97}.
On the theoretical side, \cite{AroGeKanMoi12} provides a fixed-parameter tractable algorithm for NMF, which solves algebraic equations and thus has poor noise tolerance. 
\cite{AroGeKanMoi12} also studies NMF under separability assumptions about the features. \cite{bhattacharyya2016nonnegative} studies NMF under heavy noise, but also needs assumptions related to separability, such as the existence of dominant features. Also, their noise model is different from ours. 

\textbf{Topic modeling.} A closely related problem to NMF is topic modeling, a common generative model for textual data~\cite{blei2003latent,blei2012probabilistic}. Usually, $\|x^*\|_1 = 1$ while there also exist work that assume $x^*_i \in [0,1]$ and are independent~\cite{zhu2012sparse}. A popular heuristic in practice for learning $\bA^*$ is \emph{variational inference}, which can be interpreted as alternating minimization in KL divergence norm. On the theory front, there is a sequence of works by 
based on either spectral or combinatorial approaches, which need certain ``non-overlapping'' assumptions on the topics. For example, 
\cite{arora2} assume the topic-word matrix contains ``anchor words'': words which appear in a single topic. 
Most related is the work of \cite{AwaRis15} who analyze a version of the variational inference updates when documents are long.  However, they require strong assumptions on both the warm start, and the amount of ``non-overlapping'' of the topics in the topic-word matrix.



\textbf{ICA.} Our generative model for $x^*$ will assume the coordinates are independent, therefore our problem can be viewed as a non-negative variant of ICA with high levels of noise. 
Results here typically are not robust to noise, with the exception of \cite{arora2012provable} that tolerates Gaussian noise. However, to best of our knowledge, no result in this setting is provably robust to adversarial noise.
%


\textbf{Non-convex optimization.} The framework of having a ``decoding'' for the samples, along with performing an update for the model parameters has proven successful for dictionary learning as well. The original empirical work proposing such an algorithm (in fact, it suggested that the V1 layer processes visual signals in the same manner) was due to \cite{olshausen1997sparse}. Even more, similar families of algorithms based on ``decoding'' and gradient-descent are believed to be neurally plausible as mechanisms for a variety of tasks like clustering, dimension-reduction, NMF, etc (\cite{mitya1,mitya4}). 
A theoretical analysis came latter for dictionary learning due to \cite{aroradictionary1}
under the assumption that the columns of $\bA^*$ are incoherent. The technique is not directly applicable to our case, as we don't wish to have any assumptions on the matrix $\bA^*$. For instance,  if $\bA^*$ is non-negative and columns with $l_1$ norm 1, incoherence effectively means the the columns of $\bA^*$ have very small overlap.



\section{Problem definition and assumptions} \label{sec:problem}

Given a matrix $\bY \in \Real^{\nword \times N}$, the goal of non-negative matrix factorization (NMF) is to find a matrix $\bA \in \Real^{\nword \times \ntopic}$ and a  non-negative matrix $\bX \in \Real^{\ntopic \times N}$, so that $\bY \approx \bA \bX$. The columns of $\bY$ are called data points, those of $\bA$ are features, and those of $\bX$ are weights.  We note that in the original NMF, $\bA$ is also assumed to be non-negative, which is not required here. We also note that typically $m \gg n$, i.e., the features are a few representative components in the data space. This is different from dictionary learning where overcompleteness is often assumed. 

The problem in the worst case is NP-hard~\cite{AroGeKanMoi12}, so some assumptions are needed to design provable efficient algorithms. In this paper, we consider a generative model for the data point 
\begin{align}
  y = \bAg x^* + \noise \label{eqn:model}
\end{align}
where $\bAg$ is the ground-truth feature matrix, $x^*$ is the ground-truth non-negative weight from some unknown distribution, and $\nu$ is the noise. Our focus is to recover $\bAg$ given access to the data distribution, assuming some properties of $\bAg$, $x^*$, and $\noise$. 
To describe our assumptions, we let $[\bM]^i$ denote the $i$-th row of a matrix $\bM$, $[\bM]_j$ its $i$-th column, $\bM_{i,j}$ its $(i,j)$-th entry. Denote its column norm, row norm, and symmetrized norm as
$
  \nbr{\bM}_1 = \max_j {\sum_i \abr{ \bM_{i,j}}},   
	\nbr{\bM}_\infty = \max_i {\sum_j \abr{ \bM_{i,j}}},
$
and
$
	\sym{\bM} = \max\cbr{ \nbr{\bM}_1, \nbr{\bM}_\infty},
$
respectively.

We assume the following hold for parameters $\xexpc, \xsl, \xsu, \ell, \cnoise$ to be determined in our theorems. 
\begin{itemize}
\item[(\textbf{A1})] The columns of $\bA^*$ are linearly independent.
\item[(\textbf{A2})] For all $i \in [\ntopic]$, $x_i^* \in [0, 1]$, $\E[x^*_i] \le \frac{\xexpc}{\ntopic} $ and $\frac{\xsl}{\ntopic} \le \E[(x^*_i)^2] \le \frac{\xsu}{\ntopic}$, and $x^*_i$'s are independent. 
\item[(\textbf{A3})] The initialization $\titime{\bA}{0} = \bA^* ( \titime{\bSigma}{0} + \titime{\bE}{0}) + \bN^{(0)}$, where $\titime{\bSigma}{0}$ is diagonal, $\titime{\bE}{0}$ is off-diagonal, and
\[
  \titime{\bSigma}{0} \succeq (1 -\ell) \bI, \quad \sym{\titime{\bE}{0}} \le \ell. 
\] 
\end{itemize}

We consider two noise models.
\begin{itemize}
\item[(\textbf{N1})]  Adversarial noise: only assume that $\max_i |\noise_i| \le \cnoise$ almost surely.
\item[(\textbf{N2})] Unbiased noise: $\max_i |\noise_i| \le \cnoise$ almost surely, and $\E[\noise | x^*] = 0$.
\end{itemize}

\textbf{Remarks.} We make several remarks about each of the assumptions. \\
(\textbf{A1}) is the assumption about $\bAg$. It only requires the columns of $\bAg$ to be linear independent, which is very mild and needed to ensure identifiability. Otherwise, for instance, if $(\bA^*)_3 = \lambda_1 (\bA^*)_1 + \lambda_2 (\bA^*)_2$, it is impossible to distinguish between the case when $x^*_3 = 1$ and the case when $x^*_2 = \lambda_1$ and $x^*_1 = \lambda_2$.
In particular, we do not restrict the feature matrix to be non-negative, which is more general than the traditional NMF and is potentially useful for many applications. We also do not make incoherence or anchor word assumptions that are typical in related work.  

(\textbf{A2}) is the assumption on $x^*$. First, the coordinates are non-negative and bounded by 1; this is simply a matter of scaling. Second, the assumption on the moments requires that, roughly speaking, each feature should appear with reasonable probability. This is expected: if the occurrences of the features are extremely unbalanced, then it will be difficult to recover the rare ones. 
The third requirement on independence is motivated by that the features should be different so that their occurrences are not correlated. Here we do not stick to a specific distribution, since the moment conditions are more general, and highlight the essential properties our algorithm needs. Example distributions satisfying our assumptions will be discussed later.

The warm start required by (\textbf{A3}) means that each feature $\bA^{(0)}_i$ has a large fraction of the ground-truth feature $\bAg_i$ and a small fraction of the other features, plus some noise outside the span of the ground-truth features.  We emphasize that $\bN^{(0)}$ is the component of $\bA^{(0)}$ outside the column space of $\bAg$, and is not the difference between $\bA^{(0)}$ and $\bAg$. This requirement is typically achieved in practice by setting the columns of $\bA^{(0)}$ to reasonable ``pure'' data points that contains one major feature and a small fraction of some other features (e.g. \cite{ldac,AwaRis15}); in this initialization, it is generally believed that $\bN^{(0)} = 0$. But we state our theorems to allow some noise $\bN^{(0)}$ for robustness in the initialization.

The adversarial noise model (\textbf{N1}) is very general, only imposing an upper bound on the entry-wise noise level. Thus, $\noise$ can be correlated with $x^*$ in some complicated unknown way. (\textbf{N2}) additionally requires it to be zero mean, which is commonly assumed and will be exploited by our algorithm to tolerate larger noise.

\section{Main algorithm} \label{sec:algo_purification}

\newsavebox{\algopurification}
\savebox{\algopurification}{
\begin{minipage}{\textwidth}
\begin{algorithm}[H]
\caption{Purification}\label{alg:main_sim_noise}
\begin{algorithmic}[1]
\REQUIRE{initialization $\bA^{(0)}$, threshold $\alpha$, step size $\eta$, scaling factor $\etaratio$, sample size $N$, iterations $T$}
\FOR{$t = 0, 1, 2, ..., T - 1$}
\STATE Draw examples $y_1, \dots, y_N$.
\STATE (Decode) Compute $\bA^\dagger$, the pseudo-inverse of $\bA^{(t)}$ with minimum $\|(\bA)^\dagger\|_\infty$.\\
\qquad\qquad~Set $x = \phi_{\alpha} (\bA^{\dagger} y)$ for each example $y$. 
\hfill \textit{// $\phi_\alpha$ is ReLU activation; see \eqnref{eqn:relu} for the definition}
\STATE (Update)
Update the feature matrix \\
$
  \qquad\qquad \qquad\qquad 
	\bA^{(t + 1)} = \left(1 - \eta \right)\bA^{(t)} + \etaratio \eta \hat{\E}\left[ (y - y')(x - x')^{\top}\right]
$\\
where $\hat{\E}$ is over independent uniform $y, y'$ from $\cbr{y_1, \dots, y_N}$,  and $x, x'$ are their decodings.  
\ENDFOR
\ENSURE $\bA = \bA^{(T)}$ 
\end{algorithmic}
\end{algorithm}
\end{minipage}
}
\usebox{\algopurification}

Our main algorithm is presented in Algorithm~\ref{alg:main_sim_noise}.
It keeps a working feature matrix and operates in iterations. In each iteration, it first compute the weights for a batch of $N$ examples (\emph{decoding}), and then uses the computed weights to update the feature matrix (\emph{updating}). 

The decoding is simply multiplying the example by the pseudo-inverse of the current feature matrix and then passing it through the rectified linear unit (ReLU) $\phi_\alpha$ with offset $\alpha$.  The pseudo-inverse with minimum infinity norm is used so as to maximize the robustness to noise (see the theorems). The ReLU function $\phi_\alpha$ operates element-wise on the input vector $v$, and for an element $v_i$, it is defined as
\begin{align}
   \phi_\alpha(v_i) & = \max\cbr{v_i - \alpha, 0}. \label{eqn:relu}
\end{align}
To get an intuition why the decoding makes sense, suppose the current feature matrix is the ground-truth. Then 
$
  \bA^{\dagger} y = \bA^{\dagger} \bAg x^* + \bA^{\dagger} \noise = x^* + \bA^{\dagger} \noise.
$
So we would like to use a small $\bA^{\dagger}$ and use threshold to remove the noise term.

In the encoding step, the algorithm move the feature matrix along the direction $\E\left[ (y - y')(x - x')^{\top}\right]$. To see intuitively why this is a good direction, note that when the decoding is perfect and there is no noise, $\E\left[ (y - y')(x - x')^{\top}\right] = \bA^*$, and thus it is moving towards the ground-truth.  
Without those ideal conditions, we need to choose a proper step size, which is tuned by the parameters $\eta$ and $r$.

\section{Results for a simplified case} \label{sec:result_simplified}

We will state and demonstrate our results and proof intuition in a simplified setting first, with assumptions (\textbf{A1}), (\textbf{A2'}), (\textbf{A3}), and (\textbf{N1}), where 
\begin{itemize}
\item[(\textbf{A2'})] $x^*_i$'s are independent, and $x^*_i = 1$ with probability $s/n$ and $0$ otherwise for a constant $s > 0$.
\end{itemize}
Furthermore, we will assume $\bN^{(0)} = 0$. 

Note this is a special case of our general assumptions, with $C_1 = c_2 = C_2 = s$ where $s$ is the parameter in (\textbf{A2'}). It is still an interesting setting: to the best of our knowledge there is no existing guarantee of alternating type algorithms for it. Moreover, we will present the general result in Section~\ref{sec:result_general} which will be easier to digest after we have presented this simplified setting. 

For notational convenience, let $(\bAg)^\dagger$ denote the matrix satisfying $(\bAg)^\dagger \bAg = \bI$. If there are multiple such matrices we let it denote the one with minimum $\|(\bAg)^\dagger\|_\infty$.

\begin{thm}[Simplified case, adversarial noise] \label{thm:main_sim_adv}
There exists an \emph{absolute constant} $\abscc$ such that if Assumptions (\textbf{A1}),(\textbf{A2}'),(\textbf{A3}) and (\textbf{N1}) are satisfied with $l = 1/10$, $\cnoise \le \frac{ \abscc c}{\max\cbr{\nword, \ntopic  \nbr{ \rbr{\bAg}^\dagger}_\infty} }$ for some $0 \le c\le 1$  and $\titime{\bN}{0} = 0$, then there is a choice of parameters $\alpha, \eta, r$ such that for every $0< \epsilon, \delta < 1$ and  $N  = \mathrm{poly}(\ntopic, \nword, 1/\epsilon, 1/\delta)$ the following holds with probability at least $1-\delta$:

After $T = O\rbr{ \ln \frac{1}{\epsilon}}$ iterations, Algorithm~\ref{alg:main_sim_noise} outputs a solution $\bA = \bAg (\bSigma + \bE) + \bN$ where $\bSigma \succeq (1-\ell) \bI$ is diagonal, $\nbr{\bE}_1 \le \epsilon+c$ is off-diagonal, and $\nbr{\bN}_1 \le c$. 
\end{thm}

\textbf{Remarks.} 
Consequently, when $\nbr{\bAg}_1 = 1$, we can do normalization $\bhA_i = \bA_i/ \nbr{\bA_i}_1$, and the normalized output $\bhA$ satisfies 
\[
  \|\bhA - \bAg \|_1 \le \epsilon + 2c.
\]
In particular, under mild conditions and with proper parameters, our algorithm recovers the ground-truth in a geometric rate. It can achieve arbitrary small recovery error in the noiseless setting, and achieve error up to the noise limit even with adversarial noise whose level is comparable to the signal. 

The condition on $\ell$ means that a constant warm start is sufficient for our algorithm to converge, which is much better than previous work such as~\cite{AwaRis15}: indeed, there $\ell$ depends on the dynamic range of the entries of $\bA^*$ which is problematic in practice. 

The result implies that with large adversarial noise, the algorithm can still recover the features up to the noise limit. When $m \ge n \|\rbr{\bAg}^\dagger\|_\infty$, each data point has adversarial noise with $\ell_1$ norm as large as $\| \nu \|_1 = \cnoise \nword = \Omega(c)$, which is in the same order as the signal $\|\bAg x^*\|_1 = O(1)$. Our algorithm still works in this regime. Furthermore, the \emph{final} error $\nbr{\bA - \bAg}_1$ is $O(c)$, in the same order as the \emph{adversarial} noise in \emph{one} data point. 

Note the appearance of $\|\rbr{\bAg}^\dagger\|_\infty$ is not surprising. The case when the columns are the canonical unit vectors for instance, which corresponds to $\|\rbr{\bAg}^\dagger\|_\infty = 1$, is expected to be easier than the case when the columns are nearly the same, which corresponds to large $\|\rbr{\bAg}^\dagger\|_\infty$.  

A similar theorem holds for the unbiased noise model.
\begin{thm}[Simplified case, unbiased noise] \label{thm:main_sim_unbiased}
If Assumptions (\textbf{A1}),(\textbf{A2}'),(\textbf{A3}) and (\textbf{N2}) are satisfied with $\cnoise =  \frac{\mathcal{G} c\sqrt{\ntopic}}{\max\cbr{ \nword, \ntopic \nbr{\rbr{\bAg}^\dagger}_\infty } }$, then the same guarantee as Theorem~\ref{thm:main_sim_adv} holds. 
 \end{thm}

\textbf{Remarks.}
With unbiased noise which is commonly assumed in many applications, the algorithm can tolerate noise level $\sqrt{n}$ larger than the adversarial case. 
When $m \ge n \|\rbr{\bAg}^\dagger\|_\infty$, each data point has noise with $\ell_1$ norm as large as $\| \nu \|_1 = \cnoise \nword = \Omega(c\sqrt{n})$, which can be $\Omega(\sqrt{n})$ times larger than the signal $\|\bAg x^*\|_1 = O(1)$. The algorithm can recover the ground-truth in this heavy noise regime. Furthermore, the \emph{final} error $\nbr{\bA - \bAg}_1$ is $O\rbr{\| \nu \|_1 / \sqrt{n} }$, which is only $O(1/\sqrt{n})$ fraction of the noise in \emph{one} data point. This is a strong denoising effect and a bit counter-intuitive. It is possible since we exploit averaging of the noise for cancellation, as well as thresholding to remove noise spread out in the coordinates.

\subsection{Analysis: intuition}

A natural approach typically employed to analyze algorithms for non-convex problems is to define a function on the intermediate solution $\bA$ and the ground-truth $\bAg$ measuring their distance and then show that the function decreases at each step. However, a single potential function will not be enough in our case, as we argue below, so we introduce a novel framework of maintaining two potential functions which capture different aspects of the intermediate solutions.

Let us denote the intermediate solution and the update as (omitting the superscript $(t)$)
\begin{equation} \bA  = \bAg (\bSigma + \bE) + \bN, \quad  \hat{\E}[(y - y')(x - x')^{\top}]  = \bAg(\btSigma + \btE) + \btN,  \label{eqn:intuitionupdate02} \end{equation}
where $\bSigma$ and $\btSigma$ are diagonal, $\bE$ and $\btE$ are off-diagonal, and $\bN$ and $\btN$ are the terms outside the span of $\bAg$ which is caused by the noise. To cleanly illustrate the intuition behind ReLU and the coupled potential functions, we focus on the noiseless case and assume that we have infinite samples. 

Since $\bA_i = \bSigma_{i,i} \bAg_i + \sum_{j \neq i}\bE_{j,i} \bAg_j$,  if the ratio between $\nbr{\bE_i}_1  = \sum_{j \neq i} \abr{\bE_{j,i}}$  and  $\bSigma_{i,i}$ gets smaller, then the algorithm is making progress; if the ratio is large at the end, a normalization of $\bA_i$ gives a good approximation of $\bAg_i$. 
So it suffices to show that $\bSigma_{i,i}$ is always about a constant while $\nbr{\bE_i}_1$ decreases at each iteration. 
We will focus on $\bE$ and consider the update rule in more detail to argue this. After some calculation, we have
\begin{align} 
  & \bE  \leftarrow (1- \eta)  \bE +  r \eta \btE, & \btE = \E [ \rbr{x^* - (x')^*} \rbr{x - x'}^{\top} ], \label{eqn:intuition:update2}
\end{align}
where $x, x'$ are the decoding for $x^*, (x')^*$ respectively:
\begin{align} \label{eqn:intuition:decode}
  & x = \phi_\alpha\rbr{ (\bSigma + \bE)^{-1} x^* }, 
	& x' = \phi_\alpha\rbr{ (\bSigma + \bE)^{-1} (x')^* }.
\end{align}

To see why the ReLU function matters, consider the case when we do not use it. 
\begin{align*} 
 \btE & =  \E (x^*-(x')^*)\left[\bA^{\dagger}\bA^*(x^*-(x')^*)\right]^\top = \E \left[(x^*-(x')^*) (x^*-(x')^*)^\top\right] \left[(\bSigma + \bE)^{-1}\right]^\top 
\\
& \propto \left[(\bSigma + \bE)^{-1}\right]^\top \approx \bSigma^{-1} - \bSigma^{-1}\bE\bSigma^{-1}.
\end{align*}
where we used Taylor expansion and the fact that $\E \left[(x^*-(x')^*) (x^*-(x')^*)^\top\right]$ is a scaling of identity.
Hence, if we think of $\bSigma$ as approximately $\bI$ and take an appropriate $r$, the update to the matrix $\bE$ is approximately 
$\bE \leftarrow \bE - \eta \bE^{\top}.$ 
Since we do not have control over the signs of $\bE$ throughout the iterations, the problematic case is when the entries of $\bE^{\top}$ and $\bE$ roughly match in signs, which would lead to the entries of $\bE$ increasing. 

Now we consider the decoding to see why the ReLU is helpful. Ignoring the higher order terms and regarding $\bSigma = \bI$, we have
\begin{align} \label{eqn:intuition:inverse}
  x & = \phi_\alpha\rbr{ (\bSigma + \bE)^{-1} x^* } \approx \phi_\alpha\rbr{ \bSigma^{-1} x^* - \bSigma^{-1}\bE\bSigma^{-1} x^*} \approx \phi_\alpha\rbr{ x^* - \bE x^*}.
\end{align}
The problematic term is $\bE x^*$. These errors when summed up will be comparable or even larger than the signal, and the algorithm will fail. However, since the signal coordinates are non-negative and most coordinates with errors only have small values, the hope is that thresholding with ReLU can remove those errors while keeping a large fraction of the signal coordinates. This leads to large $\btSigma_{i,i}$ and small $\btE_{j,i}$'s, and then we can choose an $r$ such that $\bE_{j,i}$'s keep decreasing while $\bSigma_{i,i}$'s stay in a certain range. 

To quantify the intuition above, we need to divide $\bE$ into its positive part $\bE_+$ and its negative part $\bE_-$:
\begin{align}
 & \sbr{\bE_+}_{i,j} = \max\cbr{\bE_{i,j}, 0}, 
& \sbr{\bE_-}_{i,j} = \max\cbr{-\bE_{i,j}, 0}.
\end{align}
The reason to do so is the following: when $\bE_{i,j}$ is negative, by the Taylor expansion approximation, $\sbr{(\bSigma + \bE)^{-1}x^*}_i$ will tend to be more positive and will not be thresholded to 0 by the ReLU most of the time. Therefore, $\bE_{j, i}$ will become more positive at next iteration. On the other hand, when $\bE_{i,j}$ is positive, $\sbr{(\bSigma + \bE)^{-1}x^*}_i$ will tend to be more negative and zeroed out by the ReLU function. Therefore, $\bE_{j, i}$ will \emph{not} be more negative at next iteration. Informally, we will show for positive and negative parts of $\bE$:
$$
\text{postive}^{(t + 1)} \leftarrow (1 - \eta) \text{positive}^{(t)} + (\eta) \text{negative}^{(t)},\text{negative}^{(t + 1)} \leftarrow (1 - \eta) \text{negative}^{(t)} + (\varepsilon \eta) \text{positive}^{(t)}
$$
for a small $\varepsilon \ll 1$. Due to the appearance of $\varepsilon$ in the above updates, we can ``couple'' the two parts, namely show that a weighted average of them will decrease, which implies that $\|\bE\|_s$ is small at the end. This leads to our coupled potential function.\footnote{Note that since intuitively, $\bE_{i,j}$ gets affected by $\bE_{j,i}$ after an update, if we have a row which contains negative entries, it is possible that $\|\bA_i - \bA^*_i\|_1$ increases. So we cannot simply  use $\max_i \|\bA_i - \bA^*_i\|_1$ as a potential function.}

\subsection{Analysis: proof sketch}

We now provide a proof sketch for the simplified case presented above. The complete proof of the results for the general case (which is stated in the next section) is presented in the appendix. The lemmas here are direct corollaries of those in the appendix.

\textbf{One iteration.}
We focus on one update and omit the superscript $(t)$. 
Recall the definitions of $\bE$, $\bSigma$, $\bN$ and $\btE$, $\btSigma$ and $\btN$ from \eqref{eqn:intuitionupdate02}. Our goal is to derive lower and upper bounds for $\btE$, $\btSigma$ and $\btN$, assuming that $\bSigma_{i,i}$ falls into some range around 1, while $\bE$ and $\bN$ are small. 
This will allow us to do induction on $t$. 

First, begin with the decoding. A simple calculation shows that the decoding for $y = \bA^* x^* + \noise$ is
\begin{align}
x & = \phi_{\alpha}\rbr{ \forder x^* + \xi }, \text{~~where~}
  \forder  = \left( \bSigma + \bE \right)^{-1}, ~
	\xi = - \bA^\dagger \bN \forder x^* +  \bA^\dagger \noise.
\end{align}

Now, we can present our key lemmas bounding $\btE$, $\btSigma$, and $\btN$. Before doing this, we add that the particular value for $r$ we will choose is $r = \frac{n}{s}$ (recalling $s$ is the sparsity of $x^*$ according to Assumption (\textbf{A2'})). We also set the threshold of the ReLU as $\rho < \alpha \ll \frac{s}{n}$. Then, we get:   

\begin{lem}[Simplified bound on $\btE$, informal] \label{lem:update_sim_E}
(1) if $\forder_{i, j} < 0$, then 
$
  \left| \btE_{j, i} \right|  \leq o\rbr{ \frac{s}{n} \left(|\forder_{i,j}| + \rho \right)},
$
\\
(2) if $\forder_{i, j} \ge 0$, then 
$
   -O\left(\left(\frac{s}{n}\right)^2  \forder_{i,j} + \rho \forder_{i,j} \right)
  \leq 
  \btE_{j, i}  
	\leq  
	O\rbr{(\frac{s}{n} + \rho) | \forder_{i, j} | }.
$
\end{lem}

Note that $\forder \approx \bSigma^{-1} - \bSigma^{-1}\bE\bSigma^{-1}$, so $\forder_{i, j}<0$ corresponds roughly to $\bE_{i,j} > 0$. In this case, keeping in mind that $r = \frac{n}{s}$, the upper bound on $|\btE_{j, i}|$ is small enough to ensure $|\bE_{j, i}|$ decreases, as described in the intuition.

On the other hand, when $\forder_{i, j}\ge 0$ (roughly $\bE_{i,j} < 0$), the upper bound on $\btE_{j,i}$ is large enough that $r \btE_{j,i}$ can be on the same order as $\bE_{i,j}$, 
corresponding to the intuition that negative $\bE_{i,j} $ can contribute a large positive value to $\bE_{j,i}$. Fortunately, the lower bound on $\btE_{j,i}$ is of much smaller absolute value, which allows us to show that a potential function that couples Case (1) and Case (2) in Lemma~\ref{lem:update_sim_E} actually decreases; see the induction below.
 
\begin{lem}[Simplified bound on $\btSigma$, informal] \label{lem:update_sim_diag}
$ \btSigma_{i, i } \ge  \Omega((\bSigma_{i,i}^{-1} - \alpha)/n)$.
%
\end{lem}

\begin{lem}[Simplified bound on $\btN$, adversarial noise, informal] \label{lem:bound_error}
$\abr{ \btN_{i,j} }\le O(\cnoise/n)$. 
%
%
\end{lem}


\textbf{Induction by iterations.}
We now show how to use the three lemmas to prove the theorem for the adversarial noise. The proof for the unbiased noise statement is similar. 

Let $a_t := \sym{ \pt{\bE} } $ and $b_t := \sym{ \nt{\bE} } $, and choose $\eta = \ell/6$.
We begin with proving the following three claims by induction on $t$: at the beginning of iteration $t$, 
\begin{itemize}
\item[(1)] $(1 - \ell)\bI \preceq \titime{\bSigma}{t} $ 
\item[(2)]  $\sym{\titime{\bE}{t}} \le  1/8$, and if $t > 0$, then
$
  a_{t} + \beta b_{t} \le  \rbr{1-\frac{1}{25}\eta }  \rbr{ a_{t-1} + \beta b_{t-1}} + \eta h,
$
for some $\beta \in (1, 8)$, and some small value $h$,
\item[(3)] $\sym{\titime{\bN}{t}} \le c/10$. 
\end{itemize}
The most interesting part is the second claim. At a high level, by Lemma~\ref{lem:update_sim_E}, we can show that
\begin{align*}
& a_{t+1} \le  \rbr{1-\frac{3}{25}\eta  } a_{t}  +  7\eta b_t + \eta h, 
& b_{t+1}  \le  \rbr{1- \frac{24}{25}\eta  } b_{t}  +  \frac{1}{100}\eta a_t + \eta h. 
\end{align*}
Notice that the contribution of $b_t$ to $a_{t+1}$ is quite large (due to the larger upper bound in Case (2) in Lemma~\ref{lem:update_sim_E}), but the other contributions are all small.  This allows to choose a $\beta \in (1,8)$ so that 
$a_{t+1} + \beta b_{t+1}$ leads to the desired recurrence in the second claim. In other words, $a_{t+1} + \beta b_{t+1}$ is our potential function which decreases at each iteration up to the level $h$.  The other claims can also be proved by the corresponding lemmas.
Then the theorem follows from the induction claims.


\section{More general results} \label{sec:result_general}

\paragraph{More general weight distributions.} Our argument holds under more general assumptions on $x^*$. 

\begin{restatable}[Adversarial noise]{thm}{purificationadv}  \label{thm:main_correlated_noise} 
There exists an absolute constant $\abscc$ such that if Assumption (\textbf{A0})-(\textbf{A3}) and (\textbf{N1}) are satisfied with $l = 1/10$, $\xsu \le 2 \xsl$, $\xexpc^3 \le \abscc \xsl^2 n$, 
$\cnoise   \le \cbr{ \frac{c_2^2 \abscc c }{ C_1^2 m}, \frac{c_2^4 \abscc c }{ C_1^5\ntopic  \nbr{ \rbr{\bAg}^\dagger}_\infty} } $
  for $0 \le c\le 1$,  and $\nbr{\titime{\bN}{0}}_{\infty} \le \frac{c_2^2 \abscc c}{ C_1^3 \nbr{(\bAg)^\dagger}_\infty}$, then there is a choice of parameters $\alpha, \eta, r$ such that for every $0< \epsilon, \delta < 1$ and  $N  = \mathrm{poly}(\ntopic, \nword, 1/\epsilon, 1/\delta)$, with probability at least $1-\delta$ the following holds:

After $T = O\rbr{ \ln \frac{1}{\epsilon}}$ iterations, Algorithm~\ref{alg:main_sim_noise} outputs a solution $\bA = \bAg (\bSigma + \bE) + \bN$ where $\bSigma \succeq (1-\ell) \bI$ is diagonal, $\nbr{\bE}_1 \le \epsilon+c/2$ is off-diagonal, and $\nbr{\bN}_1 \le c/2$. 
\end{restatable}

\begin{restatable}[Unbiased noise]{thm}{purificationunbiased}  \label{thm:main_unbiased_noise}
If Assumption (\textbf{A0})-(\textbf{A3}) and (\textbf{N2}) are satisfied with 
$\cnoise =  \frac{c_2 \mathcal{G} \sqrt{c \ntopic}}{C_1 \max\cbr{ \nword, \ntopic \nbr{\rbr{\bAg}^\dagger}_\infty }  }$ 
and the other parameters set as in Theorem~\ref{thm:main_correlated_noise}, then the same guarantee holds.
\end{restatable}

The conditions on $\xexpc, \xsl, \xsu$ intuitively mean that each feature needs to appear with reasonable probability. $\xsu \le 2\xsl$ means that their proportions are reasonably balanced. This may be a mild restriction for some applications -- however, we additionally propose a pre-processing step that can relax this in the following subsection.  

The conditions allow a rather general family of distributions, so we point out an important special case to provide a more concrete sense of the parameters. For example, for the uniform independent distribution considered in the simplified case, we can actually allow $s$ to be much larger than a constant; our algorithm just requires $s \le \abscc n$ for a fixed constant $\abscc$. So it works for uniform sparse distributions even when the sparsity is linear, which is an order of magnitude larger than what can be achieved in the dictionary learning regime. 
Furthermore, the distributions of $x^*_i$ can be very different, since we only require $C_1^3 = O(c_2^2 n)$. 
Moreover, all these can be handled without specific structural assumptions on $\bAg$. 

%

\paragraph{More general proportions.} A mild restriction in Theorem~\ref{thm:main_correlated_noise} and~\ref{thm:main_unbiased_noise} is that $\xsu \le 2 \xsl$, that is, $\max_{i \in [n]} \E[(x^*_i)^2] \le 2 \min_{i \in [n]}\E[(x^*_i)^2]$. To relax this, we propose a pre-processing algorithm for balancing $\E[(x^*_i)^2]$. 

The idea is quite natural: instead of solving $\bY \approx \bAg \bX$, we could also solve $\bY \approx [ \bAg \bD] [ (\bD)^{-1} \bX]$ for a positive diagonal matrix $\bD$, where ${\E[(x_i^*)^2]}/{\bD_{i, i}^2}$ is with in a factor of $2$ from each other. We show in the appendix that this can be done under assumptions as the above theorems, and additionally $\bSigma \preceq (1+\ell) \bI$ and $\bE^{(0)} \ge 0$ entry-wise. After balancing, one can use Algorithm~\ref{alg:main_sim_noise} on the new ground-truth matrix $[ \bAg \bD]$ to get the final result.

\section{Conclusion} \label{sec:conclusion}

A simple and natural algorithm that alternates between decoding and updating is proposed for non-negative matrix factorization and theoretical guarantees are provided. The algorithm provably recovers a feature matrix close to the ground-truth and is robust to noise. Our analysis provides insights on the effect of the ReLU units in the presence of the non-negativity constraints, and the resulting interesting dynamics of the convergence. 


\section*{Acknowledgements} 
This work was supported in part by NSF grants CCF-1527371, DMS-1317308, Simons Investigator Award, Simons Collaboration Grant, and ONR-N00014-16-1-2329. 

\bibliographystyle{alpha}
\bibliography{nonnegative_main}

\newpage
\appendix

\section{Preliminary} \label{app:preliminary}

Given a matrix $\bY \in \Real^{\nword \times N}$, the goal of non-negative matrix factorization (NMF) is to find a matrix $\bA \in \Real^{\nword \times \ntopic}$ and a non-negative matrix $\bX \in \Real^{\ntopic \times N}$, so that $\bY \approx \bA \bX$. 
The columns of $\bY$ are called data points, those of $\bA$ are features, and those of $\bX$ are weights. 

The notation $[\bM]_j$ denotes the $j$-th column of $\bM$, $[\bM]^i$ denotes the $i$-th row of $\bM$, and $\bM_{i,j}$ denotes the element of $\bM$ at the $i$-th row and $j$-th column. 
Furthermore, let $\bM_+ = $ denote the positive part of the matrix, and let $\bM_-$ denote the absolute value of the negative part of the matrix:
\begin{align*}
[\bM_+]_{i,j} & = 
\begin{cases}
\bM_{i,j} & \textnormal{~if~} \bM_{i,j} \ge 0, \\
0 & \textnormal{~if~} \bM_{i,j} < 0,
\end{cases} 
\\
[\bM_-]_{i,j} & = 
\begin{cases}
0 & \textnormal{~if~} \bM_{i,j} \ge 0, \\
\abr{\bM_{i,j}} & \textnormal{~if~} \bM_{i,j} < 0.
\end{cases}
\end{align*} 

For analysis, the following norms of the matrices are needed. 

\begin{defn}[$l_1$ norm of a matrix (induced column norm)]
The (induced) $l_1$ norm of a matrix $\bE \in \Real^{n \times n}$ is
\[
  \| \bE\|_1 = \max_{i \in [n]}\left\{ \sum_{j = 1}^n  |\bE_{j, i} | \right\}.
\]
\end{defn}

\begin{defn}[$l_{\infty}$ norm of a matrix (induced row norm)]
The (induced) $l_{\infty}$ norm of a matrix $\bE \in \Real^{n \times n}$ is
\[
  \| \bE\|_{\infty} = \max_{i \in [n]}\left\{ \sum_{j = 1}^n| \bE_{i, j}| \right\}.
\]
\end{defn}

These two norms are related, and they enjoy the sub-multipicity property of the induced norm.  

\begin{property}[dual norm]
For a matrix $\bE \in \Real^{n \times n}$, 
\[
  \| \bE \|_1 = \| \bE^{\top} \|_{\infty}.
\]
\end{property}

Note that unlike $l_2$ norm, it is possible that $\|\bE\|_1 \not= \|\bE^{\top} \|_1$ or  $\|\bE\|_{\infty} \not= \|\bE^{\top} \|_{\infty}$.

\begin{property}[induced norm of a matrix]
Let $\bE_1, \bE_2 \in \Real^{n \times n}$ be two matrices, then 
\begin{align*}
  \| \bE_1 \bE_2 \|_1 & \le \|\bE_1\|_1 \|\bE_2 \|_1, \\ 
	\| \bE_1 \bE_2 \|_{\infty} & \le \|\bE_1\|_{\infty}\|\bE_2 \|_{\infty}.
\end{align*}
\end{property}

The following two kinds of norms are also useful for the analysis.

\begin{defn}[symmetrized norm of a matrix]
The symmetrized norm of a matrix $\bE \in \Real^{n \times n}$ is
\[
  \sym{\bE} = \max(\|\bE\|_1,\|\bE\|_{\infty}).
\]
\end{defn}

Note that $\sym{\bE}$ is a norm since it's the maximum of two norms.

\begin{defn}[max norm]
The max norm of a matrix $\bE \in \Real^{m \times n}$ is 
\begin{eqnarray*}
\nbr{\bE}_{\max} = \max_{i,j}  \abr{ \bE_{i,j} }.
\end{eqnarray*}
\end{defn}

For the function $\phi_\alpha$ used in our decoding algorithm, we frequently use the following properties in the analysis.
\begin{property}[ReLU] \label{prop:relu}
$\phi_\alpha(z) = \max\rbr{0, z - \alpha}$ is non-decreasing. It is 1-Lipschitz, i.e.,
\begin{align}
  \abr{\phi_\alpha(z_1) - \phi_\alpha(z_2)} & \le \abr{z_1 - z_2}.
\end{align}
It satisfies
\begin{align}
  \phi_\alpha(z) & \ge z - \alpha, \\
  \phi_\alpha(z) & \le \abr{z - \alpha}.
\end{align}
Furthermore,  if $\alpha > 0 $, 
\begin{align}
  \phi_\alpha(z) & \le \abr{z}.
\end{align}
\end{property}

\section{Proofs for main algorithm: Purification} \label{app:proof_purification}

Since NMF is NP-hard in the worst case, some assumptions are needed to make it tractable. 
In this paper, we consider a generative model for the data point $y = \bAg x^* + \noise$, where $\bAg$ is the ground-truth feature matrix, $x^*$ is the ground-truth non-negative weight from some unknown distribution, and $\nu$ is the noise. Our focus is to recover $\bAg$ given access to the data distribution, assuming the following hold for parameters $\xexpc, \xsl, \xsu, \ell, \cnoise$ that will be determined in our theorems.

\begin{itemize}
\item[(\textbf{A1})] The columns of $\bA^*$ are linearly independent.
\item[(\textbf{A2})] For all $i \in [\ntopic]$, $x_i^* \in [0, 1]$, $\E[x^*_i] \le \frac{\xexpc}{\ntopic} $ and $\frac{\xsl}{\ntopic} \le \E[(x^*_i)^2] \le \frac{\xsu}{\ntopic}$, and $x^*_i$'s are independent. 
\item[(\textbf{A3})] The initialization $\titime{\bA}{0} = \bA^* ( \titime{\bSigma}{0} + \titime{\bE}{0}) + \bN^{(0)}$, where $\titime{\bSigma}{0}$ is diagonal, $\titime{\bE}{0}$ is off-diagonal, and
\[
  \titime{\bSigma}{0} \succeq (1 -\ell) \bI, \quad \sym{\titime{\bE}{0}} \le \ell. 
\] 
\end{itemize}

We consider two noise models.
\begin{itemize}
\item[(\textbf{N1})]  Adversarial noise: only assume that $\max_i |\noise_i| \le \cnoise$ almost surely.
\item[(\textbf{N2})] Unbiased noise: $\max_i |\noise_i| \le \cnoise$ almost surely, and $\E[\noise | x^*] = 0$.
\end{itemize}

\usebox{\algopurification}

Our main algorithm is presented in Algorithm~\ref{alg:main_sim_noise}.  
It keeps a working feature matrix and operates in iterations. 
In each iteration, it first compute the weights for $N$ examples (decoding), and then use the computed weights to update the feature matrix (updating). 

The decoding is simply multiplying the example by the pseudo-inverse of the current feature matrix and then passing it through a one-sided threshold function $\phi_\alpha$.  
The pseudo-inverse with minimum infinity norm is used so as to maximize the robustness to noise (see the theorems). 
The one-sided threshold function operates element-wisely on the input vector $v$, and for an element $v_i$, it is defined as
\begin{align*}
   \phi_\alpha(v_i) & = \max\cbr{v_i - \alpha, 0}.
\end{align*}
This is just the rectified linear unit (ReLU) with offset $\alpha$. 
To get some sense about the decoding, suppose the current feature matrix is the ground-truth. Then 
$
  \bA^{\dagger} y = \bA^{\dagger} \bAg x^* + \bA^{\dagger} \noise = x^* + \bA^{\dagger} \noise.
$
So we would like to use a small $\bA^{\dagger}$ and use threshold to remove the noise term.

In the encoding step, the algorithm move the feature matrix along the direction $\E\left[ (y - y')(x - x')^{\top}\right]$. Suppose we have independent $x^*_i$'s, perfect decoding and no noise, then $\E\left[ (y - y')(x - x')^{\top}\right] = \bA^*$, and thus it is moving towards the ground-truth.  
Without those ideal conditions, we need to choose a proper step size, which is tune by the parameters $\eta$ and $r$. 

At the end, the algorithm simply outputs the scaled features with unit norm.  The output enjoys the following guarantee in the adversarial noise model.

\subsection{Analysis of one update step}
In this subsection, we focus on one update step, bounding the changes of $\bSigma, \bE, \bN$ and some auxiliary variables, and then in the next subsection we put things together to prove the theorem.
So through out this subsection we will focus on a particular iteration $t$ and omit the superscript $(t)$,  while in the next subsection we will put back the superscript.

For analysis, denote $\bA^{(t)}$ as
\begin{align*}
    \bA & = \bAg(\bSigma + \bE) + \bN
\end{align*}
where $\bSigma$ is a diagonal matrix, $\bE$ is an off-diagonal matrix, and $\bN$ is the component of $\bA$ that lies outside the span of $\bAg$ (e.g., the noise caused by the noise in the sample). 

Recall the following notations:
\begin{align*}
  \forder & = \left( \bSigma + \bE \right)^{-1},
	\\
  \gorder & = \forder - \bSigma^{-1} = \bSigma^{-1} \sum_{k = 1}^{\infty}( - \bE \bSigma^{-1})^k,
	\\
	\xi & = - \bA^\dagger \bN \forder x^* +  \bA^\dagger \noise.
\end{align*}

Consider the update term $\hat{\E} \left[(y - y')(x - x')^{\top}\right]$ and denote it as
\[
\bDelta =  \hat{\E} \left[(y - y')(x - x')^{\top}\right] = \bAg(\btSigma + \btE) + \btN
\] 
where $\btSigma$ is a diagonal matrix, $\btE$ is an off-diagonal matrix, and $\bN$ is the component of $\bDelta$ that lies outside the span of $\bAg$.

Since we now use empirical average,  we will have sampling noise. Denote it as
\[
 \bNs  = \hat{\E}[(y - y')(x - x')^{\top}]  - \E[(y - y')(x - x')^{\top}].
\]

Then by definition, for $y = \bAg x^* + \noise$ and $y' = \bAg (x')^* + \noise'$, we have 
\begin{align*}
\hat{\E}[(y - y')(x - x')^{\top}] & =  \E[(y - y')(x - x')^{\top}] + \bNs
\\
& =  \bAg ~\underbrace{ \E\sbr{(x^* - (x')^*) (x - x')^\top} }_{\btSigma + \btE} + \underbrace{\E\sbr{ (\noise - \noise') (x - x')^\top } + \bNs }_{\btN}. 
\end{align*}

Our goal is then bounding $\btSigma, \btE, \btN$ in terms of $\bSigma, \bE, \bN$. 
Before doing so, we present a lemma for the decoding. 

\begin{lem}[Main: Decoding] \label{lem:decoding} 
Let $m \ge \ntopic$ be two positive integers. 
Let $\bA \in \Real^{m \times \ntopic}$ be a matrix such that $\bA = \bAg(\bSigma  + \bE)  + \bN$ where $\bAg$ is full rank, $\bSigma$ is a diagonal matrix such that $\bSigma \succeq \frac{1}{2} \bI $ and $\| \bE \|_1 < \frac{1}{2}$. 
Then for $y = \bA^* x^* + \noise$, the decoding is
\begin{align*}
  x & = \phi_{\alpha}\rbr{ \forder x^* + \xi } 
	\\
  & = \phi_{\alpha}\rbr{ \rbr{\bSigma^{-1} + \gorder} x^* + \xi }.
\end{align*}
\end{lem}

\begin{proof}[Proof of Lemma~\ref{lem:decoding}]
Since $\bA = \bAg(\bSigma + \bE) + \bN$, we have
\begin{align*}
  \bAg & = (\bA - \bN)(\bSigma + \bE)^{-1}  
	\\
	y & = (\bA - \bN)(\bSigma + \bE)^{-1} x^* + \noise.
\end{align*}
Plugging into the decoding we get the first statement.
 
Observing that $\bSigma+ \bE =  (\bI +\bE \bSigma^{-1} ) \bSigma$ and $\| \bE \bSigma^{-1} \|_1 \le \|\bSigma^{-1} \|_1 \| \bE\|_1 \le 2 \| \bE \|_1 < 1$, we have $ (\bSigma+ \bE)^{-1} = \rbr{\bSigma^{-1} + \gorder}$, resulting in the second statement.
\end{proof}

\begin{lem}[Main: Bound on $\btSigma$] \label{lem:update_diag}
Suppose $\abr{\xi_i} \le \nthres < \alpha$ for any example and every $i \in [\ntopic]$, and suppose $\bSigma \succeq \frac{1}{2} \bI$. 
Then for any $i \in [\ntopic]$, 
\begin{align*}
  \btSigma_{i, i} & \ge 
	  \E\sbr{\rbr{x_i^*}^2} \rbr{ 2 \bSigma^{-1}_{i, i} - 2 \abr{\gorder_{i,i} } } 
		- \frac{2\xexpc}{\ntopic} \rbr{  \alpha +2\nthres  + \frac{\xexpc}{\ntopic} \bSigma^{-1}_{i, i} + \frac{2\xexpc}{\ntopic}  \nbr{ \sbr{ \gorder}^i  }_1   }, 
 \\
  \btSigma_{i, i } & \le  
	  \E\sbr{\rbr{x_i^*}^2} \rbr{ 2 \bSigma^{-1}_{i, i} + 2 \abr{\gorder_{i,i} } } 
		+ \frac{2\xexpc}{\ntopic} \rbr{   \nthres  + \frac{\xexpc}{\ntopic}  \nbr{ \sbr{ \gorder}^i  }_1  }.
\end{align*}
\end{lem}

\begin{proof}[Proof of Lemma \ref{lem:update_diag}]
According to the definition, we have
\begin{align*}
  \btSigma_{i, i} & = 
	\left[(\bAg)^{\dagger}\E[(y - y')(x - x')^{\top}]\right]_{i,i} 
	\\
  & = \E\left[(x_i^* - (x_i')^*)(x_i - x'_i) \right]
	\\
  & = \E\left[(x_i^*  - (x_i')^*) x_i \right] + \E\left[((x_i')^* - x_i^*)x_i'  \right].
\end{align*}
Since $(x_i^*  - (x_i')^*) x_i$ and $((x_i')^* - x_i^*)x_i'$ has the same distribution, and $(x')^*, x^*$ are i.i.d.\ , we have
\begin{align*}
  \btSigma_{i, i} & = 2 \E\left[(x_i^*  - (x_i')^*) x_i \right] 
	\\
	&= 2 \E[x_i^* x_i ] - 2\E[x_i^*] \E[x_i].
\end{align*}
So it suffices to bound $\E[x_i^* x_i ]$ and $\E[x_i]$. 
To do so, we first take a look at $x_i$.  
By the decoding rule, 
\begin{align*}
  x_i & = \left[\phi_{\alpha}\rbr{ \rbr{ \bSigma^{-1}   +  \gorder } x^*  + \xi} \right]_i.
\end{align*}
Since $\phi_{\alpha}$ is 1-Lipschitz, denoting $\Delta = \left| \left[ \gorder x^* \right]_i  + \xi_i \right|$ we have
\begin{align} \label{eqn:sigma:1}
  \left[\phi_{\alpha}\left( \bSigma^{-1}  x^* \right)\right]_i - \Delta 
	& \le x_i \le 
	\left[\phi_{\alpha} \left( \bSigma^{-1}  x^* \right)\right]_i + \Delta.
\end{align}
For $\left[\phi_{\alpha}\left( \bSigma^{-1}  x^* \right)\right]_i$, by the Property~\ref{prop:relu} of $\phi_{\alpha}(z)$, 
\begin{align} \label{eqn:sigma:2}
 \bSigma_{i, i}^{-1}  x^*_i  - \alpha \le \left[\phi_{\alpha}\left( \bSigma^{-1}  x^* \right)\right]_i = \phi_{\alpha}\left( \bSigma_{i, i}^{-1}  x^*_i \right) \le \bSigma_{i, i}^{-1}  x^*_i.
\end{align}
For $\Delta = \left| \left[ \gorder x^* \right]_i  + \xi_i \right|$, 
\begin{align} 
  \E[\Delta] 
	& \le \E\sbr{ \left| \sum_{j} \gorder_{i,j} x^*_j  \right|}  + \E\sbr{ |\xi_i| } \nonumber 
	\\
  & \le \E\sbr{  \sum_{j} \left|\gorder_{i,j} \right| x^*_j  }  + \nthres \nonumber  
	\\
  & = \sum_{j} \left|\gorder_{i,j} \right| \E\sbr{  x^*_j  } + \nthres \nonumber 
	\\
	& \le \frac{\xexpc}{\ntopic} \nbr{ \sbr{\gorder}^i}_1 + \nthres \label{eqn:sigma:3}
\end{align}
where the second step follows from the assumption $\abr{\xi_i} \le \nthres$, and the last step follows from Assumption \textbf{(A2)}.

\textbf{Bounding $\E[x_i]$.} By \eqnref{eqn:sigma:1},\eqnref{eqn:sigma:2}, and \eqnref{eqn:sigma:3}, we have
\[
  \E[x_i] \le \E[\bSigma^{-1}_{i, i}  x^*  ] + \E[\Delta] \le \frac{\xexpc}{\ntopic} \bSigma^{-1}_{i, i}  +   \frac{\xexpc}{\ntopic} \nbr{ \sbr{\gorder}^i}_1 + \nthres.
\]

\textbf{Bounding $\E[x_i^* x_i]$.} 
First, note that  
\begin{align*}
  \E[x_i^* \Delta] 
	& \le \E \sbr{  x_i^* \abr{ \sum_j \gorder_{i,j} x_j^*}  }  + \E \sbr{  x_i^* \abr{ \xi_i  } }
	\\
  & \le \E \sbr{  x_i^*  \sum_j x_j^*\abr{\gorder_{i,j} }  } + \frac{\nthres  \xexpc}{\ntopic}
	\\
	& = \sum_j\E \sbr{  x_i^*   x_j^*  }  \abr{\gorder_{i,j} } + \frac{\nthres  \xexpc}{\ntopic}
	\\
	& = \E \sbr{ \rbr{x_i^*}^2 }  \abr{\gorder_{i,i} }  + \sum_{j: j\neq i}\E \sbr{  x_i^*   x_j^*  }  \abr{\gorder_{i,j} } + \frac{\nthres  \xexpc}{\ntopic}
	\\
  & \le \E \sbr{ \rbr{x_i^*}^2 }  \abr{\gorder_{i,i} }  + \frac{\xexpc^2}{\ntopic^2}  \sum_{j: j\neq i}  \abr{\gorder_{i,j} } + \frac{\nthres  \xexpc}{\ntopic}
	\\
	& \le \E \sbr{ \rbr{x_i^*}^2 }  \abr{\gorder_{i,i} }  + \frac{\xexpc^2}{\ntopic^2}  \nbr{ \sbr{ \gorder}^i  }_1 + + \frac{\nthres  \xexpc}{\ntopic},
\end{align*}
where the second and the fifth steps follow from Assumption \textbf{(A2)}.
Therefore,
\begin{align}
  \E[x_i^* x_i] 
	& \ge  \E\left[x_i^* \left( \bSigma^{-1}_{i, i}  x_i^*  - \alpha - \Delta \right) \right]
	\\
	& \ge \bSigma^{-1}_{i, i}\E\sbr{\rbr{x_i^*}^2} - \frac{(\alpha + \nthres) \xexpc}{\ntopic} - \E \sbr{ \rbr{x_i^*}^2 }  \abr{\gorder_{i,i} }  - \frac{\xexpc^2}{\ntopic^2}  \nbr{ \sbr{ \gorder}^i  }_1.
	\label{eqn:diag_update1}
\end{align}

\textbf{Putting together.} For the first statement,  
\begin{align*}
  \btSigma_{i, i} 
	& = 2 \E[x_i^* x_i ] - 2\E[x_i^*] \E[x_i] 
	\\
	& \ge 2 \bSigma^{-1}_{i, i}\E\sbr{\rbr{x_i^*}^2} - 2\frac{(\alpha + \nthres) \xexpc}{\ntopic} - 2\E \sbr{ \rbr{x_i^*}^2 }  \abr{\gorder_{i,i} }  -  2\frac{\xexpc^2}{\ntopic^2}  \nbr{ \sbr{ \gorder}^i  }_1 
	\\
  & - 2 \frac{\xexpc^2}{\ntopic^2} \bSigma^{-1}_{i, i}  -  2 \frac{\xexpc^2}{\ntopic^2} \nbr{ \sbr{\gorder}^i}_1 - 2 \frac{\nthres \xexpc}{\ntopic}
	\\
  & \ge \E\sbr{\rbr{x_i^*}^2} \rbr{ 2 \bSigma^{-1}_{i, i} - 2 \abr{\gorder_{i,i} } } - \frac{2\xexpc}{\ntopic} \rbr{  \alpha + 2\nthres + \frac{\xexpc}{\ntopic} \bSigma^{-1}_{i, i} + \frac{2\xexpc}{\ntopic}  \nbr{ \sbr{ \gorder}^i  }_1   }.
\end{align*}

The second statement follows from 
\[
  \btSigma_{i, i} \le 2 \E[x_i^* x_i ] \le 2 \E[x_i^* (\bSigma^{-1}_{i, i}  x_i^*  + \Delta) ]
\]
and the bound on $\E[x_i^* \Delta]$.
\end{proof}

\begin{lem}[Main: Bound on $\btE$] \label{lem:update_E} \label{lem:better_bound_noise} 
Suppose $\abr{\xi_i} \le \nthres < \alpha$ for any example and every $i \in [\ntopic]$. 
Then for all $i, j \in [\ntopic]$ such that $i \not= j$, the following holds. \\
(1) If $\forder_{i, j} < 0$, then 
\[
  \left| \btE_{j, i} \right| \le  \frac{4 \xexpc^2  \| \forder^i \|_1 }{\ntopic^2 (\alpha - \nthres)}  \rbr{ \left| \forder_{i, j} \right| + \nthres}.
\]
(2) If $\forder_{i, j} \ge 0$, then 
\[
  -\frac{8\xexpc \nthres}{\ntopic (\alpha - \nthres)} \rbr{ \frac{\xexpc \| \forder^i \|_1 }{n} + \forder_{i, j}} - \frac{2 \xexpc^2}{\ntopic^2} \forder_{i, j}
  \le
  \btE_{j, i}  
	\le 
	\frac{8\xexpc \nthres}{\ntopic (\alpha - \nthres)} \rbr{ \frac{\xexpc \| \forder^i \|_1 }{n} + \forder_{i, j}} + 2 \E[(x_j^*)^2] \forder_{i, j}.
\]
\end{lem}

\begin{proof}[Proof of Lemma \ref{lem:update_E}]
Since $i \not= j$, we know that 
\begin{align*}
  \btE_{j, i} 
	& = \E\left[(x_j^* - (x_j')^*)(x_i - x'_i) \right]
	\\ 
	& = \E\left[x_j^* (x_i - x'_i) \right] + \E\left[(x_j')^*(x_i' - x_i) \right]
	\\ 
	& = 2 \E\left[x_j^* (x_i - x'_i) \right]
\end{align*}
where the last equality follows from that $x_j^* (x_i - x'_i)$ and $(x_j')^*(x_i' - x_i)$ has the same distribution.
This quantity can be bounded by a coupling between $x_i$ and  $x_i'$. Define a new variable $\tilde{x}^*$ as
\begin{align*}
  [\tilde{x}^*]_i =  
	\begin{cases}
	  x_i^*, & \textnormal{~if~} i \not= j,
		\\
	  (x_j')^*, & \textnormal{~if~} i = j.
  \end{cases}
\end{align*}

By Assumption \textbf{(A2)}, conditional on $x_j^*$, $\tilde{x}^*$ has the same distribution as $(x')^*$. Therefore, consider the variable $\tilde{x}$ given by $\tilde{x} = \phi_{\alpha} (\bA^{\dagger} (\bAg \tilde{x}^* + \noise'))$, we then have 
\begin{align*}
  \E\left[x_j^* (x_i - x'_i) \right] & = \E\left[x_j^* (x_i - \tilde{x}_i) \right]. 
\end{align*}

In summary, we have
\[
  \btE_{j, i} = 2 \E[x_j^{*} (x_i - \tilde{x}_i)]
\]
where 
\begin{align*}
  x_i & = \left[\phi_{\alpha} \left( \forder x^* + \xi \right) \right]_i, ~~\xi = - \bA^\dagger \bN  \forder x^* +  \bA^\dagger \noise,
	\\
	\tilde{x}_i & = \left[\phi_{\alpha} \left( \forder x^* + \tilde{\xi} \right) \right]_i,~~\tilde{\xi} = - \bA^\dagger \bN  \forder \tilde{x}^* +  \bA^\dagger \noise'.
\end{align*}
Introduce the notation
\[
  w =  \forder_{i, i} x_i^* +  \sum_{l \not= i, j}  \forder_{i, l} x^*_l.
\] 
We have
\begin{align*}
  x_i & = \phi_{\alpha}\left( w +  \forder_{i, j} x_j^*  + \xi_i \right),  
	\\
  \tilde{x}_i & = \phi_{\alpha}\left( w +  \forder_{i, j} (x_j')^* + \tilde{\xi}_i \right).
\end{align*}

(1) Since $\forder_{i, j} < 0$, $|\xi_i| \le \nthres$, and $|\tilde{\xi}_i| \le \nthres$,
we know that when $w < \alpha  - \nthres $, $x_i  = \tilde{x}_i = 0$. Then
\begin{align}
  \E\sbr{x_j^*  (x_i - \tilde{x}_i) } & = \Pr[w \geq \alpha - \nthres] ~\E\sbr{x_j^*   (x_i - \tilde{x}_i) | w \geq \alpha - \nthres}.   \label{eqn:better_noise1}
\end{align}

By Property~\ref{prop:relu}, $\thres{\alpha}{\cdot}$ is 1-Lipschitz, so 
\[
  | x_i - \tilde{x}_i | \le \abr{ \forder_{i, j} }  \left|  x_j^* - (x'_j)^*\right| + \abr{ \xi_i - \tilde{\xi}_i },
\]
which implies that 
\begin{align} 
  \abr{\E\sbr{x_j^*  (x_i - \tilde{x}_i) | w \geq \alpha - \nthres} } 
	& \le  \E\sbr{x_j^* \abr{  \forder_{i, j} } \left|  x_j^* - (x'_j)^*\right|  + x_j^*  \abr{ \xi_i - \tilde{\xi}_i } \bigg | w \geq \alpha - \nthres}     \nonumber
  \\
  & \le  \abr{  \forder_{i, j} } \max\cbr{\left|  x_j^* - (x'_j)^*\right| }  \E\sbr{x_j^* | w \geq \alpha - \nthres}  + 2 \nthres \E\sbr{ x_j^*  | w \geq \alpha - \nthres} \nonumber
	\\
  & \le  \abr{  \forder_{i, j} } \max\cbr{\left|  x_j^* - (x'_j)^*\right| }  \E\sbr{x_j^*}   +  2 \nthres \E\sbr{ x_j^*} \nonumber
	\\
  & \le  2  \E\sbr{x_j^*}   \rbr{ \abr{  \forder_{i, j} }  + \nthres} \nonumber
	\\
  & \le \frac{  2  \xexpc}{\ntopic}  \rbr{ \abr{  \forder_{i, j}}   + \nthres}. \label{eqn:better_noise2}
\end{align}

Now consider $\Pr[w \geq \alpha - \nthres]$. Since
\begin{align*} 
  \E\abr{w} 
	&\le  \abr{ \forder_{i, i}} \E[x^*_i] +  \sum_{l \neq i, j}  \abr{ \forder_{i, l} } \E[x^*_j]  
  \le \frac{ \xexpc}{\ntopic} \| \forder^i \|_1,
\end{align*}
we have that
\begin{align} 
  \Pr[w \geq \alpha - \nthres] 
	& \le \frac{\E \abr{w}}{ \alpha - \nthres} 
	\le  \frac{ \xexpc \| \forder^i \|_1 }{\ntopic (\alpha - \nthres)} \label{eqn:better_noise3}
\end{align}
Combining \eqnref{eqn:better_noise1}\eqnref{eqn:better_noise2} and \eqnref{eqn:better_noise3} together completes the proof for the case when $\forder_{i, j} < 0$.

(2) Now consider the case when $\forder_{i, j} \ge 0$.
Again, we have
\begin{align*}
  x_i & = \phi_{\alpha}\left( w +  \forder_{i, j} x_j^*  + \xi_i \right),  
	\\
  \tilde{x}_i & = \phi_{\alpha}\left( w +  \forder_{i, j} (x_j')^* + \tilde{\xi}_i \right).
\end{align*}

For the analysis, introduce a variable
\begin{align*}
  \tilde{u}_i & = \phi_{\alpha}\left( w +  \forder_{i, j} x_j^* + \tilde{\xi}_i  \right).
\end{align*}

If $(x_j')^* > x_j^*$, by Property~\ref{prop:relu} $\phi_\alpha(\cdot)$ is 1-Lipschitz, so
\begin{align*}
 \tilde{x}_i & \le \tilde{u}_i + \forder_{i,j} \rbr{ (x'_j)^* - x_j^*}.
\end{align*}

If $(x_j')^* \le x_j^*$, by Property~\ref{prop:relu} $\phi_\alpha(\cdot)$ is non-decreasing, then 
\[
  \tilde{x}_i \le \tilde{u}_i.
\]
In any case, 
\begin{align*}
 \tilde{x}_i & \le \tilde{u}_i + \forder_{i,j} (x'_j)^*.
\end{align*}
Therefore, 
\begin{align*}
  \E\sbr{ x_j^* (x_i - \tilde{x}_i) } 
	& \ge \E\sbr{ x_j^* (x_i - \tilde{u}_i) }  - \E \sbr{ x_j^*\forder_{i,j} (x'_j)^* } 
	\\
	& \ge \E\sbr{ x_j^* (x_i - \tilde{u}_i) }  - \frac{\xexpc^2}{\ntopic^2}\forder_{i,j}.
\end{align*}

So we only need to consider $\E\sbr{ x_j^* (x_i - \tilde{u}_i) }$. 
Let $\overth$ denote the event that $x_i\neq 0$ or $\tilde{u}_i \neq 0$.
Then by conditioning on $x_j^*$, we have
\begin{align*}
  \E\sbr{ x_j^* (x_i - \tilde{u}_i) } & = \E\sbr{ x_j^* \E\sbr{ x_i - \tilde{u}_i \bigg| x_j^*}} 
\end{align*}
and 
\begin{align*}
	\E\sbr{ x_i - \tilde{u}_i \bigg| x_j^*} 
	& = \Pr\sbr{ \overth \bigg | x_j^* }  \E\sbr{ x_i - \tilde{u}_i \bigg| x_j^*, \overth }.
\end{align*}
By Property~\ref{prop:relu} $\phi_\alpha(\cdot)$ is 1-Lipschitz, so
\begin{align*}
  \abr{ \E\sbr{ x_i - \tilde{u}_i \bigg| x_j^*, \overth } } 
	& \le   
\E\sbr{   \abr{\xi_i} + \abr{\tilde{\xi}_i}\bigg| x_j^*, \overth } 
  \le   2 \nthres.
\end{align*}
Now consider $\Pr\sbr{ \overth \bigg | x_j^* }$. We have 
\begin{align*}
  \E\sbr{ \abr{ w + \forder_{i,j} x_j^*} \bigg | x_j^* } 
	& \le  \E\sbr{ \abr{w}  \bigg | x_j^* } + \forder_{i,j}
	\\ 
  & \le \frac{\xexpc}{\ntopic} \nbr{\forder^i}_1 + \forder_{i,j},
\end{align*}
where the first step follows from $ x_j^* \le 1$ and the second step follows from the conditional independence in Assumption \textbf{(A2)}.
Then by Markov's inequality, 
\begin{align*}
  \Pr\sbr{ x_i \neq 0 \bigg | x_j^* }  
	& \le \Pr\sbr{ \abr{ w + \forder_{i,j} x_j^*} \ge  \alpha - \nthres \bigg | x_j^* }  
	\\
  &\le \frac{1}{\alpha - \nthres} \rbr{ \frac{\xexpc}{\ntopic} \nbr{\forder^i}_1 + \forder_{i,j}  }.
\end{align*}

A similar argument leads to that
\begin{align*}
  \Pr\sbr{ \tilde{u}_i \neq 0 \bigg | x_j^* }	& \le \frac{1}{\alpha - \nthres} \rbr{ \frac{\xexpc}{\ntopic} \nbr{\forder^i}_1 + \forder_{i,j}  }
\end{align*}
and thus
\begin{align*}
  \Pr\sbr{ \overth \bigg | x_j^* } 
  &\le \frac{2}{\alpha - \nthres} \rbr{ \frac{\xexpc}{\ntopic} \nbr{\forder^i}_1 + \forder_{i,j}  }.
\end{align*}

Putting things together, 
\begin{align*}
  \abr{\E\sbr{x_j^* (x_i - \tilde{u}_i) } } 
	& \le 
		\frac{4\nthres}{\alpha - \nthres}  \rbr{ \frac{\xexpc \| \forder^i \|_1 }{n} + \forder_{i,j}   }  \E \sbr{ x_j^* }
	\\
	& \le
	\frac{4 \xexpc \nthres}{\ntopic(\alpha - \nthres)}   \rbr{ \frac{\xexpc \| \forder^i \|_1 }{\ntopic} +  \forder_{i, j}}.
\end{align*}
This completes the proof for the lower bound. 

Similarly, for the upper bound, introduce 
\begin{align*}
  u_i & = \phi_{\alpha}\left( w +  \forder_{i, j} (x'_j)^* + \xi_i  \right).
\end{align*}
Then in any case, 
\begin{align*}
  x_i & \le u_i + \forder_{i,j} x_j^*
\end{align*}
and thus
\begin{align*}
  \E\sbr{ x_j^* (x_i - \tilde{x}_i) } & \le \E\sbr{ x_j^* (u_i - \tilde{x}_i) }  +  \E \sbr{  (x_j^*)^2 } \forder_{i,j}.
\end{align*}
The same argument as above shows that
\begin{align*}
  \abr{\E\sbr{x_j^* (u_i - \tilde{x}_i) } } 
	& \le
	\frac{4 \xexpc \nthres}{\ntopic(\alpha - \nthres)}   \rbr{ \frac{\xexpc \| \forder^i \|_1 }{\ntopic} +  \forder_{i, j} }.
\end{align*}
This completes the whole proof.
\end{proof}

\begin{lem}[Main: Bound on $\btN$] \label{lem:bound_error}
Suppose $\sym{\bE} \le \ell$, $ \bSigma \succeq (1-\ell)\bI$, and $|\xi_j| \le \nthres < \alpha$. 

(1) If the noise is correlated (Assumption (\textbf{N1})), then
\[
  \abr{ \btN_{i,j} } \le \frac{4 \cnoise \xexpc }{(1 - 2 \ell)^2 n  (\alpha - \rho)} + \abr{ [\bN_s]_{i,j} }
\] 

(2) If the noise is unbiased (Assumption (\textbf{N2})) and $\|\bA^{\dagger} \nu \|_{\infty} \le \rho' < \alpha$, then 
\[
  \abr{\btN_{i,j} } \le \frac{2 \xexpc \cnoise \rho' (1 + \| \bA^{\dagger} \bN \|_{\infty})}{(1 - 2 \ell)  n (\alpha - \rho')}  + \abr{ \sbr{\bN_s}_{i,j} }.
\]
\end{lem}  

\begin{proof}[Proof of Lemma \ref{lem:bound_error}]
(1) By the update rule, 
\[
  \btN = 2 \E[\nu (x - x')^{\top}] + \bN_s.
\]

Under Assumption (\textbf{N1}), we have that for every $i \in [\ntopic], j \in [\ntopic]$, 
\begin{align*}
  |\btN_{i, j} | 
	& = |2\E[\nu_i (x_j - x_j')] + [\bN_s]_{i, j}| 
	\\
	& \le  4 \cnoise \E[x_j] + | [\bN_s]_{i, j}|
	\\
	& =  4 \cnoise  \E\left[ \phi_{\alpha} \left([\bZ x^*]_j + \xi_j \right)\right] + | [\bN_s]_{i, j}|.
\end{align*}
since $ |\nu_i| $ is bounded by $\cnoise$. 

Now focus on the term $\E\left[ \phi_{\alpha} \left([\bZ x^*]_j + \xi_j \right)\right] $. We have
\begin{align*} 
  \abr{ [\bZ x^*]_j } \le \| \bZ\|_{\infty} \| x^* \|_{\infty} \le  \| \bZ\|_{\infty} \le \frac{1}{1 - 2 \ell}
\end{align*}
by the fact that $\| x^* \|_{\infty} \le 1$ in Assumption \textbf{(A2)}, and the assumptions of the lemma on $\bSigma$ and $\bE$.  
Then when $[\bZ x^*]_j + \xi_j \ge \alpha$, 
\begin{align*}
	\phi_{\alpha} \left([\bZ x^*]_j + \xi_j \right) & \le [\bZ x^*]_j + \xi_j  -\alpha \le \frac{1}{1 - 2 \ell} + \rho - \alpha \le \frac{1}{1 - 2 \ell},
\end{align*}
and thus
\begin{align*}
  \E \sbr{ \phi_{\alpha} \left([\bZ x^*]_j + \xi_j \right) }  
	& \le \frac{1}{1 - 2 \ell} \Pr \cbr{ [\bZ x^*]_j + \xi_j  \ge \alpha }
	\\
	& \le \frac{1}{1 - 2 \ell} \Pr \cbr{ \abr{[\bZ x^*]_j } \ge \alpha - \rho }
	\\
  & \le \frac{1}{1 - 2 \ell} \frac{ \E { | [\bZ x^*]_j | } }{\alpha - \rho } 
	\\
  & \le \frac{1}{1 - 2 \ell} \frac{ \nbr{\bZ}_\infty }{\alpha - \rho } ~ \E \sbr{ x^*_j }
	\\
  & \le \frac{\xexpc}{(1 - 2 \ell)^2 \ntopic (\alpha - \rho)}
\end{align*}
where the last step uses the bound on $\E \sbr{ x^*_j }$ in Assumption \textbf{(A2)}.
Therefore, 
\[
  |\btN_{i, j} | \le  \frac{4 \cnoise \xexpc }{(1 - 2 \ell)^2 \ntopic (\alpha - \rho)} + | [\bN_s]_{i, j}|. 
\]

(2) When the noise is unbiased, we have $\E[\nu|x^*]  = 0$. Then $ \E[\nu_i x_j']  = 0 $, and 
\begin{align} \label{eqn:noise:1}
  \abr{ \btN_{i, j} } 
	& = \abr{ 2\E[\nu_i (x_j - x_j')] + [\bN_s]_{i, j} }
	\le 2 \abr{ \E[\nu_i x_j] } +  \abr{ [\bN_s]_{i, j} }.
\end{align}
Consider the first term for a fixed $x^*$, i.e., consider the conditional expectation $ \E[\nu_i x_j \mid x^*] $.
For notational simplicity, let $\btZ = (\bZ - \bA^{\dagger} \bN \bZ) $ and $\txi = \bA^{\dagger} \nu$. Then
\begin{align*}
  \E[\nu_i x_j \mid x^*]
  & = \E \sbr{   \nu_i \phi_{\alpha} \left([\bZ x^*]_j + \xi_j \right) \mid x^* }  
	= \E \sbr{ \nu_i \phi_{\alpha} \left([\btZ x^*]_j + \txi_j \right)  \mid x^* }.
\end{align*}

We consider the following two cases about $[\btZ x^*]_j$. 
\begin{itemize}
\item[(a)] If $[\btZ x^*]_j \le \alpha - \rho'$, then $\phi_{\alpha} \left([\btZ x^*]_j + \txi_j \right) = 0$ always holds, which implies that 
\[
  \abr{ \E[\nu_i x_j \mid x^*] } =  \E\left[\nu_i \phi_{\alpha} \left([\btZ x^*]_j + \txi_j \right) \mid x^* \right]  = 0. 
\]

\item[(b)] If $[\btZ x^*]_j > \alpha - \rho'$, then 
\[
  \phi_{\alpha} \left([\btZ x^*]_j + \txi_j \right) \le \phi_{\alpha} \left([\btZ x^*]_j + \rho' \right) \le [\btZ x^*]_j  + \rho' - \alpha.
\]

On the other side, by Property~\ref{prop:relu}, 
\[
  \phi_{\alpha} \left([\btZ x^*]_j + \txi_j \right)  \ge [\btZ x^*]_j + \txi_j - \alpha \ge [\btZ x^*]_j  - \rho' - \alpha.
\]

Putting together, we conclude that
\[
  \nu_i ( [\btZ x^*]_j  - \alpha) -| \nu_i \rho'| \le \nu_i \phi_{\alpha} \left([\btZ x^*]_j + \txi_j \right) \le \nu_i ( [\btZ x^*]_j  - \alpha) +| \nu_i \rho'|.
\]
 
Note that $\E[ \nu_i ( [\btZ x^*]_j  - \alpha) | x^* ] = 0$, so 
\[
  \abr{ \E[\nu_i x_j \mid x^*] } =  \left| \E \left[ \nu_i \phi_{\alpha} \left([\btZ x^*]_j + \txi_j \right) \mid x^*  \right] \right| \le \E[| \nu_i \rho'| | x^* ]\le \cnoise \rho'.
\]
\end{itemize}

Putting case (a) and case (b) together, we have
\begin{align*}
  \abr{ \E[\nu_i x_j \mid x^*] } 
	& \le  \cnoise \rho' \Pr\cbr{ [\btZ x^*]_j > \alpha - \rho' }
	\le  \cnoise \rho' \Pr\cbr{ \abr{ [\btZ x^*]_j } > \alpha - \rho' }.
\end{align*}
By definition of $\btZ$ and the assumptions of the lemma on $\bSigma$ and $\bE$, 
\begin{align} \label{eqn:noise:2}
  \abr{[\btZ x^*]_j} \le (1 + \| \bA^{\dagger} \bN \|_{\infty}) \abr{ [\bZ x^*]_j } \le (1 + \| \bA^{\dagger} \bN \|_{\infty}) \abr{\bZ}_\infty x^*_j  \le \frac{1 + \| \bA^{\dagger} \bN \|_{\infty}}{1 - 2 \ell} x^*_j. 
\end{align}
Then 
\begin{align*}
  \Pr\cbr{ \abr{ [\btZ x^*]_j } > \alpha - \rho' } 
	& \le \frac{\E \abr{ [\btZ x^*]_j } } {\alpha - \rho'}
	\le \frac{\xexpc (1 + \| \bA^{\dagger} \bN \|_{\infty}) }{(1 - 2 \ell)n (\alpha - \rho')}.
\end{align*}
The lemma then follows from \eqnref{eqn:noise:1} and \eqnref{eqn:noise:2}.
\end{proof}

There are three terms $\forder$, $\gorder$ and $\xi$ in the above lemmas that need to be bounded.  Since $\forder = \gorder + \bSigma^{-1}$, we only need to bound $\gorder$ and $\xi$ in the following two lemmas, respectively.

\begin{lem}[Bound on $\gorder$] \label{lem:higher_order}
Suppose $\sym{\bE} < \ell_e$ and $ \bSigma \succeq (1 - \ell)\bI$.
Then 
\begin{flalign*}
  (1)~ & \sym{\posmatrix{\gorder}} \le \frac{1 - \ell_e}{(1 - \ell)(1 - \ell_e - \ell)} \sym{ \bE_-  } +  \frac{\ell}{(1 - \ell)^2(1 - \ell_e - \ell)} \sym{\bE_+}, 
	\\
  (2)~ & \sym{\negmatrix{\gorder}} \le \frac{1 - \ell_e}{(1 - \ell)(1 - \ell_e - \ell)} \sym{ \bE_+  } +  \frac{\ell}{(1 - \ell)^2(1 - \ell_e - \ell)} \sym{\bE_-},
	\\
  (3)~ & \sym{\gorder} \le \frac{\ell_e (1-\ell_e)}{(1 - \ell)^2 (1 - \ell_e - \ell)}, 
	\\
	(4)~ & \abr{\gorder_{i,i}} \le \frac{\ell \ell_e}{(1 - \ell)^2(1 - \ell_e - \ell)} , ~~\forall i \in [\ntopic]. &&
\end{flalign*}
\end{lem}

\begin{proof}[Proof of Lemma~\ref{lem:higher_order}]
Denote $\horder = \bSigma^{-1} \sum_{k = 2}^{\infty}( - \bE \bSigma^{-1})^k$, so that
\[
	\gorder = -\bSigma^{-1} \bE \bSigma^{-1} + \horder.
\]
The following bound on $\nbr{\horder}_1$ will be useful.   
\begin{align}
  \nbr{\horder}_1
	& \le  \nbr{ \bSigma^{-1} }_1 \sum_{k = 2}^{\infty} \nbr{ (\bE \bSigma^{-1})^k  }_1 
	\nonumber \\
	& \le   \nbr{ \bSigma^{-1} }_1 \sum_{k = 2}^{\infty} \nbr{ \bE \bSigma^{-1} }_1^k
	\nonumber \\
	& \le   \nbr{ \bSigma^{-1} }_1 \frac{ \nbr{ \bE \bSigma^{-1} }^2_1 }{  1- \nbr{ \bE \bSigma^{-1} }_1 }
	\nonumber \\
	& \le   \frac{1}{(1 - \ell)^3} \times \ell \times \frac{ \nbr{ \bE  }_1 }{ 1- \frac{\ell_e}{1 - \ell} } 
	\nonumber \\
	&\le   \frac{\ell}{(1 - \ell)^2(1 - \ell_e - \ell)} \nbr{\bE}_1. \label{eqn:v:1}
\end{align}

(1) We need to show the bound for both $\nbr{ \posmatrix{\gorder}}_1$ and $\nbr{\posmatrix{\gorder}}_\infty$. 
By definition of $\gorder$, for any $i$, 
\begin{align*}
  \nbr{ \posmatrix{\gorder}}_1 & =  \nbr{\sbr{ -\bSigma^{-1} \bE \bSigma^{-1}  + \horder}_+  }_1.
\end{align*}
Since for any $\bA$ and $\bB$, 
\[
  \nbr{ [\bA + \bB]_+ }_1 \le \nbr{ [\bA]_+}_1 + \nbr{[\bB]_+ }_1, \textnormal{~and~} \nbr{ [\bA]_+}_1 \le \nbr{\bA}_1,
\]
we have
\begin{align}
  \nbr{ \sbr{\posmatrix{\gorder}}_i}_1 
	& \le \nbr{ \sbr{-\bSigma^{-1} \bE \bSigma^{-1}  }_+ }_1  + \nbr{ \horder_+ }_1 
	\nonumber \\
	& \le \frac{1}{(1 - \ell)^2}\nbr{ \bE_-}_1 + \nbr{\horder }_1.  \label{eqn:v:2}
\end{align}
By \eqnref{eqn:v:1},	
\begin{align*}
  \nbr{\horder}_1 \le \frac{\ell}{(1 - \ell)^2(1 -  \ell_e - \ell)}\nbr{ \bE}_1 \le \frac{\ell}{(1 - \ell)^2(1 - \ell_e - \ell)}(\nbr{ \bE_-  }_1 + \nbr{\bE_+}_1). 
\end{align*}
Combined with \eqnref{eqn:v:2}, it implies
\begin{align*}
  \nbr{ \sbr{\posmatrix{\gorder}}_i}_1  
	& \le \frac{1 - \ell_e}{(1 - \ell)^2(1 - \ell_e - \ell)} \nbr{ \bE_-  }_1 + \frac{\ell}{(1 - \ell)^2 (1 - \ell_e - \ell)}\nbr{\bE_+}_1.
\end{align*}
Similarly, we have 
\begin{align*}
  \nbr{ \sbr{\posmatrix{\gorder}}^i}_1 & \le \frac{1 - \ell_e}{(1 - \ell)^2 (1 - \ell_e - \ell)}  \nbr{ \bE_-  }_\infty +  \frac{\ell}{(1 - \ell)^2(1 - \ell_e - \ell)}\nbr{\bE_+}_\infty.
\end{align*}

Putting things together we have
\begin{align*}
  \sym{\posmatrix{\gorder}} & \le \frac{1 - \ell_e}{(1 - \ell)(1 - \ell_e - \ell)} \sym{ \bE_-  } + \frac{\ell}{(1 - \ell)^2 (1 - \ell_e - \ell)} \sym{\bE_+}.
\end{align*}

(2) The argument for $\sym{\negmatrix{\gorder}} $ is similar to that for $\sym{\posmatrix{\gorder}} $. 

(3) We need to show the bound for both $\nbr{ \gorder}_1$ and $\nbr{\gorder}_\infty$. 
\begin{align*}
  \nbr{ \gorder}_1
	& \le \nbr{ -\bSigma^{-1} \bE \bSigma^{-1}  }_1  + \nbr{\horder }_1
	\\
	& \le \frac{\ell_e}{(1 - \ell)^2}  + \frac{\ell}{(1 - \ell)^2(1 - \ell_e - \ell)} \nbr{\bE}_1
  \\
	& \le \frac{\ell_e}{(1 - \ell)^2}  + \frac{\ell \ell_e}{(1 - \ell)^2(1 - \ell_e - \ell)} 
	\\
  & = \frac{\ell_e (1-\ell_e)}{(1 - \ell)^2 (1 - \ell_e - \ell)}
\end{align*}
where the second step is by \eqnref{eqn:v:1}. 

Similarly, $\nbr{ \gorder }_\infty  \le  \frac{\ell_e (1-\ell_e)}{(1 - \ell)^2 (1 - \ell_e - \ell)}$,  so $\sym{\gorder}  \le  \frac{\ell_e (1-\ell_e)}{(1 - \ell)^2 (1 - \ell_e - \ell)}$.

(4) Now consider $\gorder_{i,i}$. By definition of $\horder$. 
\begin{align*}
  \gorder_{i,i} & = \sbr{ -\bSigma^{-1} \bE \bSigma^{-1} }_{i,i} + \horder_{i,i}.
\end{align*}
Note that since $\bE_{i,i} = 0$, $\sbr{ -\bSigma^{-1} \bE \bSigma^{-1} }_{i,i} = 0$. 
Then 
\begin{align*}
  \abr{ \gorder_{i,i} } & = \abr{ \horder_{i,i} } 
	\\
	& \le \nbr{\horder}_1 
  \\
	& \le  \frac{\ell}{(1 - \ell)^2(1 - \ell_e - \ell)} \nbr{\bE}_1
  \\
	& \le  \frac{\ell \ell_e}{(1 - \ell)^2(1 - \ell_e - \ell)} 
\end{align*}
where the third step is by \eqnref{eqn:v:1}. This completes the proof.
\end{proof}

\begin{lem}[Bound on $\xi$] \label{lem:noise_term}
Suppose $\sym{\bE} < \ell \le 1/8$ and $ \bSigma \succeq (1-\ell)\bI$.
Then  for any $i \in [\ntopic]$,
\[
  \abr{\xi_i} \le  \gamma := \frac{ 1 }{1-2\ell} \nbr{ \bA^\dagger}_\infty  \nbr{\bN}_\infty + \cnoise \nbr{ \bA^\dagger}_\infty.
\]
If furthermore, $\|\bN\|_{\infty} \nbr{\rbr{\bAg}^\dagger}_\infty < 1/8$, then
\begin{align*}
  \nbr{ \bA^\dagger}_\infty & \le 2\nbr{ \rbr{\bAg}^\dagger}_\infty,  
	\\
	\gamma  & \le  3 \nbr{ (\bAg)^\dagger}_\infty \rbr{ \nbr{\bN}_\infty + \cnoise}.
\end{align*}

\end{lem} 

\begin{proof}[Proof of Lemma~\ref{lem:noise_term}]
First, we have
\[
  \nbr{\xi}_\infty \le \nbr{ \bA^\dagger \bN \forder x^*}_\infty  + \nbr{ \bA^\dagger \noise}_\infty \le \nbr{ \bA^\dagger}_\infty  \nbr{\bN}_\infty  \nbr{\forder}_\infty \nbr{ x^*}_\infty  + \nbr{ \bA^\dagger}_\infty \nbr{\noise}_\infty.
\]

Note that $\nbr{ x^*}_\infty \le 1$ and $\nbr{\noise}_\infty \le \cnoise$.
Furthermore, 
\[
  \nbr{\forder}_\infty \le \frac{1}{1 - 2 \ell}.
\]

The first statement follows from combining these terms.

Now consider the second statement. We apply Lemma~\ref{lem:pseudo_inverse}. Since
\begin{align*}
  \zeta & = \|  \bE \bSigma^{-1} +(\bAg)^{\dagger} \bN \bSigma^{-1} \|_{\infty}  
	\\
	& \le \|  \bE \bSigma^{-1}\|_{\infty}  + \| (\bAg)^{\dagger} \bN \bSigma^{-1} \|_{\infty}  
	\\
	& \le \frac{1}{7}+ \| (\bAg)^{\dagger} \|_{\infty} \times  \|\bN\|_{\infty}  \times  \| \bSigma^{-1} \|_{\infty}  
	\\
	& \le \frac{2}{7},
\end{align*}
Lemma~\ref{lem:pseudo_inverse} implies that
\begin{align*}
  \| \bA^{\dagger} \|_{\infty} 
	& \le  \frac{ \| \bSigma^{-1} \|_{\infty}}{1 - \zeta} \| (\bAg)^{\dagger}\|_{\infty} 
   \le 2 \| (\bAg)^{\dagger}\|_{\infty} .
\end{align*}

Then $\gamma$ is bounded by 
\begin{align*}
  \gamma & = \frac{ 1 }{1-2\ell} \nbr{ \bA^\dagger}_\infty  \nbr{\bN}_\infty + \cnoise \nbr{ \bA^\dagger}_\infty
	\\
  & \le \frac{ 1 }{1-2\ell} \times \left(2  \| (\bAg)^{\dagger}\|_{\infty}  \right) \times \nbr{\bN}_{\infty } + \cnoise \times  \rbr{ 2 \| (\bAg)^{\dagger}\|_{\infty} } 
	\\
  & \le 3 \nbr{ (\bAg)^\dagger}_\infty \rbr{ \nbr{\bN}_\infty + \cnoise}. \qedhere
\end{align*}
\end{proof}

The following is the lemma about the norm of the pseudo-inverse, which is used in Lemma~\ref{lem:noise_term}.

\begin{lem}[Pseudo-inverse]\label{lem:pseudo_inverse}
Let $\bAg, \bN \in \Real^{m \times n}$ be two matrices with $m \ge n$. 
Let $(\bAg)^{\dagger}$ be one pseudo-inverse of $\bAg$ such that $(\bAg)^{\dagger}\bAg  = \bI$. 
Let $\bA = \bAg(\bSigma + \bE) + \bN$ be another matrix, with $\bSigma$ being diagonal and 
\[
  \zeta := \|  \bE \bSigma^{-1} +(\bAg)^{\dagger} \bN \bSigma^{-1} \|_{\infty}  .
\]
satisfies $\zeta < 1$.
Then there exists a pseudo-inverse $\bA^{\dagger}$ of $\bA$ such that $\bA^{\dagger} \bA = \bI$ and 
\[
  \| \bA^{\dagger} \|_{\infty} \le  \frac{ \| \bSigma^{-1} \|_{\infty}}{1 - \zeta} \| (\bAg)^{\dagger}\|_{\infty}.
\]
\end{lem}

\begin{proof}[Proof of Lemma \ref{lem:pseudo_inverse}]
Consider the matrix 
\begin{align*}
  \bA^{\dagger} & = (\bSigma + \bE + (\bAg)^{\dagger} \bN)^{-1} (\bAg)^{\dagger}.
\end{align*}
Then by definition, 
\begin{align*}
  \bA^{\dagger} \bA 
	& =  (\bSigma + \bE + (\bAg)^{\dagger} \bN)^{-1} (\bAg)^{\dagger} \left( \bAg(\bSigma + \bE) + \bN \right)  
	\\
	& =  (\bSigma + \bE + (\bAg)^{\dagger} \bN)^{-1}  (\bSigma + \bE + (\bAg)^{\dagger} \bN) 
	\\
	& = \bI.
\end{align*}
What remains is to bound $\|\bA^{\dagger}\|_{\infty}$. We have 
\begin{align*}
  \|\bA^{\dagger}\|_{\infty} & \le \| (\bSigma + \bE + (\bAg)^{\dagger} \bN)^{-1} \|_{\infty} \| (\bAg)^{\dagger} \|_{\infty}.
\end{align*}
By Taylor expansion rule, the first term on the right-hand side is
\begin{align*}
 (\bSigma + \bE + (\bAg)^{\dagger} \bN)^{-1} 
&= \left(\left(\bI + \bE \bSigma^{-1} +  (\bAg)^{\dagger} \bN \bSigma^{-1}\right) \bSigma\right)^{-1}
 \\
 &= \bSigma^{-1} \left(\bI + \bE \bSigma^{-1} +  (\bAg)^{\dagger} \bN \bSigma^{-1}\right)^{-1}
 \\
 &= \sum_{i = 0}^{\infty} \bSigma^{-1}\left(- \bE \bSigma^{-1} -  (\bAg)^{\dagger} \bN \bSigma^{-1}\right)^i
\end{align*}
where we use the assumption that $\|  \bE \bSigma^{-1} +(\bAg)^{\dagger} \bN \bSigma^{-1} \|_{\infty}  = \zeta  < 1$.
Therefore, 
\begin{align*}
  \| (\bSigma + \bE + (\bAg)^{\dagger} \bN)^{-1}  \|_{\infty} 
	& \le \| \bSigma^{-1} \|_{\infty} \sum_{i = 0}^{\infty} \zeta^i 
	 =  \frac{ \| \bSigma^{-1} \|_{\infty}}{1 - \zeta}. \qedhere
\end{align*} 
\end{proof}
 
\subsection{Putting things together}
We are now ready to prove our main theorems.

\purificationadv*

\begin{proof}[Proof of Theorem~\ref{thm:main_correlated_noise}]
We consider the following set of parameters
\[
  \alpha = \frac{c_2}{80 C_1},  r = \frac{n}{c_2}, \eta = \frac{\ell}{6}.
\]
Furthermore, set $\rho = B_1 \frac{c_2^2 c}{ C_1^3}$ for a sufficiently small absolute constant $B_1$. Since $C_1 \ge n \E[x_i^*] \ge n\E[(x^*_i)^2] \ge c_2$, this is small enough so that
\[
  \rho \le \min\cbr{\frac{\alpha}{2}, \frac{c_2 \alpha}{2048 C_1}, \frac{c_2 \alpha}{8000\times 100 C_1^2}, \frac{c c_2 \alpha}{48000 C_1^2} }
\]
which will be used in the proof. 
The proof also needs
$
  C_1^2 \le B_1 c_2 n, C_1^3 \le B_2 c^2_2 n
$
for sufficiently small absolute constants $B_1$ and $B_2$. Since $C_1 > c_2$, we only need
$
  C_1^3 \le \abscc c^2_2 n.
$
Similarly, we need 
\begin{align*}
  \cnoise \le B_1 \min\cbr{\frac{c (\alpha - \rho) c_2}{m C_1}, \frac{(\alpha - \rho)c_2}{n C_1 \nbr{(\bAg)^\dagger}_\infty }, \frac{(\alpha - \rho)c_2 \rho}{n C_1 \nbr{(\bAg)^\dagger}_\infty}, \frac{\rho}{\| (\bAg)^{\dagger} \|_{\infty} }}
\end{align*}
for a sufficiently small absolute constant $B_1$. 
This can be satisfied by setting $\abscc$ small enough in the theorem assumption. 

After setting the parameters needed, we now prove the theorem. We prove it by proving the following three claims by induction on $t$: at the beginning of iteration $t$, 
\begin{itemize}
\item[(1)] $(1 - \ell)\bI \preceq \titime{\bSigma}{t} $,
\item[(2)]  $\sym{\titime{\bE}{t}} \le  \frac{1}{8} $, and if $t > 0$
\[
  \sym{ \pt{\bE} } + \beta \sym{ \nt{\bE} } \le  \rbr{1-\frac{1}{25}\eta }  \rbr{ \sym{ \sbr{\bE}_{+}^{(t-1)} } + \beta \sym{ \sbr{\bE}_{-}^{(t-1)} } } + \frac{c}{10},
\]
for $\beta = \frac{\sqrt{84^2 + 2800} - 84}{2} \in (1, 8)$,
\item[(3)] $\nbr{\titime{\bN}{t}}_{\infty} \le \frac{1}{ 8 \nbr{(\bAg)^\dagger}_\infty}$, and $\| \xi^{(t)} \|_{\infty} \le \rho$.
\end{itemize}

Claim (1) and (2) are clearly true at $t = 0$ by the assumption on initialization. The first part of Claim (3) is true because of the assumption that $\nbr{\titime{\bN}{0}}_{\infty} \le \frac{\abscc c}{ 8 \mu^3 \nbr{(\bAg)^\dagger}_\infty}$ and that $\mu = C_1/c_2 \ge 1$. 
Then the second part follows from Lemma~\ref{lem:noise_term}.

Now we assume they are true up to $t$, and show them for $t+1$.

(1) First consider the diagonal terms.
 Combining Lemma~\ref{lem:update_diag} and Lemma~\ref{lem:higher_order}, we have
\begin{align*}
\btSigma^{(t)}_{i, i }  
& \ge  \E\sbr{\rbr{x_i^*}^2} \rbr{ 2 (\bSigma_{i, i}^{(t)})^{-1} - 2 \abr{\gorder^{(t)}_{i,i} } } - \frac{2\xexpc}{\ntopic} \rbr{  \alpha +2\nthres  + \frac{\xexpc}{\ntopic} (\bSigma_{i, i}^{(t)})^{-1} + \frac{2\xexpc}{\ntopic}  \nbr{ \sbr{ \gorder^{(t)}}^i  }_1   }.
\\
& \ge  \frac{2\xsu}{\ntopic} \rbr{  0  -     \frac{\ell^2}{(1 - 2\ell)(1 - \ell)^2}  }  - \frac{2\xexpc}{\ntopic} \rbr{  \alpha + \alpha  + \frac{\xexpc}{\ntopic} \frac{1}{1-\ell} + \frac{2\xexpc}{\ntopic} \frac{\ell}{(1 - \ell)(1 - 2 \ell)}    }.
\\
& =  \frac{2\xsu}{\ntopic} \rbr{  0  -     \frac{\ell^2}{(1 - 2\ell)(1 - \ell)^2}  }   -  \frac{2\xexpc}{\ntopic} \left( 2 \alpha  + \frac{\xexpc}{\ntopic (1 - \ell)(1 - 2 \ell)} \right)
\\
& >  - \frac{\xsl}{5 \ntopic}.
\end{align*}
The first inequality uses $\rho < \alpha/2$ and the last inequality is due to $\alpha \le \frac{\xsl}{80 \xexpc}$ and $\xexpc^2 \le \frac{c_2\ntopic}{80}$.
Therefore,
\[
  \titime{\bSigma}{t+1}_{i,i} = (1-\eta) \titime{\bSigma}{t}_{i,i} + \eta \etaratio \titime{\btSigma}{t}_{i,i} \ge (1-\eta) \titime{\bSigma}{t}_{i,i} - \frac{\eta}{5}.
\]

Assume for contradiction $\titime{\bSigma}{t+1}_{i,i} < 1-\ell$. Then by the above inequality,
\[
   1-\ell > \titime{\bSigma}{t+1}_{i,i} \ge (1-\eta) \titime{\bSigma}{t}_{i,i} - \frac{\eta}{5}.
\] 
which implies $\titime{\bSigma}{t}_{i,i} \le 1 - \ell + 2\eta$. 
In this case, by Lemma~\ref{lem:update_diag} and Lemma~\ref{lem:higher_order},
\begin{align*}
\btSigma^{(t)}_{i, i }  & \ge \E\sbr{\rbr{x_i^*}^2} \rbr{ 2 (\bSigma_{i, i}^{(t)})^{-1} - 2 \abr{\gorder^{(t)}_{i,i} } } - \frac{2\xexpc}{\ntopic} \rbr{  \alpha +2\nthres  + \frac{\xexpc}{\ntopic} (\bSigma_{i, i}^{(t)})^{-1} + \frac{2\xexpc}{\ntopic}  \nbr{ \sbr{ \gorder^{(t)}}^i  }_1   }.
\\
& \ge  \frac{2\xsl}{\ntopic}\rbr{ \frac{1}{1 - \ell + 2\eta}  -   \frac{\ell^2}{(1 - 2\ell)(1 - \ell)^2} }   - \frac{2\xexpc}{\ntopic} \left( 2 \alpha  + \frac{\xexpc}{\ntopic (1 - \ell)(1 - 2 \ell)} \right)
\\
& >  \frac{\xsl}{ \ntopic}.
\end{align*}
Then 
\[ 
  \titime{\bSigma}{t+1}_{i,i} = (1-\eta) \titime{\bSigma}{t}_{i,i} + \eta \etaratio \titime{\btSigma}{t}_{i,i}  = (1-\eta) \titime{\bSigma}{t}_{i,i} + \eta > \titime{\bSigma}{t}_{i,i},
\]
which is a contradiction. Therefore, $(1 - \ell)\bI \preceq \titime{\bSigma}{t} $.

%

(2) Now consider the off-diagonal terms. We shall split them into the positive part and the negative part. 
By the update rule, for any $i \in [\ntopic]$,
\begin{align*}
  \nbr{ \sbr{  \pto{\bE} }_i }_1  & \le (1 - \eta) \nbr{  \sbr{  \pt{\bE} }_i }_1 +  \eta \etaratio \nbr{ \sbr{\pt{\btE} }_i }_1.
\end{align*}
Recall the notations 
\begin{align*}
  \titime{\forder}{t} & = (\titime{\bSigma}{t} + \titime{\bE}{t})^{-1} = (\titime{\bSigma}{t})^{-1} + \titime{\gorder}{t},
\\
  \titime{\gorder}{t} & = (\titime{\bSigma}{t})^{-1} \sum_{k=1}^\infty (-\titime{\bE}{t} (\titime{\bSigma}{t})^{-1} )^k
\end{align*}
By Lemma~\ref{lem:update_E}, we have 
\begin{align*}
  \nbr{ \sbr{\pt{\btE} }_i }_1 
	& \le  
	\underbrace{
		\sum_{j\neq i} \frac{4 \xexpc^2  
	}{\ntopic^2 (\alpha - \nthres)} \nbr{ \sbr{ \titime{\forder}{t}}^i }_1 \rbr{  \abr{ \sbr{ \nt{\forder} }_{i, j} }  + \nthres} 
	}_{T1}
	\\
	 &\quad + 	\underbrace{
	\sum_{j\neq i}  \frac{8\xexpc \nthres}{\ntopic (\alpha - \nthres)} \rbr{ \frac{\xexpc  }{n} \nbr{ \sbr{ \titime{\forder}{t}}^i }_1 	+ \abr{ \sbr{\pt{\forder} }_{i, j} }   } 
	}_{T2}
	\\
	&\quad + \underbrace{ 
	\sum_{j\neq i}  2 \E[(x_j^*)^2] \abr{ \sbr{ \pt{\forder} }_{i, j} }
	}_{T3}.
\end{align*}

First, by Lemma~\ref{lem:higher_order},
\[
  \nbr{ \sbr{ \titime{\forder}{t}}^i }_1 \le  \sbr{\rbr{\titime{\bSigma}{t}}^{-1}}_{i,i}  + \nbr{ \sbr{\titime{\gorder}{t} }^i }_1 \le\frac{1}{1 - 2 \ell} 
\]
Now consider $\pt{\forder}$ and $\nt{\forder}$.
We have 
\[
  \sum_{j: j \ne i}  \abr{ \sbr{ \nt{\forder} }_{i, j} }  \le  \nbr{ \sbr{\nt{\gorder}}_i }_1,~~
  \sum_{j: j \ne i}  \abr{ \sbr{ \pt{\forder} }_{i, j} }  \le  \nbr{ \sbr{\pt{\gorder}}_i }_1.
\]

Therefore,
\begin{align*}
T1 & \le \frac{8 \xexpc^2}{\ntopic^2 (\alpha - \rho)}  \nbr{ \sbr{\nt{\gorder}}_i }_1  + \frac{8 \xexpc^2 \rho}{\ntopic (\alpha - \rho)},
	\\
	 T2 & \le  \frac{16 \xexpc^2 \rho}{\ntopic (\alpha - \rho)}+  \frac{8 \xexpc \rho}{\ntopic (\alpha - \rho)} \nbr{ \sbr{\pt{\gorder}}_i }_1,
	\\
	T3 & \le  \frac{2\xsu}{n} \nbr{ \sbr{\pt{\gorder}}_i }_1.
\end{align*}
and thus we have 
\begin{align*}
  \nbr{ \sbr{\pt{\btE} }_i }_1  & \le \frac{8 \xexpc^2}{\ntopic^2 (\alpha - \rho)}  \nbr{ \sbr{\nt{\gorder}}_i }_1 + \rbr{ \frac{2\xsu}{n} + \frac{8 \xexpc \rho}{\ntopic (\alpha - \rho)} } \nbr{ \sbr{\pt{\gorder}}_i }_1 + \frac{24 \xexpc^2 \rho}{\ntopic (\alpha - \rho)}. 
\end{align*}
Similarly, for any $i \in [\ntopic]$, 
\begin{align*}
 \nbr{ \sbr{\pt{\btE} }^i }_1 & \le \frac{8 \xexpc^2}{\ntopic^2 (\alpha - \rho)}  \nbr{ \sbr{\nt{\gorder}}^i }_1 + \rbr{ \frac{2\xsu}{n} + \frac{8 \xexpc \rho}{\ntopic (\alpha - \rho)} } \nbr{ \sbr{\pt{\gorder}}^i }_1 + \frac{24 \xexpc^2 \rho}{\ntopic (\alpha - \rho)}. 
\end{align*}

Putting the two together, we have
\begin{align}
  \sym{\pt{\btE} } & \le   \frac{8 \xexpc^2}{\ntopic^2 (\alpha - \rho)}  \sym{\nt{\gorder} } +  \rbr{ \frac{2\xsu}{n} + \frac{8 \xexpc \rho}{\ntopic (\alpha - \rho)} } \sym{\pt{\gorder} } +\frac{24 \xexpc^2 \rho}{\ntopic (\alpha - \rho)}. \label{eqn:main11}
\end{align}

By Lemma~\ref{lem:higher_order} and $\ell \le \frac{1}{8}$, we have:
\begin{align*}
  \sym{\pt{\gorder}} & \le \frac{32}{21} \sym{ \nt{\bE}  } +  \frac{32}{147} \sym{ \pt{\bE}}, \\
  \sym{\nt{\gorder}} & \le \frac{32}{21} \sym{ \pt{\bE}  } +  \frac{32}{147} \sym{\nt{\bE}}
\end{align*}

So \eqnref{eqn:main11} becomes
\begin{align}
  \sym{\pt{\btE} } & \le \rbr{ \frac{64 \xsu}{147 \ntopic}  + \frac{256C_1 \rho}{147 n (\alpha - \rho)} +  \frac{256 C_1^2}{21n^2 (\alpha - \rho)} } \sym{ \pt{\bE}  }  
	\\
  & \quad + \rbr{ \frac{64 \xsu}{21 \ntopic}  + \frac{256C_1 \rho}{21 n (\alpha - \rho)}   + \frac{256 C_1^2}{147n^2 (\alpha - \rho)} } \sym{ \nt{\bE}  } +\frac{24 \xexpc^2 \rho}{\ntopic (\alpha - \rho)}.  \label{eqn:main1}
\end{align} 

Now consider the negative part. The same argument as above leads to
\begin{align}
  \sym{\nt{\btE} } & \le \rbr{ \frac{64 C_1^2}{147 \ntopic^2}  + \frac{256C_1 \rho}{147 n (\alpha - \rho)} +  \frac{256 C_1^2}{21n^2 (\alpha - \rho)} } \sym{ \pt{\bE}  }  
	\nonumber \\
  & \quad + \rbr{ \frac{64 C_1^2}{21 \ntopic^2}  + \frac{256C_1 \rho}{21 n (\alpha - \rho)}   + \frac{256 C_1^2}{147n^2 (\alpha - \rho)} } \sym{ \nt{\bE}  } +\frac{24 \xexpc^2 \rho}{\ntopic (\alpha - \rho)}. 
 \label{eqn:main2}
\end{align}
Note the difference between \eqnref{eqn:main1} and \eqnref{eqn:main2}:  $\frac{\xsu}{n}$ in the former is replaced by $\frac{\xexpc^2}{n^2}$ in the latter, which is much smaller. This is crucial for our proof, which will be clear below.

For simplicity, we introduce the following notations:
\[
  a_{t} : = \sym{ \pt{\bE} }, ~~b_t := \sym{ \nt{\bE} }.
\]
Then by the update rule, we have
\begin{align*}
  a_{t+1} &  \le  (1-\eta) a_{t}  + \eta \etaratio  \sym{ \pt{\btE} },
\\
  b_{t+1} &  \le  (1-\eta) b_{t} + \eta \etaratio \sym{ \nt{\btE} }.
\end{align*}
Plugging in \eqnref{eqn:main1}and since $\etaratio = \frac{n}{c_2} \le \frac{2n}{C_2}$, we have
\begin{align*}  
  a_{t+1} &  \le  (1-\eta) a_{t}  + \eta \frac{2n}{C_2}  \rbr{ \frac{64 \xsu}{147 \ntopic}  + \frac{256C_1 \rho}{147 n (\alpha - \rho)} +  \frac{256 C_1^2}{21n^2 (\alpha - \rho)} } a_t  
	\\
  & \quad + \eta \frac{2n}{C_2} \rbr{ \frac{64 \xsu}{21 \ntopic}  + \frac{256C_1 \rho}{21 n (\alpha - \rho)}   + \frac{256 C_1^2}{147n^2 (\alpha - \rho)} } b_t  +  \eta \frac{2n}{C_2}\frac{24 \xexpc^2 \rho}{\ntopic (\alpha - \rho)}
	\\	
  b_{t+1} &  \le  (1-\eta) b_{t}  + \eta \frac{2n}{C_2}  \rbr{ \frac{64 C_1^2}{147 \ntopic^2}  + \frac{256C_1 \rho}{147 n (\alpha - \rho)} +  \frac{256 C_1^2}{21n^2 (\alpha - \rho)} } a_t
	\\
  & \quad + \eta \frac{2n}{C_2} \rbr{ \frac{64 C_1^2}{21 \ntopic^2}  + \frac{256C_1 \rho}{21 n (\alpha - \rho)}   + \frac{256 C_1^2}{147n^2 (\alpha - \rho)} } b_t  +  \eta \frac{2n}{C_2}\frac{24\xexpc^2 \rho}{\ntopic (\alpha - \rho)}.
\end{align*}
When $\frac{512 C_1 \rho}{C_2 (\alpha - \rho)} \le \frac{1}{2}$ and $\frac{512 C_1^2}{C_2 n (\alpha - \rho)} \le \frac{1}{14}$, 
\begin{align*}  
  a_{t+1} &  \le  (1-\eta) a_{t}  + \frac{129}{147}\eta a_t  + \frac{129}{21} \eta b_t + \eta\frac{48 \xexpc^2 \rho}{ C_2 (\alpha - \rho)} 
	\\
	& \le \rbr{1-\frac{18}{147}\eta  } a_{t}  + \frac{129}{21}  \eta b_t + \eta \frac{48 \xexpc^2 \rho}{ C_2(\alpha - \rho)}
\end{align*} 
Similarly, 
when $\frac{512 C_1 \rho}{C_2 (\alpha - \rho)} \le \frac{1}{2}$ and $\frac{512 C_1^2}{C_2 n (\alpha - \rho)} \le \frac{1}{14}$, and furthermore, $\frac{128 C_1^2}{C_2 n} \le \frac{1}{4}$, 
\begin{align*}  
  b_{t+1} &  \le  (1-\eta) b_{t}  + \frac{1}{100}\eta a_t  + \frac{1}{25} \eta b_t + \eta\frac{48 \xexpc^2 \rho}{ C_2 (\alpha - \rho)} 
	\\
	& \le \rbr{1-\frac{24}{25}\eta  } b_{t}  + \frac{1}{100}  \eta a_t + \eta \frac{48 \xexpc^2 \rho}{ C_2(\alpha - \rho)}
\end{align*} 
Let $h =  \frac{48 \xexpc^2 \rho}{ C_2(\alpha - \rho)} $, we then have:
\begin{align*}
  a_{t+1} &  \le  \rbr{1-\frac{3}{25}\eta  } a_{t}  +  7\eta b_t +\eta h,
	\\
  b_{t+1} &  \le    \rbr{1- \frac{24}{25}\eta  } b_{t}  +  \frac{1}{100}\eta a_t + \eta h.
\end{align*}

Now set $\beta = \frac{\sqrt{84^2 + 2800} - 84}{2}$, so that 
\begin{align*}
a_{t+1} + \beta b_{t+1} &  \le   \rbr{1-\frac{3}{25}\eta  } a_{t}  +  7\eta b_t + \eta h +  \rbr{ \beta  - \frac{24}{25}\eta  \beta   } b_{t}  +  \frac{ \beta }{100}\eta a_t  + \eta \beta h
\\
 & = \rbr{1-\frac{3}{25}\eta  + \frac{\beta}{100} \eta }  \rbr{ a_{t} + \beta b_{t} } + \eta (1 + \beta) h
\\
 & \le \rbr{1-\frac{1}{25}\eta }  \rbr{ a_{t} + \beta b_{t} } + 9 \eta h,
\end{align*}

where the last inequality follows from that $\beta < 8$.

Note that the recurrence is true up to $t+1$. 
Using Lemma~\ref{l:simplerec} to solve this recurrence, we obtain
\[
  a_t + b_t \le a_0 + b_0 + 250 h \le \frac{1}{10} + 250 h \le  \frac{1}{8}
\]
when $\frac{4000 C_1^2 \rho}{C_2 (\alpha - \rho)} \le \frac{1}{100}$.
Moreover, we know that
\[
  \sym{\titime{\bE}{t+1}} \le a_{t+1} + \beta b_{t+1} \le  \rbr{1-\frac{1}{25}\eta }^t + 250  h.
\]

(3) Finally, consider the noise term. 
Set the sample size $N$ to be large enough, so that by Lemma~\ref{lem:bound_error}, we have
\begin{align*}
  \abr{ \titime{\btN}{t}_{i,j} }
	& \le  \frac{4 \cnoise \xexpc }{(1 - 2 \times \ell)^2 n  (\alpha - \rho)} + \abr{ [\titime{\bN}{t}_s]_{i,j}} 
	\\
	& \le   \frac{8 \cnoise C_1}{n (\alpha - \rho)}.
\end{align*}
Then by the update rule, we have $ \abr{ \titime{\bN}{t+1}_{i,j} } \le \frac{8 \cnoise C_1}{(\alpha - \rho)c_2} $.  
Then 
\[
  \nbr{\titime{\bN}{t+1}}_{\infty} \le n \max_{i,j} \abr{ \titime{\bN}{t+1}_{i,j} } \le \frac{8 n\cnoise C_1}{(\alpha - \rho)c_2} \le \frac{1}{ 8 \nbr{(\bAg)^\dagger}_\infty}
\]
where the last inequality is due to 
\[ 
  \cnoise  \le \frac{(\alpha - \rho)c_2}{ 64n C_1\nbr{(\bAg)^\dagger}_\infty}.
\]
On the other hand, by Lemma \ref{lem:noise_term}, we have
\begin{align*}
  \| \xi^{(t + 1)} \|_{\infty} 
	& \le 3 \| (\bAg)^{\dagger} \|_{\infty} (\|\titime{\bN}{t + 1} \|_{\infty} + \cnoise ) 
	\\
	& \le  3 \| (\bAg)^{\dagger} \|_{\infty} \rbr{ \frac{8  n \cnoise C_1}{(\alpha - \rho)c_2} + \cnoise}\le \rho
\end{align*}
where the last inequality is due to 
\[ 
  \cnoise \le \frac{(\alpha -\rho) c_2 \rho}{48 n C_1 \| (\bAg)^{\dagger} \|_{\infty} }, \text{~and~}\cnoise \le \frac{\rho}{6 \| (\bAg)^{\dagger} \|_{\infty} }.
\]
We also have (which will be useful in proving the final bound)
\[
  \nbr{\titime{\bN}{t+1}}_{1} \le m \max_{i,j} \abr{ \titime{\bN}{t+1}_{i,j} } \le  \frac{8 m \cnoise C_1}{(\alpha - \rho)c_2}  \le  \frac{c}{10}
\]
where the last inequality is due to 
\[ 
  \cnoise  \le \frac{c (\alpha - \rho)c_2}{80 m C_1}.
\]

Now, we shall prove the theorem statements.
Recall that solving the recurrence about $a_t$ and $b_t$ leads to 
\[
  \sym{\titime{\bE}{t+1}} \le a_{t+1} + \beta b_{t+1} \le  \rbr{1-\frac{1}{25}\eta }^t + 250  h.
\]
Since the setting of $\rho$ makes sure $h=O(c)$, when $t = O\rbr{\ln \frac{1}{\epsilon}}$, we have the second statement $\sym{\widehat{\bE} } \le \epsilon + \frac{c}{2}$. Note that 
\[
  \bAg\widehat{\bSigma} = \bA - \bAg   \widehat{\bE}   - \widehat{\bN}
\]
and 
\[
 \nbr{\sbr{\bAg\widehat{\bSigma} }_i}  = \widehat{\bSigma}_{i,i},~~ \nbr{\bA}_1 = 1, ~~\nbr{\bAg   \widehat{\bE}}_1 =  \nbr{   \widehat{\bE}}_1,
\]
so we have 
\begin{align*}
  \widehat{\bSigma}_{i,i} & \ge \nbr{\bA}_1 - \nbr{   \widehat{\bE}}_1 - \nbr{   \widehat{\bN}}_1 
	\\
	& \ge  1 - \epsilon - c.
\end{align*}
Similarly, 
\begin{align*}
  \widehat{\bSigma}_{i,i} & \le \nbr{\bA}_1 + \nbr{   \widehat{\bE}}_1 + \nbr{   \widehat{\bN}}_1 
	\\
	& \le 1 + \epsilon + c.
\end{align*}
Then the final statement of the theorem follows by replacing $c$ with $c/4$.
This completes the proof.
\end{proof}

%
\purificationunbiased*

\begin{proof}
The proof is similar to that of Theorem~\ref{thm:main_correlated_noise}, except using the second bound for unbiased noise in Lemma~\ref{lem:bound_error}. We highlight the different part, that is, the induction on the noise term. 

In the induction, by Lemma~\ref{lem:bound_error}
we have when $N$ is large enough,
\begin{align*}
  \abr{\btN^{(t)}_{i,j} } 
	& \le    \frac{2 \xexpc \cnoise \rho' (1 + \| \bA^{\dagger} \bN^{(t)} \|_{\infty})}{(1 - 2 \ell)  n (\alpha - \rho')}  + \abr{ \sbr{\bN^{(t)}_s}_{i,j} }  \le    \frac{3 \xexpc \cnoise \rho' (1 + \| \bA^{\dagger} \bN^{(t)} \|_{\infty})}{ n (\alpha - \rho')}.
\end{align*}
By Lemma~\ref{lem:noise_term} and the induction, we have $\| \bA^{\dagger} \bN^{(t)} \|_{\infty} \le 1/4$. Furthermore, $\rho' \le \cnoise \| \bA^{\dagger} \|_{\infty} \le 2\cnoise \| (\bAg)^{\dagger} \|_{\infty}$ and the parameter setting makes sure $\rho' \le \alpha/2$. Then
\begin{align*}
  \abr{\btN^{(t)}_{i,j} } \le  \frac{16 \cnoise^2 C_1 \nbr{(\bAg)^\dagger}_\infty}{n \alpha}.
\end{align*}
Then by the update rule, we have 
\[
\abr{ \titime{\bN}{t+1}_{i,j} } \le \frac{32 \cnoise^2  C_1 \nbr{(\bAg)^\dagger}_\infty}{c_2 \alpha}
\]
and 
\begin{align}
  \nbr{ \titime{\bN}{t+1} }_\infty \le \frac{32 n \cnoise^2  C_1 \nbr{(\bAg)^\dagger}_\infty}{c_2 \alpha} & \le \frac{1}{8  \nbr{(\bAg)^\dagger}_\infty}
\end{align}
by the definition of $\alpha$, and $\cnoise \le \frac{1}{256} \frac{c_2}{C_1} \frac{\sqrt{n}}{n \nbr{(\bAg)^\dagger}_\infty}$.
This completes the induction for the noise. 

Also, in proving the final bounds, we have
\begin{align}
  \nbr{ \titime{\bN}{t+1} }_1 \le \frac{32 m \cnoise^2  C_1 \nbr{(\bAg)^\dagger}_\infty}{c_2 \alpha}
	\le  \frac{c}{10}
\end{align}
by the definition of $\alpha$, and 
\[
\cnoise \le \frac{c}{320} \frac{c_2}{C_1} \frac{\sqrt{n}}{\max\cbr{m, n\nbr{(\bAg)^\dagger}_\infty}} \le \frac{\sqrt{c}}{320} \frac{c_2}{C_1} \frac{1}{\sqrt{m \nbr{(\bAg)^\dagger}_\infty}}
\]
where the last inequality can be shown by consider the two cases when $\nbr{(\bAg)^\dagger}_\infty \le m/n$ and 
$\nbr{(\bAg)^\dagger}_\infty \ge m/n$. The rest of the proof is the same as in Theorem~\ref{thm:main_correlated_noise}.
\end{proof}

\section{Results for general proportions: Equilibration} \label{app:proof_equilibration}

\begin{algorithm}[H]
\caption{ColumnUpdate}\label{alg:col_update}
\begin{algorithmic}[1]
\REQUIRE A matrix $\bA$, a threshold value $\alpha$, a step size $\eta$, ratios $\{r_j: j \in [\ntopic]\}$,  iteration number $T$, a subset $S \subseteq [n]$, sample size $N$
\STATE Set $\bA^{(0)} = \bA$
\FOR{$t = 0 \to T - 1$}    
\STATE 
\begin{align}
  \forall i \in S,  [{\bA}^{(t + 1)}]_i = \left[\left(1 - \eta\right){\bA}^{(t)} + \damp_i \eta  \tilde\E\left[ (y - y')  (x - x')^{\top}\right]\right]_i \label{eq:topicupdate} 
\end{align}
\ENDFOR
\ENSURE $\bhA = \bA^{(T)}$
\end{algorithmic}
\end{algorithm}

\begin{algorithm}[H]
\caption{Rescale}\label{alg:rescale}
\begin{algorithmic}[1]
\REQUIRE  A matrix $\bA$, a threshold value $\alpha$, a step size $\eta$, ratios $\{r_j: j \in [\ntopic]\}$, iteration number $T$, and a set $S \subseteq [n]$, $\epsilon \in (0, 1)$. 
\STATE Let $\btA = \text{ColumnUpdate}(\bA, \alpha, \eta, \{r_j\}_j, T, S, N)$
\FOR{$i \in S$}
\STATE Set $[\bhA]_i = \frac{1}{1 - \epsilon}[\btA]_i$
\ENDFOR
\ENSURE $\bhA$
\end{algorithmic}
\end{algorithm}

\begin{algorithm}[H]
\caption{Equilibration}\label{alg:meta}
\begin{algorithmic}[1]
\REQUIRE  $\bA$, $\alpha$, $\eta$, $T$, and $\epsilon \in (0, 1)$, $\lambda, N$
\STATE $S \leftarrow \emptyset$, $\bD \leftarrow \bI$
\WHILE{$|S| \le \ntopic $} 
\STATE $m_j \leftarrow \hat\E[x_j^2]$ for $j \not\in S$ using $N$ examples
\WHILE{$ \max_{j\not\in S} m_j < \lambda $ \label{step:innerwhile}}
\STATE $\bA \leftarrow \text{Rescale}(\bA, \alpha, \eta, \{3/(5 m_j): j \in [\ntopic]\}, T, S, \epsilon, N)$
\STATE $\lambda \leftarrow \rbr{1 - \epsilon} \lambda$, $\bD_{j,j} \leftarrow \bD_{j,j}/(1-\epsilon)$
\STATE $m_j \leftarrow (1-\epsilon)^2 m_j$ for $j \in S$, and $m_j \leftarrow \hat\E[x_j^2]$ for $j \not\in S$ using $N$ examples
\ENDWHILE
\STATE $S \leftarrow S \cup \{j: m_j \ge \lambda \}$
\ENDWHILE
\ENSURE $\bA$
\end{algorithmic}
\end{algorithm}

When the feature have various proportions (i.e., $\E[(x_i^*)^2] $ varies for different $i$), we propose Algorithm~\ref{alg:meta} for balancing them. The idea is quite simple: instead of solving $\bY \approx \bAg \bX$, we could also solve $\bY \approx [ \bAg \bD] [ (\bD)^{-1} \bX]$ for a positive diagonal matrix $\bD$. Our goal is to find $\bA = \bAg\bD(\bSigma + \bE) + \bN$ so that $\bSigma$ is large, $\bE,\bN$ are small, while ${\E[(x_i^*)^2]}/{\bD_{i, i}^2}$ is with in a factor of $2$ from each other. 

The algorithm works at stages and keeps a working set $S$ of column index $i$ such that ${\E[(x_i^*)^2]}/{\bD_{i, i}^2}$ is above a threshold $\lambda$. At each stage, it only updates the columns in $S$; at the end of the stage, it increases these columns by a small factor so that ${\E[(x_i^*)^2]}/{\bD_{i, i}^2}$ decreases. Then it decreases the threshold $\lambda$, and add more columns to the working set and repeat. In this way, ${\E[(x_i^*)^2]}/{\bD_{i, i}^2} (i\in S)$ are always balanced; in particular, they are balanced at the end when $S = [n]$.
Formally, 

\begin{restatable}[Main: Equilibration]{thm}{equilibration} \label{thm:equilibration}
If there exists an absolute constant $\abscc$ such that Assumption (\textbf{A1})-(\textbf{A3}) and (\textbf{N1}) are satisfied with $l = 1/50$, $ C_1^3 \le \abscc c_2^2 n$, $\max\cbr{\cnoise,\|\bN^{(0)}\|_\infty}  \le \frac{\abscc c_2^4}{ C_1^5  n \|(\bAg)^\dagger\|_\infty}$, and additionally $\bSigma^{(0)} \preceq (1-\ell) \bI$, and $\bE \ge 0$ entry-wise, then there exist $\alpha, \eta, T, \lambda$ such that for sufficiently small $\epsilon> 0$ and sufficiently large $N  = \mathrm{poly}(\ntopic, \nword, 1/\epsilon, 1/\delta)$ the following hold with probability at least $1-\delta$: 
Algorithm~\ref{alg:meta} outputs a solution $\bA = \bAg \bD (\bSigma + \bE) + \bN$
where $\bSigma \succeq (1 - \ell ) \bI$ is diagonal, $\| \bE\|_s \le \gamma \ell$ is off-diagonal, $\| \bN \|_{\infty} \le 2\|\bN^{(0)}\|_\infty$, and $\bD$ is diagonal and satisfies
\[
  \frac{\max_{i \in [\ntopic]} \frac{1}{\bD_{i,i}^2}\E[(x_i^*)^2]}{\min_{j \in [\ntopic]} \frac{1}{\bD_{j,j}^2}\E[(x_j^*)^2]}  \le 2.
\]

If Assumption (\textbf{A1})-(\textbf{A3}) and (\textbf{N2}) are satisfied with the same parameters except $\max\cbr{\cnoise,\|\bN^{(0)}\|_\infty}  \le \min\cbr{ \sqrt{\frac{\abscc c_2^4}{ C_1^5  n }} \frac{1}{\|(\bAg)^\dagger\|_\infty}, \frac{\abscc c_2^2}{C_1^3\|(\bAg)^\dagger\|_\infty}}$, then the same guarantees hold.
\end{restatable}

Now, we can view $\bAg \bD$ as the ground-truth feature matrix and $\bD^{-1} x^*$ as the weights.
Then applying Algorithm~\ref{alg:main_sim_noise} with $\bA$ can recover $\bAg \bD$, and after normalization we get $\bAg$.

The initialization condition of the theorem can be achieved by the popular practical heuristic that sets the columns of $\bA^{(0)}$ to reasonable almost pure data points. It is generally believed that it gives $\bE^{(0)}_{i,j} \ge 0$ and $\bN^{(0)} = 0$. We note that the parameters are not optimized; the algorithm can potentially tolerate much better initialization.

\paragraph{Intuition.} Before delving into the specifics of the algorithm, it will be useful to provide a high-level outline of the proof. As described above, the algorithm makes use of the fact that samples from a ground truth matrix $\bA^*$ and distribution $x^*$ can equivalently be viewed as coming from the ground truth matrix $\bA^* \bD$ and distribution $\bD^{-1} x^*$, for some diagonal matrix $\bD$. Therefore, the goal is to find a $\bD$ such that the features are balanced:
\[
  \frac{\max_{i \in [n]} \frac{\E[(x^*_i)^2]}{\bD^2_{i,i}}}{\min_{i \in [n]} \frac{\E[(x^*_i)^2]}{\bD^2_{i,i}}} \leq \kappa.
\]
The algorithm will implicitly calculate such a $\bD$ gradually.  Namely, at any point in time, the algorithm will have an active set $S \subseteq [n]$  of features, which are balanced, i.e.  
\begin{align} \label{eqn:balanceS}
  \frac{\max_{i \in [\ntopic]} \frac{\E[(x^*_i)^2]}{\bD^2_{i,i}}}{\min_{i \in S} \frac{\E[(x^*_i)^2]}{\bD^2_{i,i}}} \leq \kappa.
\end{align}
It is clear that when $S = [n]$ the algorithm achieves the goal. Our algorithm begins with $S = \emptyset$ and gradually increase $S$ until $S = [n]$.

The mechanism for increasing $S$ will be as follows. Given $S$, $\bA$ is of the form 
\begin{align*}
  \bA = \bA^*\bD(\bSigma + \bE) + \bN 
\end{align*}
with 
\begin{align*}
  \bE = \begin{bmatrix} \bE_{1, 1} & \bE_{1, 2}  \\
  \bE_{2, 1}  & \bE_{2, 2}  \end{bmatrix}
\end{align*}
where the columns of $\bA$ are sorted such that the first $|S|$ columns correspond to the features of $S$, and $\bE_{1,1} \in \mathbb{R}^{|S| \times |S|}$, 
$\bE_{2,1} \in \mathbb{R}^{(\ntopic - |S|) \times |S|}$, $\bE_{1,2} \in \mathbb{R}^{|S| \times (\ntopic - |S|) }$, $\bE_{2,2} \in \mathbb{R}^{(\ntopic- |S|) \times (\ntopic- |S|)}$.
Then scaling up the columns of $\bA$ indexed by $S$ by a factor of $\frac{1}{1-\epsilon}$ is equivalent to 
\begin{itemize}
\item[(1)] scaling up the columns of $\bD$ indexed by $S$ by a factor of $\frac{1}{1-\epsilon}$ and
\item[(2)] scaling up the columns of $\bE_{2,1}$ by a factor of $\frac{1}{1-\epsilon}$ and
\item[(3)] scaling down the columns of $\bE_{1,2}$ by a factor of $1-\epsilon$.
\end{itemize}
Therefore, to increase the set $S$, the algorithm will scale up the columns of $\bA$ indexed by $S$, until some $j \not\in S$ satisfies
\[
  \max_{i \in [\ntopic]} \frac{\E[(x^*_i)^2]}{\bD^2_{i,i}}  \leq \kappa \frac{\E[(x^*_j)^2]}{\bD^2_{j,j}}.
\]
Then it can add $j$ into $S$ while keeping the corresponding features balanced as in (\ref{eqn:balanceS}). Note that we do not need to explicitly maintain $\bD$, though it can be calculated along with the scaling. Further note that the values of $\E[(x^*_i)^2]$ are not known but they can be estimated using the current $\bA$.

However, there is still one caveat: $\bE$ should be kept small, so that at the end of the algorithm, we still have a good initialization $\bA$.
For this reason, the algorithm additionally maintains that for a small constant $1 < \gamma < 2$, 
\begin{align}
  \sym{\bE_{1,1}} & \leq \gamma \ell, ~~\sym{\bE_{1,2}}  \leq \ell, \nonumber\\
  \sym{\bE_{2,1}} & \leq \gamma \ell, ~~\sym{\bE_{2,2}}  \leq \ell. \label{eqn:boundederror}
\end{align}
Since scaling up $\bA$ will scale up $\bE_{2,1}$, we will need to first decrease $\sym{\bE_{2,1}}$ before the scaling step. 
The key observation is that by applying our training algorithm only on the columns indexed by $S$, $\sym{\bE_{1,1}}$ and $\sym{\bE_{2,1}}$ will be decreased, while $\sym{\bE_{1,2}}$ and $\sym{\bE_{2,2}}$ unchanged. 
On a high level, using the fact that the matrix $\bE_{1,2}$ has no negative entries (which we get by virtue of our initialization), and the fact that the contribution in the updates to the entry $(\bE_{1,1})_{i,j}$ mostly comes from $(\bE_{1,1})_{j,i}$ (i.e. the matrix $\bE_{1,1}$ in the first order contribution ``updates itself''), and the fact that the features in $S$ are balanced, we can show that after sufficiently many updates, the symmetric norm of $\bE_{1,1}$ and $\bE_{2,1}$ drops by a reasonable amount: $\sym{\bE_{1,1}} \leq (\gamma-1) \ell$ and $\sym{\bE_{2,1}} \leq (1-\epsilon)(\gamma-1) \ell$. Now, we can do the scaling step without hurting the invariant~\ref{eqn:boundederror}.

\paragraph{Organization.} The result of the section is as follows. We first prove in Section~\ref{sec:columnupdate} that applying our training algorithm only on the columns indexed by $S$ will decrease $\sym{\bE_{1,1}}$ and $\sym{\bE_{2,1}}$.
Then in Section~\ref{sec:rescale} we analyze the scaling step, and show that the invariant (\ref{eqn:boundederror}) is maintained.
In Section~\ref{sec:equilibration_main_algo}, we show how to increase $S$ while maintaining the invariant (\ref{eqn:balanceS}), where the main technical details are about how to estimate $\E[(x^*_i)^2]$.

\subsection{Equilibration: ColumnUpdate} \label{sec:columnupdate}

In this subsection, we focus on the update step, bounding the changes of $\bSigma, \bE$, and $\bN$.

First recall some notations. 
Let $\bA = \bAg (\bSigma + \bE) + \bN$ where $\bSigma$ is diagonal, $\bE$ is off diagonal, and $\bN$ is the component outside the span of $\bAg$.\footnote{Note that $\bAg$ here can be any ground-truth matrix; in particular, later Lemma~\ref{lem:main_col_update} will be applied where $\bAg$ in the lemma corresponds to $\bAg \bD$ in the intuition described above.}
Given the set $S \subseteq [n]$ and a matrix $\bM \in \Real^{n \times n}$, let $\bM_{1,1}$ denote the submatrix indexed by $S \times S$,
and $\bM_{2,1} $ denote the submatrix indexed by $([\ntopic] - S) \times S$, $\bM_{1,2}$ denote the submatrix indexed by $S \times ([\ntopic] - S)$, and $\bM_{2,2}$ denote the submatrix indexed by $([\ntopic] - S) \times ([\ntopic] - S)$. \footnote{These notations will be used for $\bM= \bE$, $\bM = \btE$, and related matrices.}
In the special case when $S = [s]$ where $s = |S|$,  
\begin{align*}
  \bM = \begin{bmatrix} \bM_{1, 1} & \bM_{1, 2}  \\
  \bM_{2, 1}  & \bM_{2, 2}  \end{bmatrix}.
\end{align*}
Also, let $\bM_S$ denote the submatrix formed by the columns indexed by $S$, and $\bM_{-S}$ the submatrix formed by the other columns. \footnote{These notations will be used for $\bM = \bN$ or $\bM = \btN$, and related matrices.}

The input $\bA^{(0)}$ of Algorithm \ref{alg:col_update} can be written as
$\bA^{(0)} = \bAg (\bSigma^{(0)} + \bE^{(0)}) + \bN^{(0)}$ where $\bSigma^{(0)}$ is diagonal, and $\bE^{(0)}$ is off diagonal. 
Define $\bE^{(0)}_{1, 1}, \bE^{(0)}_{1, 2}, \bE^{(0)}_{2, 1}$ and $\bhE_{2, 2}$ as described above.
Similarly, define $\bhE_{1, 1}, \bhE_{1, 2}, \bhE_{2, 1}$ and $\bhE_{2, 2}$ for the output $\bhA = \bAg(\bhSigma + \bhE) + \bhN$ of Algorithm \ref{alg:col_update}.
Finally, define $\bN^{(0)}_S$, $\bN^{(0)}_{-S}$, $\bhN_S$, and $\bhN_{-S}$ as described above.

The main result of the subsection is Lemma~\ref{lem:main_col_update}.
\begin{lem}[Main: ColumnUpdate]\label{lem:main_col_update} 
Define 
\begin{align} 
  R_j & = \E[(x_j^*)^2], \quad R  = \max_{j \in [\ntopic]} R_j, \quad r = \max_{j \in S} r_j,
	\\
  h_1 & = r \frac{8 \xexpc (\xexpc+1)  \rho }{(1-\ell - \beta \ell)\ntopic (\alpha - \nthres)} + \frac{4 \xexpc^2  }{(1-\ell - \beta\ell)\ntopic^2 (\alpha - \nthres)} r \rbr{\frac{1}{(1-\ell - \beta\ell)} + 1},
	\\ 
  h_2 & = r \frac{R \beta^2 \ell^2 }{(1 - \ell)^2 (1 - \ell - \beta\ell)}   
	+   \frac{12 \xexpc(\xexpc+1) }{\ntopic^2 (\alpha - \nthres)(1-\ell - \beta \ell)}  \rbr{ \frac{1}{1-\ell - \beta \ell} +  \ntopic \nthres} r,
	\\
  h & = h_1 + h_2,
	\\
  U_a & = \frac{8 r \cnoise C_1}{ \alpha - \rho}, 
	\\
  U_n & = \frac{10 r \xexpc \cnoise^2   \nbr{(\bAg)^\dagger}_{\infty} }{(1 - 2 \ell)  (\alpha -  2 \cnoise   \nbr{(\bAg)^\dagger}_{\infty})}.
\end{align} 
Suppose $\ell \le 1/8$,  $\beta$ is a constant with $\beta \ell \le 1/2$, $\gamma \in (1,2)$, $\epsilon \in (0, 1)$. 
The initialization satisfies $(1-\ell) \bI  \preceq \bSigma^{(0)} $, $\sym{ \bE^{(0)}_{1,1}} \le \gamma\ell, \sym{ \bE^{(0)}_{2,1}} \le \gamma\ell, \sym{ (\bE^{(0)}_{1,2}; \bE^{(0)}_{2,2} )} \le \ell$, $\bE^{(0)}_{1,2} \ge 0$ and $\bE^{(0)}_{2,2} \ge 0$ entry-wise, and $\| \bN^{(0)}_{-S} \|_{\infty} \le U$ and $\| \bN_S^{(0)} \|_{\infty} \le 2U \le 1/(16 \| (\bAg)^\dagger \|_{\infty}) $.
Furthermore, the parameters satisfy that for any $i \in S$, 
\begin{align}
\eta \left( 1+ 2 r_i R_i \frac{1}{(1-\ell)^2}\frac{\beta \ell^2}{1- \beta \ell - \ell}  +  r_i \frac{2\xexpc}{\ntopic} \rbr{  \alpha +2\nthres  + \frac{\xexpc}{\ntopic}   + \frac{2\xexpc}{\ntopic}  \frac{\beta \ell (1- \beta \ell)}{(1-\ell)^2(1-\beta \ell - \ell)} \right)  }  
\leq \ell \label{eq:diag1} 
\end{align}  
\begin{align} 
  r_i R_i \rbr{ 2  - 2 \frac{1}{(1-\ell)^2}\frac{\beta \ell^2}{1- \beta \ell - \ell}  } 
-  r_i \left( \frac{2\xexpc}{\ntopic} \rbr{  \alpha +2\nthres  + \frac{\xexpc}{\ntopic} \frac{1}{1 - \ell} + \frac{2\xexpc}{\ntopic} \frac{\beta \ell (1- \beta  \ell)}{(1-\ell)^2(1-\beta \ell - \ell)}   } \right)  
\geq 1-\ell \label{eq:diag2} 
\end{align}
\begin{align} 
  h_1 \le \ell, \quad \rbr{\frac{rR}{(1-\ell)^2} + 1 } (\epsilon + h_1) + (\epsilon + h_2) \le (\gamma-1) \ell \label{eqn:neg_eterm}
\end{align} 
\begin{align} 
  \epsilon + h_2 \le (1-\epsilon)(\gamma-1) \ell \label{eqn:pos_eterm}
\end{align} 
\begin{align} 
  h_1 + \ell \le (\beta-1) \ell, \quad h_2 + \rbr{ \frac{r R}{(1 -  \ell)^2}+1}  \ell \le  (\beta -1) \ell \label{eqn:eterm}
\end{align} 
\begin{align}
  3 \| (\bAg)^{\dagger} \|_{\infty} \rbr{ 3 U + \cnoise} \le \rho < \alpha. \label{eqn:noise}
\end{align} 
If we have adversarial noise (Assumption (\textbf{N1})), assume
\begin{align} 
 \epsilon' + U_a \le (1- \epsilon)U, \textnormal{~~and~~} 3 \| (\bAg)^{\dagger} \|_{\infty} \rbr{ 2 U + U_a  + \cnoise} \le \rho < \alpha < 1. \label{eqn:adv_noise}
\end{align} 
If we have unbiased noise (Assumption (\textbf{N2})), assume
\begin{align} 
 \epsilon' + U_n \le (1- \epsilon)U. \label{eqn:unbiased_noise}
\end{align} 
Finally, let $N = \text{poly}\rbr{n, m, 1/\delta, 1/\epsilon}$ sufficiently large.

Then with probability at least $1-\delta$, after $\frac{2\ln\rbr{ \epsilon / (\gamma \ell)} }{\ln(1-\eta)} + \frac{\ln\rbr{ \epsilon' / U} }{\ln(1-\eta)}$ iterations, the output of Algorithm \ref{alg:col_update} is $\bhA = \bAg(\bhSigma + \bhE) + \bhN$ satisfying
\[
  (1-\ell) \bI \preceq \bhSigma \preceq u \bI,  \quad
	\| \bhE_{1, 1} \|_s \le (\gamma  - 1) \ell,  \quad 
	\|  \bhE_{2, 1} \|_s \le (1 - \epsilon)(\gamma - 1) \ell, \quad  
	\| (\bhE_{1, 2} ; \bhE_{2, 2})\|_s \le \ell,
\]
and $\bhE_{1, 2}\ge 0$ and $\bhE_{2, 2} \ge 0$ entry-wise. 
Furthermore, $\|\bhN_{-S}\|_{\infty} \le U$ and $\|\bhN_S\|_{\infty} \le (1-\epsilon)U.$
\end{lem}

\begin{proof}[Proof of Lemma \ref{lem:main_col_update}]
It follows from Lemma~\ref{lem:recurrence} and the conditions~\eqnref{eqn:neg_eterm} and~\eqnref{eqn:pos_eterm}.
\end{proof}

To prove Lemma~\ref{lem:recurrence}, we will first consider how $\bE$ changes after one update step, and then derive the recurrence for all steps in Lemma~\ref{lem:recurrence}. 

\subsubsection{One update step of $\bE$} 
In this subsection, we focus on one update step, bounding the change of $\bE$.
So through out this subsection we will focus on a particular iteration $t$ and omit the superscript $(t)$,  while in the next subsection we will put back the superscript.

For analysis, denote $\bA^{(t)}$ as
\begin{align*}
    \bA & = \bAg (\bSigma + \bE) + \bN
\end{align*}
where $\bSigma$ is a diagonal matrix, $\bE$ is an off-diagonal matrix, and $\bN$ is the component of $\bA$ that lies outside the span of $\bAg$ (e.g., the noise caused by the noise in the sample). 

Recall the following notations:
\begin{align*}
  \forder & = \left( \bSigma + \bE \right)^{-1},
	\\
  \gorder & = \forder - \bSigma^{-1} = \bSigma^{-1} \sum_{k = 1}^{\infty}( - \bE \bSigma^{-1})^k,
	\\
	\xi & = - \bA^\dagger \bN \forder x^* +  \bA^\dagger \noise.
\end{align*}

Consider the update term $\hat{\E} \left[(y - y')(x - x')^{\top}\right]$ and denote it as
\[
  \bDelta =  \hat{\E} \left[(y - y')(x - x')^{\top}\right] = \bAg  (\btSigma + \btE) + \btN
\] 
where $\btSigma$ is a diagonal matrix, $\btE$ is an off-diagonal matrix, and $\bN$ is the component of $\bDelta$ that lies outside the span of $\bAg$.

Since we now use empirical average,  we will have sampling noise. Denote it as
\[
 \bNs  = \hat{\E}[(y - y')(x - x')^{\top}]  - \E[(y - y')(x - x')^{\top}].
\]

Then by definition, for $y = \bAg x^* + \noise$ and $y' = \bAg (x')^* + \noise'$, we have 
\begin{align*}
\hat{\E}[(y - y')(x - x')^{\top}] & =  \E[(y - y')(x - x')^{\top}] + \bNs
\\
& =  \bAg ~\underbrace{ \E\sbr{(x^* - (x')^*) (x - x')^\top} }_{\btSigma + \btE} + \underbrace{\E\sbr{ (\noise - \noise') (x - x')^\top } + \bNs }_{\btN}. 
\end{align*}

Recall the definition of $\bE_{1,1}$, i.e., it is the submatrix of $\bE$ indexed by $S \times S$. Define $\btE_{1, 1}$ similarly, i.e., it is the submatrix of $\btE$ indexed by $S \times S$. Define $\btE_{1, 2}, \btE_{2, 1}$ and $\btE_{2, 2}$ accordingly. 
So in the special case when $S = [s]$ where $s = |S|$, 
\begin{align*}
\btE =  \begin{bmatrix}
\btE_{1, 1} & \btE_{1, 2}  \\
\btE_{2, 1}  & \btE_{2, 2}  \end{bmatrix}.
\end{align*}
We also use the notation $\bM^+$ or $\bM^-$ to denote the positive or negative part of a matrix $\bM$.   

\begin{lem}[Update $\btE_{1, 1}$]\label{lem:main_update_col_up} 
Let $\btE_{1, 1}$ be defined as above. If $\| \xi \|_{\infty} \le \rho < \alpha < 1$ and $\bSigma \succeq (1 - \ell) \bI$, then \\
(1). Negative entries:  
\begin{align*}
 \| \btE_{1, 1}^{-} \|_{s} 
 \le 
 \frac{4 \xexpc^2  \| \forder \|_s (\| \forder \|_s + 1)}{\ntopic^2 (\alpha - \nthres)}  
 +  \frac{8 \xexpc (\xexpc+1)  \rho \| \forder \|_s }{\ntopic (\alpha - \nthres)}.
\end{align*}
(2) Positive entries: 
\begin{align*}
  \| \btE_{1, 1}^+ \|_s 
  \le
  \frac{12 \xexpc(\xexpc+1)  \| \forder \|_s }{\ntopic^2 (\alpha - \nthres)}  \rbr{ \left\|\bZ \right\|_s +  \ntopic \nthres}  
	+ 2 \max_{j \in [\ntopic]} \{\E[(x_j^*)^2] \} \left( \frac{1}{(1 -  \ell)^2}\|\bE_{1, 1}^- \|_s + \frac{\| \bE\|_s^2}{(1 - \ell)^2 (1 - \ell - \|\bE\|_s)} \right).
\end{align*}
\end{lem}

\begin{proof}[Proof of Lemma \ref{lem:main_update_col_up}]
(1) By Lemma~\ref{lem:update_E}, we have
\begin{align*}
  \| \btE_{1, 1}^{-} \|_{s} 
	\le 
	\max\left\{ 
	\frac{4 \xexpc^2  \| \forder \|_s }{\ntopic^2 (\alpha - \nthres)}   \left\|\bZ \right\|_s  +  \frac{4 \xexpc^2  \| \forder \|_s }{\ntopic^2 (\alpha - \nthres)}  \ntopic \nthres, 
  \frac{8\xexpc \nthres}{\ntopic (\alpha - \nthres)}  (\xexpc  + 1) \| \bZ\|_s  + \frac{2 \xexpc^2}{\ntopic^2} \| \bZ\|_s\right\}.
\end{align*}
Observe that for $\alpha < 1$, 
\begin{align*}
  \frac{4 \xexpc^2  \| \forder \|_s (\| \forder \|_s + 1)}{\ntopic^2 (\alpha - \nthres)}   
	\ge 
	\max\cbr{\frac{4 \xexpc^2  \| \forder \|^2_s }{\ntopic^2 (\alpha - \nthres)}, \frac{2 \xexpc^2}{\ntopic^2} \| \bZ\|_s}.
\end{align*}
Moreover, 
\begin{align*}
  \frac{8\xexpc \nthres}{\ntopic (\alpha - \nthres)}  (\xexpc  + 1) \| \bZ\|_s 
	\ge 
  \frac{4 \xexpc^2  \| \forder \|_s }{\ntopic^2 (\alpha - \nthres)} \ntopic \nthres.
\end{align*}
Therefore, 
\begin{align*}
\| \btE_{1, 1}^{-} \|_{s} \le \frac{4 \xexpc^2  \| \forder \|_s
}{\ntopic^2 (\alpha - \nthres)}  +  \frac{8 \xexpc (\xexpc+1)  \rho \| \forder \|_s
}{\ntopic (\alpha - \nthres)}.
\end{align*}

(2) By Lemma \ref{lem:update_E}, when $\bZ_{i, j} < 0$, 
\begin{align*}
  \btE_{j , i} \le  \frac{4 \xexpc^2  \| \forder^i \|_1
}{\ntopic^2 (\alpha - \nthres)}  \rbr{ \left| \forder_{i, j} \right| + \nthres}.
\end{align*}
When $\bZ_{i, j} \ge 0$, 
\begin{align*}
  \btE_{j, i} \le \frac{8\xexpc \nthres}{\ntopic (\alpha - \nthres)} \rbr{ \frac{\xexpc \| \forder^i \|_1 }{n} + \forder_{i, j}} + 2 \E[(x_j^*)^2] \forder_{i, j}
\end{align*}

Consider a fixed $i$. Let $G = \{ j \in S, \bZ_{i, j} \ge 0 \}$ and let $G^c = S - G$. We know that 
\begin{align*}
  \|[\btE_{1, 1}^+ ]_i \|_1 & = 
	\sum_{j \in [\ntopic] } [\btE_{1, 1}^+ ]_{j, i}
	\\
  & \le  \sum_{j \in G^c}  \frac{4 \xexpc^2  \| \forder^i \|_1 }{\ntopic^2 (\alpha - \nthres)}  \rbr{ \left| \forder_{i, j} \right| + \nthres}
	\\
  & \quad + \sum_{j \in G}\left(\frac{8\xexpc \nthres}{\ntopic (\alpha - \nthres)} \rbr{ \frac{\xexpc \| \forder^i \|_1 }{n} + \forder_{i, j}} + 2 \E[(x_j^*)^2] \forder_{i, j} \right)
	\\
  & \le  \frac{4 \xexpc^2  \| \forder \|_s}{\ntopic^2 (\alpha - \nthres)}  \rbr{  \nbr{ \forder}_s + n \nthres} 
	+ \frac{8\xexpc (\xexpc +1) \nthres}{\ntopic (\alpha - \nthres)} \| \forder\|_s +  \sum_{j \in G}  2 \E[(x_j^*)^2] \forder_{i, j} 
	\\
  & \le  \frac{4 \xexpc^2  \| \forder\|_s^2}{\ntopic^2 (\alpha - \nthres)} +  \frac{4 \xexpc^2  \| \forder \|_s}{\ntopic^2 (\alpha - \nthres)} n \nthres
	+  \frac{8\xexpc (\xexpc +1) \nthres}{\ntopic (\alpha - \nthres)} \| \forder\|_s
	+  \sum_{j \in S}  2 \E[(x_j^*)^2] \forder_{i, j} 
	\\
  & \le  \frac{12 \xexpc(\xexpc + 1)  \| \forder \|_s}{\ntopic^2 (\alpha - \nthres)}  \rbr{  \|\forder\|_s + n \nthres} + \sum_{j \in G} 2 \E[(x_j^*)^2] \forder_{i, j}.
\end{align*}
A similar bound holds for $\|[\btE_{1, 1}^+ ]^i \|_1$.

By the definition of $\bZ$, we know that
\begin{align*} 
  \bZ & = (\bSigma + \bE)^{-1} 
	\\
  & = \bSigma^{-1}  \sum_{k = 0}^{\infty} ( - \bE \bSigma^{-1})^k 
	\\
  & = \bSigma^{-1}  -  \bSigma^{-1}  \bE  \bSigma^{-1}  +    \bSigma^{-1} \sum_{k = 2}^{\infty} ( - \bE \bSigma^{-1})^k.
\end{align*}
Therefore, we know that for $i \not= j$, 
\[
  \bZ_{i, j} \le  -  [\bSigma^{-1}  \bE  \bSigma^{-1}]_{i, j}   + |  \sum_{k = 2}^{\infty}   \bSigma^{-1} [( - \bE \bSigma^{-1})^k ]_{i, j} |.
\]
This implies that 
\begin{align*}
  \sum_{j \in G}  \forder_{i, j} 
	& \le \sum_{j \in G}  \left(-  [\bSigma^{-1}  \bE  \bSigma^{-1}]_{i, j}   + \sum_{k = 2}^{\infty}\left|   \bSigma^{-1} [( - \bE \bSigma^{-1})^k ]_{i, j} \right| \right)
	\\
  & \le \frac{1}{(1 -  \ell)^2}\|\bE_{1, 1}^- \|_s + \frac{1}{1 -  \ell}\frac{\frac{\|\bE\|_s^2}{(1 -  \ell)^2}}{1 - \frac{\|\bE\|_s}{1 -  \ell}}
	\\
  & \le  \frac{1}{(1 -  \ell)^2}\|\bE_{1, 1}^- \|_s + \frac{\| \bE\|_s^2}{(1 - \ell)^2 (1 - \ell - \|\bE\|_s)}.
\end{align*}

Putting together, we complete the proof. 
\end{proof}

\begin{lem}[Update $\btE_{2, 1}$] \label{lem:main_update_col_down} 
Let $\btE_{2, 1}$ be defined as above, and suppose $\| \xi \|_{\infty} \le \rho < \alpha < 1$, $\bSigma \succeq (1 - \ell) \bI$ and $\bE_{1,2} \ge 0$, then we have
\begin{align*}	
  \| \btE_{2, 1} \|_s 
  &\le 
	\frac{12 \xexpc(\xexpc+1) \| \forder \|_s}{\ntopic^2 (\alpha - \nthres)}  \rbr{ \left\|\bZ \right\|_s +  \ntopic \nthres}  
	+ 2 \max_{j \in [\ntopic]} \{\E[(x_j^*)^2] \} \left(  \frac{\| \bE\|_s^2}{(1 - \ell)^2 (1 - \ell - \|\bE\|_s)} \right).
\end{align*}
\end{lem}

\begin{proof}[Proof of Lemma \ref{lem:main_update_col_down}]

The proof is almost the same as that of Lemma \ref{lem:main_update_col_up}, combined with the fact that $\bE_{1, 2} \ge 0$ entry-wise. 
\end{proof}

\subsubsection{Recurrence}

Recall that 
\[
  \bA = \bAg (\bSigma + \bE) + \bN 
\]
and recall that $\bE_{1, 1} $ is the submatrix indexed by $S \times S$, and $\bE_{1, 2}, \bE_{2, 1} , \bE_{2, 2} $ are defined according.
Recall that $\bM_S$ denote the submatrix of $\bM$ formed by columns indexed by $S$, and let $\bM_{-S}$ denote the submatrix formed by the other columns.

\begin{lem}[Recurrence] \label{lem:recurrence}
Suppose the conditions in Lemma~\ref{lem:main_col_update} hold. 
Then with probability at least $1-\delta$, after $\frac{2\ln\rbr{ \epsilon / (\gamma \ell)} }{\ln(1-\eta)}$ iterations, 
\begin{align*}
  (1-\ell) \bI & \preceq \bSigma^{(t)},
	\\
  \|(\bE_{1,1}^{(t)})^- \|_s & \le  \epsilon + h_1,
	\\
	\|(\bE_{1,1}^{(t)})^+ \|_s  & \le  \frac{r R}{(1 -  \ell)^2} (\epsilon+h_1)   + h_2 + \epsilon,
	\\
  \|(\bE_{2,1}^{(t)}) \|_s & \le \epsilon + h_2. 
\end{align*}
Also, after $\frac{\ln\rbr{ \epsilon' / U } }{\ln(1-\eta)}$ iterations, for both adversarial and unbiased noise,
\begin{align*}
\nbr{\bN_{-S}^{(t)}}_{\infty} \le U, \quad\nbr{\bN_S^{(t)}}_{\infty} \le (1-\epsilon)U.
\end{align*}
\end{lem}

\begin{proof}[Proof of Lemma~\ref{lem:recurrence}]
We first prove the following claims by induction. \\
(1) $(1-\ell) \bI \preceq \bSigma^{(t)}$, \\
(2) 
\begin{align*}
  \|(\bE_{1,1}^-)^{(t)} \|_s & \leq \gamma \ell
  \\
  \|(\bE_{1,1}^+)^{(t)} \|_s & \leq \frac{rR}{(1- \ell)^2} \gamma\ell + h_2
	\\
  \|\bE_{2,1}^{(t)} \|_s & \leq \gamma \ell
	\\
  \|\bE_{1,2}^{(t)} \|_s & \leq \ell
	\\
	\|\bE_{2,2}^{(t)}\|_s & \leq \ell,
\end{align*}
(3) $ \|\bE^{(t)}\|_s \leq \beta \ell$, \\
(4) for adversarial noise, 
$
  \nbr{ \bN_S^{(t)} }_\infty  \le U + U_a, 
$ 
and 
$
\| \xi^{(t)} \|_{\infty}  \le \rho;
$
or for unbiased noise, 
$
  \nbr{ \bN_S^{(t)} }_\infty  \le U + U_u.
$ 

The basis case for $t=0$ is trivial by assumptions. Now assume they are true for iteration $t$ and show that they are true for iteration $t+1$.

(1) By the update of $\bSigma$, we have 
\begin{align*}
  \bSigma^{(t+1)} = (1-\eta) \bSigma^{(t)} + \eta r \btSigma^{(t)}.
\end{align*}

To lower bound $\bSigma^{(t+1)}_{i,i}$, we will consider two cases, $\bSigma^{(t)}_{i,i} \geq 1$ and $\bSigma^{(t)}_{i,i} \leq 1$.

For $\bSigma^{(t)}_{i,i} \geq 1$, by Lemma~\ref{lem:update_diag},
\begin{align*}
  \btSigma_{i, i } 
	& \ge    
	\E\sbr{\rbr{x_i^*}^2} \rbr{ 2 \bSigma^{-1}_{i, i} - 2 \abr{\gorder_{i,i} } } - \frac{2\xexpc}{\ntopic} \rbr{  \alpha +2\nthres  + \frac{\xexpc}{\ntopic} \bSigma^{-1}_{i, i} + \frac{2\xexpc}{\ntopic}  \nbr{ \sbr{ \gorder}^i  }_1   }
	\\
	& \ge - 2 R_i \abr{\gorder_{i,i}} -   \left( \frac{2\xexpc}{\ntopic} \rbr{  \alpha + 2\nthres  + \frac{\xexpc}{\ntopic} \bSigma^{-1}_{i, i} + \frac{2\xexpc}{\ntopic}  \nbr{ \sbr{ \gorder}^i  }_1   } \right).
\end{align*}

Hence, 
\begin{align*}
  \bSigma^{(t+1)}_{i,i} 
	& \ge (1-\eta) \bSigma^{(t)}_{i,i} - \eta \left( 2 r_i R_i \abr{\gorder^{(t)}_{i,i} } +  r_i\left( \frac{2\xexpc}{\ntopic} \rbr{  \alpha +2\nthres  + \frac{\xexpc}{\ntopic} (\bSigma_{i, i}^{(t)})^{-1} + \frac{2\xexpc}{\ntopic}  \nbr{ \sbr{ \gorder^{(t)}}^i  }_1   } \right) \right) 
	\\
	& \ge 1 - \eta \left( 1+ 2 r_i R_i \abr{\gorder^{(t)}_{i,i} } +  r_i \frac{2\xexpc}{\ntopic} \rbr{  \alpha +2\nthres  + \frac{\xexpc}{\ntopic}   + \frac{2\xexpc}{\ntopic}  \nbr{ \sbr{ \gorder^{(t)}}^i  }_1   } \right) 
	\\
  & \ge 1 - \eta \left( 1+ 2 r_i R_i \frac{1}{(1-\ell)^2}\frac{\beta \ell^2}{1- \beta \ell - \ell}  +  r_i \frac{2\xexpc}{\ntopic} \rbr{  \alpha +2\nthres  + \frac{\xexpc}{\ntopic}   + \frac{2\xexpc}{\ntopic}  \frac{\beta \ell (1-\beta\ell)}{(1-\ell)^2(1-\beta \ell - \ell)} \right)  }.
\end{align*}
where we use the bound on $\bV^{(t)}$.
By condition \eqnref{eq:diag1}, the claim follows. 

For $ \bSigma^{(t)}_{i, i } \le 1$, again by Lemma~\ref{lem:update_diag}, 
\begin{align*}
  \btSigma^{(t)}_{i, i } 
	\ge 
	\E\sbr{\rbr{x_i^*}^2}  \rbr{ 2  - 2 \abr{\gorder^{(t)}_{i,i} } } -    \left( \frac{2\xexpc}{\ntopic} \rbr{  \alpha +2\nthres  + \frac{\xexpc}{\ntopic} (\bSigma_{i, i}^{(t)})^{-1} + \frac{2\xexpc}{\ntopic}  \nbr{ \sbr{ \gorder^{(t)}}^i  }_1   } \right).
\end{align*}
Hence, 
\begin{align*}
  \bSigma^{(t+1)}_{i,i} 
	& = 
	(1-\eta) \bSigma^{(t)}_{i,i} + \eta r \btSigma^{(t)}_{i,i}  
	\\
  & \ge (1-\eta) (1-\ell) 
	\\
  & \quad +\eta \left( r_i R_i \rbr{ 2  - 2 \abr{\gorder^{(t)}_{i,i} } } -  r_i \left( \frac{2\xexpc}{\ntopic} \rbr{  \alpha +2\nthres  + \frac{\xexpc}{\ntopic} (\bSigma_{i, i}^{(t)})^{-1} + \frac{2\xexpc}{\ntopic}  \nbr{ \sbr{ \gorder^{(t)} }^i  }_1   } \right) \right) 
	\\
  & \ge (1-\eta) (1-\ell) +\eta  r_i R_i\rbr{ 2  - 2 \frac{1}{(1-\ell)^2}\frac{\beta \ell^2}{1- \beta \ell - \ell}  }
	\\
	& \quad - \eta  r_i \left( \frac{2\xexpc}{\ntopic} \rbr{  \alpha +2\nthres  + \frac{\xexpc}{\ntopic} \frac{1}{1 - \ell} + \frac{2\xexpc}{\ntopic}   \frac{\beta \ell(1-\beta \ell)}{(1-\ell)^2 (1- \beta \ell -\ell)}   } \right).
\end{align*}    
By condition \eqnref{eq:diag2}, the claim follows. 


(2) By Lemma~\ref{lem:main_update_col_up},
\begin{align*}
  \|(\btE_{1,1}^{(t+1)})^- \|_s & \le \frac{8 \xexpc (\xexpc+1)  \rho \| \forder^{(t)}  \|_s}{\ntopic (\alpha - \nthres)} 
	+ \frac{4 \xexpc^2  \| \forder^{(t)} \|_s (\| \forder^{(t)} \|_s + 1)}{\ntopic^2 (\alpha - \nthres)},
	\\
	\|(\btE_{1,1}^{(t+1)})^+ \|_s & \le  \frac{R}{(1 -  \ell)^2}\|(\bE_{1, 1}^-)^{(t)} \|_s  + \frac{R \| \bE^{(t)} \|_s^2}{(1 - \ell)^2 (1 - \ell - \|\bE^{(t)} \|_s)}  
	\\
  & \quad + \frac{12 \xexpc(\xexpc+1)  \| \forder^{(t)} \|_s}{\ntopic^2 (\alpha - \nthres)}  \rbr{ \left\|\forder^{(t)} \right\|_s +  \ntopic \nthres}.
\end{align*}
By the update rule, we have
\begin{align}
  \|(\bE_{1,1}^{(t+1)})^- \|_s 
  & \le (1-\eta)\|(\bE_{1,1}^{(t)})^-\|_s 
	\nonumber	\\
	& \quad +   r\eta \frac{8 \xexpc (\xexpc+1)  \rho}{(1-\ell - \beta \ell)\ntopic (\alpha - \nthres)} + \frac{4 \xexpc^2  }{(1-\ell - \beta\ell)\ntopic^2 (\alpha - \nthres)} \rbr{\frac{1}{(1-\ell - \beta\ell)} + 1} r\eta, 
	\nonumber \\
	& \le  (1-\eta)\|(\bE_{1,1}^{(t)})^-\|_s  + \eta h_1 
	\label{eqn:neg_eterm_re}
	\\
  \|(\bE_{1,1}^{(t+1)})^+ \|_s 
	& \le 
	(1-\eta) \|(\bE_{1,1}^{(t)})^+ \|_s  + r\eta\frac{R}{(1 -  \ell)^2}\|(\bE_{1, 1}^{(t)})^- \|_s  
	\nonumber \\
  & \quad +  r\eta \frac{R \beta^2 \ell^2 }{(1 - \ell)^2 (1 - \ell - \beta\ell)}  
	\nonumber \\
  & \quad + \frac{12 \xexpc(\xexpc+1)}{\ntopic^2 (\alpha - \nthres)(1- \ell - \beta \ell)}  \rbr{ \frac{1}{1- \ell - \beta \ell} +  \ntopic \nthres} r\eta 
	\nonumber \\
	& \le (1-\eta)\|(\bE_{1,1}^{(t)})^+\|_s  + r\eta\frac{R}{(1 -  \ell)^2}\|(\bE_{1, 1}^{(t)})^- \|_s   + \eta h_2 
	\label{eqn:pos_eterm_re}
\end{align}
where we use $\sym{\bE^{(t)}} \le \beta \ell $ and  $\| \forder^{(t)}  \|_s \le \frac{1}{1 - \ell - \beta\ell}$.

The claim on $\|(\bE_{1,1}^{(t+1)})^- \|_s$ follows from  \eqnref{eqn:neg_eterm_re} and the condition \eqnref{eqn:neg_eterm}.

For $\|(\bE_{1,1}^{(t+1)})^+ \|_s$, by induction \eqnref{eqn:pos_eterm_re} becomes 
\begin{align*}
  \|(\bE_{1,1}^{(t+1)})^+ \|_s 
	& \le   
	(1-\eta) \|(\bE_{1,1}^{(t)})^+ \|_s  + r\eta\frac{R}{(1 -  \ell)^2}\gamma \ell   + \eta h_2 \le \frac{rR}{(1- \ell)^2} \gamma\ell + h_2.
\end{align*}

Now we consider $\|(\bE_{2,1}^{(t+1)}) \|_s$. 
By Lemma~\ref{lem:main_update_col_down}, 
\begin{align}
  \|(\bE_{2,1}^{(t+1)}) \|_s 
	& \le (1-\eta) \|(\bE_{2,1}^{(t)}) \|_s  
	\nonumber \\
	& \quad +  r\eta \frac{R \beta^2 \ell^2 }{(1 - \ell)^2 (1 - \ell - \beta\ell)}  
	\nonumber \\ 
	& \quad +   \frac{12 \xexpc(\xexpc+1)  }{\ntopic^2 (\alpha - \nthres)(1- \ell - \beta \ell)}  \rbr{ \frac{1}{1- \ell - \beta \ell} +  \ntopic \nthres} r\eta
	\nonumber \\
	& = (1-\eta) \|(\bE_{2,1}^{(t)}) \|_s   + \eta h_2 \label{eqn:down_eterm}
	\\
  & \le \gamma \ell \nonumber
\end{align}
where the last line follows by condition \eqnref{eqn:pos_eterm} and induction. 

Finally, clearly we have $\sym{\bE_{1,2}^{(t+1)}} \leq \ell$ and $\sym{\bE_{2,2}^{(t+1)}} \leq \ell$, since they are not updated. 

(3) Note that  \eqnref{eqn:neg_eterm_re} \eqnref{eqn:pos_eterm_re} hold for all iterations up to $t+1$.
Then by Lemma~\ref{lem:rec_diffh}, we have
\begin{align}
	& \quad \|(\bE_{1,1}^{(t+1)})^- \|_s + \|(\bE_{1,1}^{(t+1)} )^+\|_s
	\nonumber\\
	& \le   \max\cbr{
	 \|(\bE_{1,1}^{(0)})^- \|_s +  \|(\bE_{1,1}^{(0)})^+ \|_s ,
	\|(\bE_{1,1}^{(0)})^+ \|_s +  h_1,
	h_2 + \rbr{ \frac{rR}{(1 -  \ell)^2}+1}  \|(\bE_{1,1}^{(0)})^- \|_s,
	h_2 + \rbr{ \frac{rR}{(1 -  \ell)^2}+1}  h_1	
	}. \nonumber
\end{align}

Since $h_1 \le \ell $ and $h_2 \le \ell$ by \eqnref{eqn:neg_eterm}\eqnref{eqn:pos_eterm}, and $\|(\bE_{1,1}^{(0)})^- \|_s +  \|(\bE_{1,1}^{(0)})^+ \|_s \le \ell$ by assumption, we have 
\begin{align}
	 \|(\bE_{1,1}^{(t+1)})^- \|_s + \|(\bE_{1,1}^{(t+1)} )^+\|_s
	& \le \max\cbr{	\ell + h_1, h_2 + \rbr{ \frac{rR}{(1 -  \ell)^2}+1}  \ell }.
\end{align}

Then we have by condition~\eqnref{eqn:eterm}, 
$$
  	\|(\bE_{1,1}^{(t+1)})^- \|_s + \|(\bE_{1,1}^{(t+1)})^+ \|_s
 \le   (\beta -1)\ell, ~~\|\bE^{(t+1)}\|_s \le \beta \ell.
$$

(4) Finally, we consider the noise.
We first consider the adversarial noise. 
Set the sample size $N$ to be large enough, so that
by Lemma~\ref{lem:bound_error}, we have
\begin{align*}
  \abr{ \titime{\btN}{t}_{i,j} }
  & \le \frac{4 \cnoise \xexpc }{(1 - 2 \ell)^2 n  (\alpha - \rho)} + \abr{ [\titime{\btN}{t}_s]_{i,j}} 
	 \le  \frac{8 \cnoise C_1}{n (\alpha - \rho)}
\end{align*}
and thus
\begin{align}
	\nbr{ \titime{\bN}{t+1} }_\infty 
	& \le (1 - \eta)  \nbr{ \titime{\bN}{t} }_\infty +  \eta  \frac{8 r \cnoise C_1}{ \alpha - \rho}. \label{eqn:ad_noise_re}
\end{align}
Then for any $t\ge 0$,
\begin{align*}
  \nbr{ \titime{\bN}{t} }_\infty \le \nbr{ \titime{\bN}{0 } }_\infty + \frac{8 r \cnoise C_1}{ \alpha - \rho} \le U + \frac{8 r \cnoise C_1}{ \alpha - \rho} \le 2U + U_a 
\end{align*}
where the last inequality is by the definition of $U_a$.
On the other hand, by Lemma \ref{lem:noise_term}, we have
\begin{align*}
  \| \xi^{(t)} \|_{\infty} 
	& \le 3 \| (\bAg)^{\dagger} \|_{\infty} (\|\titime{\bN}{t} \|_{\infty} + \cnoise )  
	\\
	& \le  3 \| (\bAg)^{\dagger} \|_{\infty}\rbr{ 2U + \frac{8 r \cnoise C_1}{ \alpha - \rho}  + \cnoise} 
	\\
	& \le \rho
\end{align*}
where the last inequality is due to condition~\eqnref{eqn:adv_noise}.

We now consider the unbiased noise, where the proof is similar. 
Set the sample size $N$ to be large enough, so that
by Lemma~\ref{lem:bound_error}, we have
\begin{align*}
  \abr{ \titime{\btN}{t}_{i,j} }
	& \le \frac{2 \xexpc \cnoise \rho' (1 + \| \bA^{\dagger} \bN^{(t)} \|_{\infty})}{(1 - 2 \ell)  n (\alpha - \rho')}  + \abr{ \sbr{\bN_s}_{i,j} }
	\\
	& \le \frac{8 \xexpc \cnoise^2   \nbr{(\bAg)^\dagger}_{\infty} }{(1 - 2 \ell)  n (\alpha -  2 \cnoise   \nbr{(\bAg)^\dagger}_{\infty})}  + \abr{ \sbr{\bN_s}_{i,j} }
	\\
	& \le \frac{10 \xexpc \cnoise^2   \nbr{(\bAg)^\dagger}_{\infty} }{(1 - 2 \ell)  n (\alpha -  2 \cnoise   \nbr{(\bAg)^\dagger}_{\infty})},
\end{align*}
and thus
\begin{align}
  \nbr{ \titime{\bN}{t+1}_S }_\infty 
	& \le (1 - \eta)  \nbr{ \titime{\bN}{t}_S }_\infty +  \eta  \frac{10 r \xexpc \cnoise^2   \nbr{(\bAg)^\dagger}_{\infty} }{(1 - 2 \ell)  (\alpha -  2 \cnoise   \nbr{(\bAg)^\dagger}_{\infty})}.
	\label{eqn:unbiased_noise_re}
\end{align}
Then for any $t\ge 0$,
\begin{align*}
  \nbr{ \titime{\bN}{t}_S }_\infty \le \nbr{ \titime{\bN_S}{0 } }_\infty + \frac{10 r \xexpc \cnoise^2   \nbr{(\bAg)^\dagger}_{\infty} }{(1 - 2 \ell)  (\alpha -  2 \cnoise   \nbr{(\bAg)^\dagger}_{\infty})} \le 2U + \frac{10 r \xexpc \cnoise^2   \nbr{(\bAg)^\dagger}_{\infty} }{(1 - 2 \ell)  (\alpha -  2 \cnoise   \nbr{(\bAg)^\dagger}_{\infty})} \le 2U + U_n
\end{align*}
where the last inequality is by the definition of $U_n$.
This completes the proof for the claims. 

Now, after proving the claims, we are ready to prove the last statement of the lemma.
First, by~\eqnref{eqn:neg_eterm_re} and Lemma~\ref{l:simplerec}, we have that after $\frac{\ln\rbr{ \epsilon / (\gamma \ell)} }{\ln(1-\eta)}$ iterations, 
\begin{align*}
  \|(\bE_{1,1}^{(t)})^- \|_s \le \epsilon + h_1. 
\end{align*}

Now \eqnref{eqn:pos_eterm_re} becomes
\begin{align}
  \|(\bE_{1,1}^{(t+1)})^+ \|_s  & \le (1-\eta)\|(\bE_{1,1}^{(t)})^+\|_s  + r\eta\frac{R}{(1 -  \ell)^2} (\epsilon+h_1)   + \eta h_2 
\end{align}

After an additional $\frac{\ln\rbr{ \epsilon / (\gamma \ell)} }{\ln(1-\eta)}$ iterations, by Lemma~\ref{l:simplerec},
\begin{align*}
 \|(\bE_{1,1}^{(t)})^+ \|_s  & \le  \frac{r R}{(1 -  \ell)^2} (\epsilon+h_1)   + h_2 + \epsilon
\end{align*}

Similarly, Lemma~\ref{l:simplerec} and \eqnref{eqn:down_eterm}, after $\frac{\ln\rbr{ \epsilon / (\gamma \ell)} }{\ln(1-\eta)}$ iterations,
$$
  \|(\bE_{2,1}^{(t)}) \|_s \le \epsilon + h_2. 
$$

$\nbr{ \titime{\bN}{t}_{-S} }_\infty$ does not change since it is not updated. Now consider $\nbr{ \titime{\bN}{t}_S }_\infty$.

For the adversarial noise, by \eqnref{eqn:ad_noise_re} and Lemma~\ref{l:simplerec},   after $\frac{\ln\rbr{ \epsilon' / U} }{\ln(1-\eta)}$ iterations,
$$
    \nbr{ \titime{\bN}{t}_S }_\infty \le \epsilon' + \frac{8 r \cnoise C_1}{ \alpha - \rho} \le (1-\epsilon)U
$$

where the last inequality is due to condition~\eqnref{eqn:adv_noise}.

For the unbiased noise, by \eqnref{eqn:unbiased_noise_re} and Lemma~\ref{l:simplerec},   after $\frac{\ln\rbr{ \epsilon' / U} }{\ln(1-\eta)}$ iterations,
$$
    \nbr{ \titime{\bN}{t}_S }_\infty \le \epsilon' + \frac{10 r \xexpc \cnoise^2   \nbr{(\bAg)^\dagger}_{\infty} }{(1 - 2 \ell)  (\alpha -  2 \cnoise   \nbr{(\bAg)^\dagger}_{\infty})} \le (1-\epsilon)U
$$

where the last inequality is due to condition~\eqnref{eqn:unbiased_noise}.

This completes the proof.
\end{proof}

\subsection{Equilibration: Rescale} \label{sec:rescale}

The input of of Algorithm \ref{alg:rescale} can be written as $\bA^{(0)} = \bAg (\bSigma^{(0)} + \bE^{(0)}) + \bN^{(0)}$.
The output $\bhA$ can be written as $\bhA = (\bAg \bD) (\bhSigma + \bhE) + \bhN$ where $\bhSigma$ is diagonal, and $\bhE$ is off diagonal, and $\bD$ is a diagonal matrix with  $\bD_{i,i} = \frac{1}{1 - \epsilon}$ for $i \in S$ and the rest being $1$.  
Recall that for a matrix $\bM$, let $\bM_{1, 1}$ denote the submatrix of $\bM$ indexed by $S \times S$, and define $\bM_{1, 2}, \bM_{2, 1}$ and $\bM_{2, 2}$ accordingly.
Also recall that $\bM_S$ denote the submatrix of $\bM$ formed by columns indexed by $S$, and let $\bM_{-S}$ denote the submatrix formed by the other columns.

\begin{lem}[Main: Rescale] \label{lem:rescale}
Let $\bA^{(0)} = \bAg (\bSigma^{(0)} + \bE^{(0)}) + \bN^{(0)}$ satisfies the condition in Lemma \ref{lem:main_col_update} and $\epsilon$ be defined as in Lemma \ref{lem:main_col_update}. Then the output of Algorithm~\ref{alg:rescale} is $\bhA = ( \bAg \bD) (\bhSigma + \bhE)  + \bhN$ satisfying
\begin{align*}
  (1 - \ell) \bI \preceq \bhSigma, \quad
  \| \bhE_{1, 1} \|_s \le (\gamma  - 1) \ell,  \quad 
	\|  \bhE_{2, 1} \|_s  \le (\gamma  - 1) \ell, \quad  
	\| (\bhE_{1, 2},  \bhE_{2, 2})\|_s \le \ell, \quad 
	\| \bhN_S\|_{\infty}  \le U, \quad \| \bhN_{-S}\|_{\infty}  \le U.
\end{align*}
Moreover, $\bhE_{1, 2} \ge 0$ and $\bhE_{2, 2} \ge 0$ entry-wise. 
\end{lem}

\begin{proof}[Proof of Lemma \ref{lem:rescale}]
Note that $\btA = \bAg(\btSigma + \btE) + \btN$ for a diagonal matrix $\btSigma$, off-diagonal matrix $\btE$ and error matrix $\btN$. By lemma \ref{lem:main_col_update}, we have
$\btSigma \succeq (1 - \ell ) \bI$, error matrix $\|\btN_S\|_{\infty} \le (1-\epsilon)U$ and 
\[
  \| \btE_{1, 1} \|_s \le (\gamma  - 1) \ell,  \quad 
	\|  \btE_{2, 1} \|_s \le (1 - \epsilon)(\gamma - 1) \ell, \quad  
	\| (\btE_{1, 2}; \btE_{2, 2} )\|_s \le \ell
\]
and $\btE_{1, 2} \ge 0$ and $\btE_{2, 2} \ge 0$ entry-wise.

Therefore, by the rescaling rule: 
\begin{align*}
  \bhA & = \btA \bD =  \bAg(\btSigma + \btE) \bD + \btN \bD
	\\
  & = \bAg \bD (\btSigma + \bD^{-1 } \btE \bD) +  \btN \bD.
\end{align*}
Therefore, $\bhSigma = \btSigma \succeq (1 - \ell ) \bI$, $\|\bhN_S\|_{\infty} \le \frac{1}{1 - \epsilon} \| \btN_S \|_{\infty} \le U$. $\|\bhN_{-S}\|_{\infty} = \| \btN_{-S} \|_{\infty} \le U$ since it is not updated. 

For the $\bhE$ term, denote $\bD_1 = \diag\left( \frac{1}{1 - \epsilon}, \ldots,  \frac{1}{1 - \epsilon} \right) \in \mathbb{R}^{s \times s}$. We know that 
\begin{align*}
  \bhE_{1, 1} & = \bD_1^{-1} \btE_{1, 1} \bD_1 =\btE_{1, 1} 
	\\
  \bhE_{2,1} & = \btE_{2,1} \bD_1 = \frac{1}{1 - \epsilon} \btE_{2,1} 
	\\
  \bhE_{1,2} & = \bD_1^{-1}\btE_{1,2}  =(1 - \epsilon)\btE_{1,2} 
	\\
  \bhE_{2, 2} & = \btE_{2, 2}.
\end{align*}
This leads to 
\[
  \| \bhE_{1, 1} \|_s \le (\gamma  - 1) \ell,  \quad 
	\|  \bhE_{2, 1} \|_s \le(\gamma - 1) \ell, \quad  
	\| (\bhE_{1, 2}, \bhE_{2, 2} )\|_s \le \ell,
\]
with $\bhE_{1, 2}, \bhE_{2, 2} \ge 0$.
This completes the proof. 
\end{proof}

\subsection{Equilibration: Main algorithm} \label{sec:equilibration_main_algo}

\begin{lem}[Main: Equilibration] \label{lem:main_col_meta} 
Suppose the conditions in Lemma ~\ref{lem:rescale} each time Algorithm~\ref{alg:rescale}. Additionally, there exists constant $0 < b < 1$, $\kappa > 1$ and $u > 1$ such that $b\kappa > 1$ such that the initial $\lambda \ge \max_{ i \in [\ntopic] } \E[(x^*_i)^2]/b$, and the initial $\bSigma \preceq u \bI$. Furthermore, for any $\lambda \ge \min_{i\in [\ntopic]} \E[(x^*_i)^2] / \kappa$, 
\begin{align}
  \left( \frac{1}{1 - \ell} + h_6\right)^2 b  \lambda  + h^2_5 b \kappa \lambda  + h_3  & \le \rbr{ 1 - \frac{1}{100} } \lambda, 
	\label{eqn:est1}
	\\
  \left( \frac{1}{u} - h_6 \right)^2 (1-\epsilon)b\kappa \lambda - h^2_5   b\kappa \lambda - h_4 & \ge \rbr{ 1 + \frac{1}{100} }  \lambda, \quad \frac{1}{u} > h_6
	\label{eqn:est2}
	\\
	h_3 & \le \frac{1}{200} \min_{i\in [\ntopic]} \E[(x^*_i)^2], \label{eqn:est3}
	\\
	h_4 & \le \frac{1}{200} \min_{i\in [\ntopic]} \E[(x^*_i)^2], \label{eqn:est4}
\end{align}
where 
\begin{align*}
  h_3 & =  \frac{\xexpc^2 }{\ntopic^2} h_5 \left(h_5 + \frac{2}{1-\ell} \right),
	\\
	h_4 & =  \frac{\xexpc^2 }{\ntopic^2} h_5 \left(h_5 + \frac{2}{1-\ell} \right) + \frac{2 (\alpha + \nthres) \xexpc}{\ntopic (1-\ell)},
	\\
  h_5 & = \frac{(\gamma + 1) \ell (1 - (\gamma+1) \ell)}{(1-\ell)^2 (1 - (\gamma + 2) \ell)},
	\\
	h_6 & = \frac{(\gamma + 1)\ell^2}{(1 - \ell)^2 (1 - (\gamma + 2) \ell)}.
\end{align*}
Finally, set $N = \text{poly}(1/\min_{i\in [\ntopic]} \E[(x^*_i)^2], n, 1/\delta)$ large enough.

Then with probability at least $1-\delta$, the following hold.
During the execution of the algorithm, for any $j \in S$,
\begin{align*}
	\rbr{ \left( \frac{1}{u} - h_6 \right)^2 - \kappa h^2_5 - \frac{1}{100} } \frac{\E[(x_j^*)^2]}{(\bD_{j,j})^2} \le m_j \le \rbr{ \left( \frac{1}{1 - \ell} + h_6\right)^2 + \kappa h^2_5 + \frac{1}{100} } \frac{\E[(x_j^*)^2]}{(\bD_{j,j})^2}.
\end{align*}
Furthermore, the output of Algorithm~\ref{alg:meta} is $\bA = \bAg \bD (\bSigma + \bE) + \bN$
where $\bSigma$ is diagonal and $(1 - \ell ) \bI \preceq \bSigma$, $\bE$ is off diagonal and $\| \bE\|_s \le \gamma \ell$, $\bN$ satisfies $\| \bN \|_{\infty} \le 2 U$, and
\begin{align*} 
  \frac{\max_{i \in [\ntopic]} \frac{1}{\bD_{i,i}^2}\E[(x_i^*)^2]}{\min_{j \in [\ntopic]} \frac{1}{\bD_{j,j}^2}\E[(x_j^*)^2]}  \le \kappa.
\end{align*}
\end{lem}

\begin{proof}[Proof of Lemma \ref{lem:main_col_meta}]
We prove the lemma by induction.
For notational convenience, let us introduce a counter $(p)$ denoting the number of times the inner while cycle has been executed, and denote $\bA$ as $\bA^{(p)}$.
Recall that for a matrix $\bM \in \Real^{n \times n}$ and index set $S \subseteq [n]$, let $\bM_{1,1}$ denote the submatrix indexed by $S \times S$, and 
$\bM_{1,2}, \bM_{2,1}$ and $\bM_{2,2}$ are defined accordingly. Also, let $\bM_S$ denote the submatrix formed by the columns indexed by $S$, and $\bM_{-S}$ the submatrix formed by the other columns.

Our inductive claims are as follows. At the beginning of each inner while cycle,
\begin{align*}
  \bA^{(p)} & = \bAg \bD^{(p)} \left(\bSigma^{(p)} + \bE^{(p)} \right) + \bN^{(p)}
\end{align*}
where $\bD^{(p)}$ and $\bSigma^{(p)}$ are diagonal,
$\bE^{(p)}$ are off diagonal satisfying 
\begin{itemize}
	\item[(1)] $ (1-\ell) \bI \preceq \bSigma^{(p)}$, 
	\item[(2)] $\bE^{(p)}_{1, 2} \ge 0$ and $\bE^{(p)}_{2, 2} \ge 0$ entry-wise and
\begin{align*}
  \sym{ \bE^{(p)}_{1, 1} } & \leq \gamma \ell,
	\\
  \sym{ \bE^{(p)}_{2, 1} } & \leq \gamma \ell,
	\\
  \sym{ (\bE^{(p)}_{1, 2}; \bE^{(p)}_{2, 2}) }  & \le \ell,
\end{align*}   
  \item[(3)] $\bN_{-S}^{(p)}\le U$ and $\bN_S^{(p)} \leq 2 U$, 
  \item[(4)] We have 
		\begin{itemize}
		\item[(a)] When $\E[(x_{j}^*)^2] < b \lambda^{(p)}, j \notin S$, then $m_j \le \lambda^{(p)}$,
		\item[(b)] When $\E[(x_{j}^*)^2] \ge (1 - \epsilon) b \kappa \lambda^{(p)}, j \notin S$, then $m_j > \lambda^{(p)}$,
		\end{itemize}
		and consequently, 
		\begin{itemize} 
			\item[(c)] $\forall i \in S, b \lambda^{(p)} \le \frac{\E[(x_i^*)^2] }{\left(\bD^{(p)}_{i,i} \right)^2}$,
			\item[(d)] $\forall i \in [\ntopic], \frac{\E[(x_i^*)^2] }{\left(\bD^{(p)}_{i,i} \right)^2} \le b \kappa \lambda^{(p)} $.
		\end{itemize}
\end{itemize}

The claims are trivially true at initialization, so we proceed to the induction. Assume the claim is true at time $p$, we proceed to show it is true at time $p+1$. 

First, consider (1), (2) and (3). 
By Lemma \ref{lem:rescale}, after applying the rescaling algorithm,
$(1 - \ell) \bI \preceq \bSigma^{(p)}$ and
\begin{align*}
  \| \bE^{(p)}_{1, 1} \|_s \le (\gamma  - 1) \ell,  \quad 
	\|  \bE^{(p)}_{2, 1} \|_s  \le (\gamma  - 1) \ell, \quad  
	\| (\bE^{(p)}_{1, 2},  \bE^{(p)}_{2, 2})\|_s \le \ell, \quad 
	\| \bN^{(p)}_S\|_{\infty}  \le U, \quad \| \bN^{(p)}_{-S}\|_{\infty}  \le U.
\end{align*}
Moreover, $\bE^{(p)}_{1, 2} \ge 0$ and $\bE^{(p)}_{2, 2} \ge 0$ entry-wise. 
Observe that when moving from time $p$ to $p+1$, potentially the algorithm includes new elements in $S$. 
Then
\begin{align*}
  \| \bE^{(p + 1)}_{1, 1} \|_s  \le \| \bE^{(p)}_{1, 1} \|_s  + \max\{ \| \bE^{(p)}_{2, 1} \|_s , \| \bE^{(p)}_{1, 2} \|_s \} \le (\gamma - 1) \ell+ \ell = \gamma \ell
\end{align*}
Where the last inequality used the fact that $\gamma < 2$. Similarly, 
\begin{align*}
  \| \bE^{(p + 1)}_{2, 1} \|_s  \le \| \bE^{(p)}_{2, 1} \|_s  + \| \bE^{(p)}_{2, 2} \|_s\le (\gamma - 1) \ell+ \ell = \gamma \ell.
\end{align*}
Also, $\|(\bE_{1, 2}^{(p+1)}, \bE_{2, 2}^{(p+ 1)} ) \|_s \le \|(\bE_{1, 2}^{(p)}, \bE_{2, 2}^{(p)} ) \|_s \le \ell$, and $(\bE_{1, 2}^{(p+1)}, \bE_{2, 2}^{(p+ 1)} ) \ge 0$ entry-wise. 
Furthermore, $\| \bN^{(p+1)}_{-S}\|_{\infty}  \le \| \bN^{(p)}_{-S}\|_{\infty}  \le U$ and 
\begin{align*}
	\| \bN^{(p+1)}_S\|_{\infty}  \le \| \bN^{(p)}_S\|_{\infty} + \| \bN^{(p)}_{-S}\|_{\infty} \le 2U.
\end{align*}
Hence, (1), (2) and (3) are also true at time $(p+1)$.

Finally, we proceed to (4). Since (a)(b) are true at time $p$, (c)(d) are true at time $p+1$. \footnote{Note that in (b), the factor $(1-\epsilon)$ is needed to ensure (d) is true at time $p+1$.} 
Furthermore, when $\lambda \le \min_{i\in [\ntopic]} \E[(x^*_i)^2] /\kappa$, it is guaranteed that all $[n] \subseteq S$, so we only need to prove that when $\lambda \ge \min_{i\in [\ntopic]} \E[(x^*_i)^2] /\kappa$, (a)(b) are also true at time $p+1$. 

To prove (a)(b) are true at time $p+1$, we will use Lemma \ref{lem:est_f_w}. Note that since $\bA$ has been scaled, so $\bAg \bD$ should be regarded as the ground truth matrix $\bAg$ in Lemma \ref{lem:est_f_w}. 
We first make sure its assumption is satisfied. First, $\nbr{\bN}_\infty \le 3U$ and $\nbr{(\bAg \bD)^\dagger}_\infty \le \nbr{(\bAg)^\dagger}_\infty$.
By Lemma~\ref{lem:noise_term} and condition~\eqnref{eqn:noise}, the assumption in Lemma \ref{lem:est_f_w} is satisfied. 

We are now ready to prove (a). By Lemma \ref{lem:est_f_w},
\begin{align*}
  \E[x_j^2] 
	& \le  
	\left(\bSigma^{-1}_{j,j} + |\bV_{j,j}| \right)^2 \frac{\E[(x_j^*)^2] }{\bD^{(p+1)}_{j,j}}
	+  \| [\bV]^j \|^2_2 \max_{k \in [n]} \frac{\E[(x_k^*)^2] }{\bD^{(p+1)}_{k,k}}
	+  \| [\bV]^j \|_1 \rbr{ \| [\bV]^j \|_1 + 2 \bSigma^{-1}_{j,j} } \frac{\xexpc^2}{\ntopic^2}.
\end{align*}
By Lemma~\ref{lem:higher_order}, $|\bV_{j,j}| \le h_6$,  $\| [\bV]^j \|^2_2 \le \| [\bV]^j \|^2_1 \le h^2_5$, so
\begin{align*}
  \E[x_{j}^2] & \le \left( \frac{1}{1 - \ell} + h_6 \right)^2 \frac{\E[(x_j^*)^2] }{(\bD^{(p+1)}_{j,j})^2}  + h^2_5  \max_{k \in [n]} \frac{\E[(x_k^*)^2] }{(\bD^{(p+1)}_{k,k})^2} + h_3. 
\end{align*}
By (d), $ \max_{k \in [n]} \frac{\E[(x_k^*)^2] }{(\bD^{(p+1)}_{k,k})^2}  \le b\kappa \lambda$, so for any $j \notin S$ with $\E[(x_j^*)^2] = \frac{\E[(x_j^*)^2] }{(\bD^{(p+1)}_{j,j})^2} < b \lambda$, we have
\begin{align*}
  \E[x_{j}^2]  & \le  \left( \frac{1}{1 - \ell} + h_6\right)^2 b  \lambda  + h^2_5 b \kappa \lambda  + h_3. 
\end{align*}
By using large enough sample, with high probability, the empirical estimation 
\begin{align*}
  \hat\E[x_{j}^2]  & \le  \E[x_{j}^2] + \frac{1}{100}\lambda \le \lambda
\end{align*}
where the last step is by condition \eqnref{eqn:est1}. 

As for (b), by Lemma \ref{lem:est_f_w} we have
\begin{align*}
  \E[x_j^2] 
	& \ge  \left(\bSigma^{-1}_{j,j}  - |\bV_{j,j}| \right)^2  \frac{\E[(x_j^*)^2] }{(\bD^{(p+1)}_{j,j})^2} - \| [\bV]^j \|_2^2 \max_{k \in [n]} \frac{\E[(x_k^*)^2] }{(\bD^{(p+1)}_{k,k})^2}  - \left( \frac{\xexpc^2 }{\ntopic^2}  \| [\bV]^j \|_1  ( \| [\bV]^j \|_1  + 2 \bSigma^{-1}_{j,j} ) + \frac{2 (\alpha + \nthres) \xexpc}{\ntopic} \bSigma^{-1}_{j,j} \right)
	\\
	& \ge \left( \frac{1}{u} - h_6 \right)^2 \frac{\E[(x_j^*)^2] }{(\bD^{(p+1)}_{j,j})^2} - h^2_5  \max_{k \in [n]} \frac{\E[(x_k^*)^2] }{(\bD^{(p+1)}_{k,k})^2} - h_4.
\end{align*}
The last step uses that $\bSigma^{-1}_{j,j} \le u$, which is by the initial condition assumed and that it is not updated for $j \not\in S$.
Putting in the bound that $ \frac{\E[(x_k^*)^2] }{(\bD^{(p+1)}_{k,k})^2}   \le b\kappa \lambda$, then for any $j\notin S$ with $\E[(x_j^*)^2] = \frac{\E[(x_{j}^*)^2]}{(\bD_{j,j}^{(p+1)})^2}  \ge (1 - \epsilon) b \kappa \lambda$, we have 
\begin{align*}
  \E[x_{j}^2]  
	& \ge \left( \frac{1}{u} - h_6 \right)^2 (1-\epsilon)b\kappa \lambda - h^2_5   b\kappa \lambda - h_4.
\end{align*}
Again, use large enough sample to ensure that with high probability
\begin{align*}
  \tilde\E[x_{j}^2]  \ge \E[x_{j}^2]  - \frac{1}{100}\lambda 	\ge \lambda
\end{align*}
where the last step follows from condition \eqnref{eqn:est2}.
This completes the proof of the induction.

We now prove the statements of the lemma. The statement about the output follows from the above claims. What is left is to prove that $m_j (j \in S)$ approximates $\E[(x_j^*)^2]$ well.
Since $m_j$ for $j \in S$ is updated along with $\bD_{j,j}$, we only need to check the right after adding $j$ to $S$, the statement holds. 
Suppose the time point is $p$, we have 
\begin{align*}
  \E[x_{j}^2] & \le \left( \frac{1}{1 - \ell} + h_6 \right)^2 \frac{\E[(x_j^*)^2] }{(\bD^{(p)}_{j,j})^2}  + h^2_5  \max_{k \in [n]} \frac{\E[(x_k^*)^2] }{(\bD^{(p)}_{k,k})^2} + h_3. 
\end{align*}
Since $j$ is in $S$, by the claims (c)(d) we have 
\[
  \max_{k \in [n]} \frac{\E[(x_k^*)^2] }{(\bD^{(p)}_{k,k})^2} \le b\kappa \lambda^{(p)} \le \kappa \frac{\E[(x_j^*)^2] }{(\bD^{(p)}_{j,j})^2}.
\]
Since $N$ is large enough so that 
\[
  \E[x_{j}^2] \le \E[(x_{j}^*)^2] \rbr{ 1 + \frac{1}{200}}. 
\]
Combined these with the condition \eqnref{eqn:est3}, we have
\[
 m_j \le \rbr{ \left( \frac{1}{1 - \ell} + h_6\right)^2 + \kappa h^2_5 + \frac{1}{100} } \frac{\E[(x_j^*)^2]}{(\bD_{j,j})^2}.
\]
The upper bound on $m_j$ can be bounded similarly. This completes the proof of the lemma.
\end{proof}

The following is the lemma used in the proof of Lemma~\ref{lem:main_col_meta}. 
\begin{lem}[Estimate of feature weight] \label{lem:est_f_w}
Suppose $\abr{\xi_i} \le \nthres < \alpha$ for any example and every $i \in [\ntopic]$, and suppose $\bSigma \succeq \frac{1}{2} \bI$. 
Then
\begin{align*}
  \E[x_i^2] 
	& \ge  \left(\bSigma^{-1}_{i, i}  - |\bV_{i, i}| \right)^2 \E[(x_i^*)^2] - \| [\bV]^i \|_2^2 \max_{j \in [n]} \E[(x_j^*)^2]
	\\
  & \quad - \left( \frac{\xexpc^2 }{\ntopic^2}  \| [\bV]^i \|_1  ( \| [\bV]^i \|_1  + 2 \bSigma^{-1}_{i, i} ) + \frac{2 (\alpha + \nthres) \xexpc}{\ntopic} \bSigma^{-1}_{i, i} \right)
	\\
  \E[x_i^2] 
	& \le  
	\left(\bSigma^{-1}_{i,i} + |\bV_{i, i}| \right)^2 \E[(x_i^*)^2] 
	+  \| [\bV]^i \|^2_2 \max_{j \in [n]} \E[(x_j^*)^2] 
	+  \| [\bV]^i \|_1 \rbr{ \| [\bV]^i \|_1 + 2 \bSigma^{-1}_{i,i} } \frac{\xexpc^2}{\ntopic^2}.
\end{align*}
\end{lem}

\begin{proof}[Proof of Lemma \ref{lem:est_f_w}]
By the decoding rule, 
\begin{align*}
  x_i & = \left[\phi_{\alpha}(\bA^{\dagger} [\bAg x^* + \nu] )\right]_i
	\\
  & = \left[\phi_{\alpha}\rbr{ \rbr{ \bSigma^{-1}   +  \gorder } x^*  + \xi }   \right]_i.
\end{align*}
Let $[\bV]^i = v$ and $\Sigma_{i, i}^{-1} = \sigma$, then we can rewrite above as
\begin{align*}
  x_i = \phi_{\alpha } (\sigma x_i^* + \langle v, x^* \rangle + \xi_i)
\end{align*}
which implies that 
\begin{align} 
 \sigma x_i^*  + \langle v, x^* \rangle - \rho  - \alpha 
\le x_i \le 
\abr{ \sigma x_i^* +  \langle v, x^* \rangle}.
  \label{eq:ulbound} 
\end{align}

First, consider the lower bound.
\begin{align*}
  \E[x_i^2] & \ge \E\left[\left(\sigma x_i^*  + \langle v, x^* \rangle  -\rho  - \alpha  \right) \phi_{\alpha} (\sigma x_i^* + \langle v, x^* \rangle + \xi_i) \right]
\end{align*}

The following simple lemma is useful. 
\begin{claim}
Let $\chi$ be a variable such that $|\chi| \le \alpha $, then for every $w \in \mathbb{R}^{\ntopic}$, $k \in [\ntopic]$, 
\begin{align}
  \E[x_k^* \phi_\alpha(\langle w, x^* \rangle + \chi)]  & \le  |w_k| \E[(x_k^*)^2] +  \frac{\xexpc^2}{n^2}\sum_{j \not= k} |w_j |
	\\
	& \le  |w_k| \E[(x_k^*)^2] +  \frac{\xexpc^2}{n^2}\nbr{w }_1.
\label{eqn:fw_claim}
\end{align}
\end{claim}
\begin{proof}
The proof is a direct observation that when $|\chi|  < \alpha$, 
$$
  \phi_\alpha(\langle w, x \rangle + \chi) \le | \langle w, x \rangle| \le \langle |w|, x \rangle
$$ 
where $|w|$ is the entry wise absolute value. 
\end{proof}

Therefore, we can obtain the following bounds.

(1). By \eqnref{eqn:diag_update1} in Lemma~\ref{lem:update_diag}, we have
\begin{align*}
  \E[x_i^* \phi_{\alpha } (\sigma x_i^* + \langle v, x^* \rangle + \xi_i) ] \ge  \bSigma^{-1}_{i, i}\E\sbr{\rbr{x_i^*}^2} - \frac{(\alpha + \nthres) \xexpc}{\ntopic} - \E \sbr{ \rbr{x_i^*}^2 }  \abr{\gorder_{i,i} }  - \frac{\xexpc^2}{\ntopic^2}  \nbr{ \sbr{ \gorder}^i  }_1, 
\end{align*}
(2). By \eqnref{eqn:fw_claim} in the above claim,  
\begin{align*}
  \E[x_j^* \phi_{\alpha } (\sigma x_i^* + \langle v, x^* \rangle + \xi_i) ] \le |v_j| \E[(x_j^*)^2] + \frac{\xexpc^2}{n^2} (\| v\|_1 + \sigma),
\end{align*}
(3). By \eqnref{eq:ulbound}, for $j \neq i$, 
\begin{align*}
  \E[ \phi_{\alpha } (\sigma x_i^* + \langle v, x^* \rangle + \xi_i) ] 
	\le 
	\E[ |\sigma x_i^* + \langle v, x^* \rangle |] 
	\le 
	\frac{(\sigma + \| v\|_1) \xexpc}{\ntopic}.
\end{align*}
Putting together, we can obtain
\begin{align*}
  \E[x_i^2] 
	& \ge  \left(\bSigma^{-1}_{i, i}  - |\bV_{i, i}| \right)^2 \E[(x_i^*)^2] - \| [\bV]^i \|_2^2 \max_{j \in [n]} \E[(x_j^*)^2]
	\\
  & \quad - \left( \frac{\xexpc^2 }{\ntopic^2}  \| [\bV]^i \|_1  ( \| [\bV]^i \|_1  + 2 \bSigma^{-1}_{i, i} ) + \frac{2 (\alpha + \nthres) \xexpc}{\ntopic} \bSigma^{-1}_{i, i} \right).
\end{align*}

Second, we proceed to the upper bound. Similarly as the lower bound, by \eqnref{eq:ulbound}, we have 
\begin{align*} 
  \E[x_i^2] 
  & \leq 
  \E\left[\left( (|v_i| + \sigma) x_i^* + \sum_{j \neq i} |v_j| x_j^*  \right) \phi_{\alpha } (\sigma x_i^* + \langle v, x^* \rangle + \xi_i) \right] 
	\\ 
  & = (|v_i| + \sigma) \E[x_i^* \phi_{\alpha } (\sigma x_i^* + \langle v, x^* \rangle + \xi_i)] + \sum_{j \neq i} |v_j| \E[x_j^* \phi_{\alpha} (\sigma x_i^* + \langle v, x^* \rangle + \xi_i)  ].
\end{align*}  

For the first summand, same as in (2), by \eqnref{eqn:fw_claim} in the above claim we get 
\begin{align*}
   \E[x_i^* \phi_{\alpha } (\sigma x_i^* + \langle v, x^* \rangle + \xi_i)] 
	& \le (\sigma + |v_i|) \E[(x_i^*)^2] + \frac{\xexpc^2}{\ntopic^2} \|v\|_1,
	\\ 
  \E[x_j^* \phi_{\alpha} (\sigma x_i^* + \langle v, x^* \rangle + \xi_i)  ] 
  & \le 
	|v_j| \E[(x_j^*)^2] + \frac{\xexpc^2}{n^2} (\| v\|_1 + \sigma).
\end{align*} 

Therefore, we get
$$
  \E[x_i^2] \le \left( \bSigma^{-1}_{i,i}    + |\bV_{i, i}| \right)^2 \E[(x_i^*)^2] +   \| [\bV]^i \|_1(\| [\bV]^i \|_1 + 2 \bSigma^{-1}_{i,i}) \frac{\xexpc^2}{\ntopic^2} +  \| [\bV]^i \|^2_2 \max_{j \in [n]} \E[(x_j^*)^2].
$$ 

which completes the proof.
\end{proof}

\subsection{Main theorem} \label{sec:equilibration_thm}
 
\equilibration*

\begin{proof}[Proof of Theorem~\ref{thm:equilibration}]

The theorem follows from Lemma~\ref{lem:main_col_meta} (taking union bound over all the iterations and setting a proper $\delta$), if the conditions are satisfied. So in the following, we first specify the parameters and then verify the conditions in Lemma~\ref{lem:main_col_update} and Lemma~\ref{lem:main_col_meta}.

Recall that $\ell = 1/50$. Define $u = 1+\ell$, $\gamma = 3/2$, $\beta = 4$, $\kappa = 2$, $b=3/4$, and let $\epsilon < 1/1000$.

\paragraph{Conditions in Lemma~\ref{lem:main_col_update}.} For \eqnref{eq:diag1}, we need to compute $r_i R_i$ and the the third term. Note that by the induction in Lemma~\ref{lem:main_col_meta}, the $m_j$ is an good approximation of $\E[(x_j^*)^2]/(\bD_{j,j})^2$. Furthermore, when Lemma~\ref{lem:main_col_update} is applied in Lemma~\ref{lem:main_col_meta}, it is applied on the ground-truth matrix $(\bAg)' = \bAg \bD$ and $(x_j^*)' = x_j^* / \bD_{j,j}$, so $m_j$ is a good approximation of $\E[((x_j^*)')^2]$. Then 
\[
 r_i R_i  = \frac{3 \E[((x_i^*)')^2] }{5 m_i} \le 
\frac{3}{5	\rbr{ \left( \frac{1}{u} - h_6 \right)^2 - \kappa h^2_5 - \frac{1}{100} }  }. 
\]
For the third term, first note that $ C_1^3 \le \abscc c_2^2 n$, and thus $C_1^2 \le \abscc c_2 n$ by $C_1 > c_2$. Furthermore, $r_i = O(1/m_i) = O(n/c_2)$ for $i \in S$. Plugging in the parameters, we know that the third term is less than $1/1000$ when $\abscc$ is sufficiently small. Then \eqnref{eq:diag1} can be verified by plugging the parameters.

Similarly, for \eqnref{eq:diag2}, we can compute $r_i R_i$ and let $\abscc$ small enough so that the second term is less than $1/1000$, and then verify the condition.

For \eqnref{eqn:neg_eterm} \eqnref{eqn:pos_eterm} and \eqnref{eqn:eterm}, we need to bound $h_1$ and $h_2$, which in turn relies on $r$ and $rR$. Since for $i \in S$, $r_i = O(n/c_2)$, $r = O(n/c_2)$. Then similar to the argument as above, $h_1 < 2/10000$ when $\abscc$ is sufficiently small. when Lemma~\ref{lem:main_col_update} is applied in Lemma~\ref{lem:main_col_meta}, it is applied on the ground-truth matrix $(\bAg)' = \bAg \bD$ and $(x_j^*)' = x_j^* / \bD_{j,j}$. By the induction claims there,
$\max_{j\in [n]} \E[((x_j^*)')^2] $ differ from $\min_{j\in S} \E[((x_j^*)')^2]$ by a factor of at most $\kappa$, so $rR \le \frac{3\kappa}{5}$. So the first term can be computed. The second term is less than $1/10000$ when $\abscc$ is small enough. Then $h_2$ can be computed. And the conditions can be verified.

Condition \eqnref{eqn:noise} is true since $\max\cbr{\cnoise, \|\bN\|_\infty} = O(\frac{c_2^2}{C_1^3 \|(\bAg)^\dagger\|_\infty})$.
Condition \eqnref{eqn:adv_noise} is true by setting $\epsilon' < U/8$ and by $U_a < U/8$ and $U = \|\bN\|_\infty \le O(\frac{c_2^2}{C_1^3 \|(\bAg)^\dagger\|_\infty})$. 
Similarly, condition \eqnref{eqn:adv_noise} is true by setting $\epsilon' < U/8$ and by $U_n < U/8$ and $ \|\bN\|_\infty$ is sufficiently small.

\paragraph{Conditions in Lemma~\ref{lem:main_col_meta}.}
First, consider \eqnref{eqn:est3} and \eqnref{eqn:est4}. As mentioned above, since $C_1^3 = O(c_2^2 n)$ and $C_1^2 = O(c_2 n)$, then $h_3$  and $h_4$ can be  made sufficiently small to satisfy the conditions. \eqnref{eqn:est1} and \eqnref{eqn:est2} can be verified by plugging \eqnref{eqn:est3} and \eqnref{eqn:est4} and the assumption that $\lambda \ge \min_{i \in [n]} \E [(x_i^*)^2] / \kappa$. 

This completes the proof.
\end{proof}

\section{Auxiliary lemmas for solving recurrence} \label{sec:recurrence}

The following lemmas are used when solving recurrence in our analysis. 

\begin{lem}[Coupling update rule] \label{lem:coupling}
Let $\{a_t\}_{t = 0}^{\infty},  \{b_t \}_{t = 0}^{\infty}$ be sequences of non-negative numbers such that for  fixed values $h \ge 0$, $\eta \in [0, 1]$, $R > 4 r > 0$:
\begin{eqnarray*}
a_{t + 1} &\le& (1 - \eta) a_t +  \eta r b_t + \eta h
\\
b_{t + 1} &\le& (1 - \eta) b_t +  \frac{\eta}{R} a_t + \eta h
\end{eqnarray*}

Then the following two properties holds:
\begin{enumerate}
\item $$\forall t \ge 0, a_t + b_t \le a_0 + b_0 + \frac{Rr + 2R + 1}{R - r}h $$
\item For all $\epsilon  > 0$, when $t \ge \ln \frac{a_0 + b_0}{8 \eta \epsilon}$, we have: 
$$a_t \le   \frac{R(r + 1)}{R - r} h + \epsilon ,  \quad b_t \le \frac{R + 1}{R - r} h + \epsilon $$
\end{enumerate}
\end{lem}

\begin{proof}[Proof of Lemma \ref{lem:coupling}]

Observe that the update rule is equivalent to 
\begin{eqnarray*}
\left(a_{t + 1} - \frac{R(r + 1)}{R - r} h \right) &\le& (1 - \eta) \left(a_t  - \frac{R(r + 1)}{R - r} h \right)+  \eta r \left( b_t  - \frac{R + 1}{R - r} h \right)
\\
\left(b_{t + 1} - \frac{R + 1}{R - r} h\right) &\le& (1 - \eta) \left(b_t - \frac{R + 1}{R - r} h\right) +  \frac{\eta}{R} \left(a_t  - \frac{R(r + 1)}{R - r} h \right)
\end{eqnarray*}

Therefore, define $c_t = a_t  - \frac{R(r + 1)}{R - r} h$ and $d_t = b_t - \frac{R + 1}{R - r} h$, we can rewrite above as:
\begin{eqnarray*}
c_{t + 1} &\le& (1 - \eta) c_t +  \eta r d_t 
\\
d_{t + 1} &\le& (1 - \eta) d_t +  \frac{\eta}{R} c_t 
\end{eqnarray*}

Since we just need to upper bound $c_t, d_t$. without lose of generality, we can assume that 
\begin{eqnarray*}
c_{t + 1} &=& (1 - \eta) c_t +  \eta r d_t 
\\
d_{t + 1} &=& (1 - \eta) d_t +  \frac{\eta}{R} c_t 
\end{eqnarray*}

Which implies that 
\begin{eqnarray*}
\left(c_{t + 1}  + \sqrt{\frac{R}{r}} d_{t + 1} \right)&=&\left( 1 - \eta + \eta \sqrt{\frac{r}{R} }\right) \left(c_{t}  + \sqrt{\frac{R}{r}} d_{t} \right)
\\
\left(c_{t + 1}  - \sqrt{\frac{R}{r}} d_{t + 1} \right)&=&\left( 1 - \eta - \eta \sqrt{\frac{r}{R} }\right) \left(c_{t}  - \sqrt{\frac{R}{r}} d_{t} \right)
\end{eqnarray*}

Which can be simplified to 
\begin{eqnarray*}
\left(c_{t }  + \sqrt{\frac{R}{r}} d_{t } \right)&=&\left( 1 - \eta + \eta \sqrt{\frac{r}{R} }\right)^t \left(c_{0}  + \sqrt{\frac{R}{r}} d_{0} \right)
\\
\left(c_{t }  - \sqrt{\frac{R}{r}} d_{t } \right)&=&\left( 1 - \eta - \eta \sqrt{\frac{r}{R} }\right)^t \left(c_{0}  - \sqrt{\frac{R}{r}} d_{0} \right)
\end{eqnarray*}

Therefore, we can solve
$$c_t = \frac{1}{2} \left[\left( 1 - \eta + \eta \sqrt{\frac{r}{R} }\right)^t + \left( 1 - \eta - \eta \sqrt{\frac{r}{R} }\right)^t  \right] c_0 + \frac{1}{2}\sqrt{\frac{R}{r}}   \left[\left( 1 - \eta + \eta \sqrt{\frac{r}{R} }\right)^t - \left( 1 - \eta - \eta \sqrt{\frac{r}{R} }\right)^t  \right] d_0 $$

$$d_t = \frac{1}{2}\sqrt{\frac{r}{R}} \left[\left( 1 - \eta + \eta \sqrt{\frac{r}{R} }\right)^t - \left( 1 - \eta - \eta \sqrt{\frac{r}{R} }\right)^t  \right] c_0 + \frac{1}{2}  \left[\left( 1 - \eta + \eta \sqrt{\frac{r}{R} }\right)^t + \left( 1 - \eta - \eta \sqrt{\frac{r}{R} }\right)^t  \right] d_0$$

Observe that for every $ t \ge 0$, $a \ge b \ge 0$, $a^t - b^t \le (a - b) t a^{t - 1}$

Which implies: 
$$\left( 1 - \eta + \eta \sqrt{\frac{r}{R} }\right)^t - \left( 1 - \eta - \eta \sqrt{\frac{r}{R} }\right)^t   \le 2 t \eta \sqrt{\frac{r}{R}} \left( 1 - \eta + \eta \sqrt{\frac{r}{R} }\right)^{t - 1} $$

Therefore, when $c_0, d_0 \ge 0$, 
$$c_t \le \left( 1 - \eta + \eta \sqrt{\frac{r}{R} }\right)^t c_0 +  t \eta  \left( 1 - \eta + \eta \sqrt{\frac{r}{R} }\right)^{t - 1} d_0$$

Moreover, $$d_t \le \frac{r}{R} \eta \left( 1 - \eta + \eta \sqrt{\frac{r}{R} }\right)^t  c_0 + \left( 1 - \eta + \eta \sqrt{\frac{r}{R} }\right)^t d_0 $$

Taking the optimal $t$, we obtain $c_t + d_t\le c_0 + d_0$, which implies that 
$$a_t + b_t \le a_0  + b_0 + \frac{Rr + 2R + 1}{R - r}h$$

On the other hand, when $t \ge \ln \frac{c_0 + d_0}{8 \eta \epsilon}$, $c_t, d_t  \le \epsilon$, which implies that 
$$a_t \le   \frac{R(r + 1)}{R - r} h + \epsilon ,  \quad b_t \le \frac{R + 1}{R - r} h + \epsilon.$$

\end{proof}

\begin{lem}[Simple coupling] \label{lem:rec_diffh}
Let $\{a_t\}_{t = 0}^{\infty},  \{b_t \}_{t = 0}^{\infty}$ be sequences of non-negative numbers such that for  fixed values $h_1, h_2 \ge 0$, $\eta \in [0, 1]$, $r > 0$:
\begin{eqnarray*}
a_{t + 1} &\le& (1 - \eta) a_t + \eta h_1
\\
b_{t + 1} &\le& (1 - \eta) b_t +  \eta s a_t + \eta h_2 
\end{eqnarray*}
Then
\begin{eqnarray*}
a_t & \le  & u_a := \max\cbr{a_0, h_1}, 
\\
 b_t & \le & \max\cbr{b_0,  h_2 + s u_a }.
\end{eqnarray*}
\end{lem}

\begin{proof}
We have 
\begin{eqnarray*}
(a_{t + 1} - h_1)&\le& (1 - \eta) (a_t - h_1)
\\
(b_{t + 1} - h_2)  &\le& (1 - \eta) (b_t - h_2) +  \eta s a_t
\end{eqnarray*}

Solving the first one gives
$$
 a_t \le  u_a := \max\cbr{a_0, h_1}.
$$

Then 
\begin{eqnarray*}
(b_{t + 1} - h_2)  &\le& (1 - \eta) (b_t - h_2) +  \eta s u_a
\end{eqnarray*}

leads to 
$$
 b_t \le \max\cbr{b_0,  h_2 + s u_a }.
$$
\end{proof}

\begin{lem}[Simple recursion] \label{l:simplerec}  
Let $\{a_t\}_{t = 0}^{\infty}$ be a sequences of non-negative numbers such that for  fixed values $h \ge 0$, $\eta \in [0, 1]$, 
\[
  a_{t + 1} \le (1 - \eta) a_t +  \eta h.
\]
Then, 
\[
  a_{t} \leq (1-\eta)^t a_0 + h,
\]
and thus for $t \ge \frac{\ln (\epsilon/a_0)}{\ln (1 - \eta)}$, we have 
\[
  a_t \le  \epsilon + h.
\]
\end{lem}
\begin{proof}

We will prove by induction that $a_{t} \leq (1-\eta)^t a_0 + h$, which implies the statement of the lemma. 
The base case is trivial, so we proceed to the induction: 
$$ a_{t+1} \leq (1-\eta)\left((1-\eta)^t a_0 + h\right) + \eta h \leq (1-\eta)^{t+1} a_0 + h$$
as we need. 
\end{proof}

\section{Detailed discussion about related work} \label{sec:detailed_relatedwork}

\subsection{Non-negative matrix factorization} 

The area of non-negative matrix factorization (henceforth NMF) has a rich empirical history, starting with the work of \cite{LeeSeu99}. In that paper, the authors propose two algorithms based on alternating minimization, one in KL divergence norm, and the other in Frobenius norm. They observe that these heuristics work quite well in practice, but no theoretical understanding of it is provided. 


On the theoretical side, \cite{AroGeKanMoi12} provide a fixed-parameter tractable algorithm for NMF: namely when if the matrix $\bA \in \mathbb{R}^{m \times n}$and $\bX \in \mathbb{R}^{n \times N}$, they provide an algorithm that runs in time $(mN)^n$. This is prohibitive unless $n$ is extremely small. Furthermore, the algorithm is based on routines from algebraic geometry, so its tolerance to noise is fairly weak. More precisely, if there are matrices $\bA^*$, $\bX^*$, s.t. 
$$ \|\bY - \bA^* \bX^*\|_F \leq \epsilon \bY$$         
their algorithm produces matrices $\bA, \bX$, s.t. 
$$ \|\bY - \bA^* \bX^*\|_F \leq O(\epsilon^{1/2} n^{1/4}) \bY$$         
They further provide matching hardness results: namely they show there is no algorithm running in time $(mN)^{o(n)}$ unless there is a sub-exponential running time algorithm for 3-SAT. 
They also study the problem under separability assumptions about the feature matrix. \cite{bhattacharyya2016nonnegative} studies the problem under heavy noise setting, but also needs assumptions related to separability, such as the existence of dominant features. Also, their noise model is different from ours. 

\subsection{Topic modeling} 

A closely related problem is topic modeling. Topic models are a generative model for text data, using the common bag-of-words assumption. In this case, the columns of the matrix $\bA^*$ (which have norm 1) can naturally be interpreted as topics, with the entries being the emmision probabilities of words in that topic. The vectors $x^*$ in this case also will have norm 1, and can be viewed as distributions over topics. In this way, $y^* = A^* x^*$ can be viewed as the vector describing the emission probabilities of words in a given document: first a topic $i$ is selected according to the distribution $x^*$, then a word is selected from topic $i$ according to the distribution in column $[\bA^*]^i$. 
There also exist work that assume $x^*_i \in [0,1]$ and are independent (e.g., ~\cite{zhu2012sparse}), which is closely related to our model.

The distinction from NMF is that when documents are fairly short, the empirical frequencies of the words in the document might be very far from $y^*$. For this reason, typicall the algorithms with provable guarantees look at the empirical covariance matrix of the words, which will concentrate to the true one when the number of documents grows, even if the documents are very short. This, however, results in algorithms that scale quadratically in the vocabulary size, which often is prohibitive in practice. 
Also note that since $x^*$ is assumed to have norm 1 in topic modeling, it does not satisfy our assumption (\textbf{A2}). However, there also exist work on topic modeling~\cite{zhu2012sparse} that do not restrict $x^*$ is assumed to have norm 1 and can satisfying our assumption. 

There is a rich body of empirical work on topic models, starting from the seminal work on LDA due to \cite{blei2003latent}. Typically in empirical papers the matrices $\bA^*$, as well as the vectors $x^*$ are learned using \emph{variational inference}, which can be interpreted as a kind of alternating minimization in KL divergence norm, and in the limit of infinite-length documents converges to the \cite{LeeSeu99} updates (\cite{AwaRis15}).        
 
From the theoretical side, there was a sequence of works by \cite{arora1},\cite{arora2}, as well as \cite{anandkumartopic}, \cite{ding2013topic}, \cite{ding2014efficient} and \cite{kannantopic}. All of these works are based on either spectral or combinatorial (overlapping clustering) approaches, and need certain ``non-overlapping'' assumptions on the topics. For example, \cite{arora1} and \cite{arora2} assume that the topic-word matrix contains ``anchor words''. This means that each topic has a word which appears in that topic, and no other.
\cite{anandkumartopic} on the other hand work with a certain expansion assumption on the word-topic graph, which says that if one takes a subset S of topics, the number of words in the support of these topics should be at least $|S| + s_{max}$, where $s_{max}$ is the maximum support size of any topic. 

Finally, in the paper \cite{AwaRis15} a version of the standard variational inference updates is analyzed in the limit of infinite length documents. The algorithm there also involves a step of ``decoding'', which recovers correctly the support of a given sample, and a ``gradient descent'' step, which updates $\bA^*$ in the direction of the gradient of a KL-divergence based objective function. However, \cite{AwaRis15} requires quite strong assumptions on both the warm start, and the amount of ``non-overlapping'' of the topics in the topic-word matrix. 
   
\subsection{ICA} 

In the problem of independent component analysis (henceforth ICA, also known as blind-source separation), one is given samples $y = \bA^* x^* + \eta$, where the distribution on the samples $x^*$ is independent for each coordinate, the 4-th moment of $x_i^*$ is strictly smaller than that of a Gaussian and $\bA^*$ has full rank.  
The classic papers \cite{comon1994independent} and \cite{frieze1996learning} solved this problem in the noiseless case, with an approach based on cumulants, and \cite{arora2012provable} solved it in another special case, when the noise $\eta$ is Gaussian (albeit with an unknown covariance matrix). 

Our approach is significantly more robust to noise than these prior approaches, since it can handle both adversarial noise and zero mean noise. This is extremely important in practice, as often the nature of the noise may not be precisely known, let alone exactly Gaussian.  

\subsection{Non-convex optimization via gradient descent} 

The framework of having a ``decoding'' for the samples, along with performing a gradient descent-like update for the model parameters has proven successful for dictionary learning as well, which is the problem of recovering the matrix $\bA^*$ from samples $y = \bA^* x^* + \eta$, where the matrix $\bA^* \in \mathbb{R}^{m \times n}$ is typically long (i.e. $n \gg m$) and $x^*$ is sparse. (No non-negativity constraints are imposed on either $\bA$ or $x^*$.)  In this scenario, the columns of $\bA^*$ are thought of as a \emph{dictionary}, and each sample $y$ is generated as a (noisy) sparse combination of the columns of the dictionary.      

The original empirical work which proposed an algorithm like this (in fact, it suggested that the V1 layer processes visual signals in the same manner) was due to \cite{olshausen1997sparse}. In fact, it is suggest that similar families of algorithms based on ``decoding'' and gradient-descent are neurally plausible as mechanisms for a variety of tasks like clustering, dimension-reduction, NMF, etc. (\cite{mitya1,mitya2,mitya3,mitya4}) 

A theoretical analysis of it came latter due to \cite{aroradictionary1}. They showed that with a suitable warm start, the gradient calculated from the ``decoding'' of the samples is sufficiently correlated with the gradient calculated with the correct value $x^*$, therefore allowing them to show the algorithm converges to a matrix $\bA$ close to the ground truth $\bA^*$. However, the assumption made in \cite{aroradictionary1} is that the columns of $\bA^*$ are \emph{incoherent}, which means that they have $l_2$ norm bounded by 1, and inner products bounded by $O(\frac{1}{\sqrt{m}})$. \footnote{This is satisfied when the columns are random unit vectors, and intuitively says the columns of the dictionary are not too correlated.}    

The above techniques are not directly applicable to our case, as we don't wish to have any assumptions on the matrix $\bA^*$. Additionally, the incoherence assumptions on the matrix $\bA^*$ used in  \cite{aroradictionary1}, in the case when $\bA^*$ needs to be non-negative and has $l_1$ column-wise norm would effectively imply that the columns of $\bA^*$ have very small overlap. 

\end{document}